\documentclass{article}[11pt]

\usepackage{times}
\usepackage{helvet}
\usepackage{courier}

\usepackage[colorlinks,citecolor=green!60!black,linkcolor=red!60!black,pdfdisplaydoctitle,
menucolor=orange!40!black,filecolor=magenta!40!black,urlcolor=blue!40!black,linkcolor=red!40!black,citecolor=green!40!black,colorlinks,pagebackref]{hyperref}

\usepackage[numbers]{natbib}

\usepackage{rotating}

\usepackage{subcaption}
\usepackage{booktabs}
\usepackage{amsthm}
\usepackage{amssymb}
\usepackage{mathtools}
\usepackage[makeroom]{cancel}
\usepackage{paralist}
\usepackage{multirow}
\usepackage{xspace}
\usepackage{graphicx}
\usepackage{tikz}
\usetikzlibrary[decorations,arrows]
\usepackage[textsize=scriptsize]{todonotes}
\usepackage{thmtools}
\usepackage{thm-restate}
\newcommand{\fixlist}{}
\usepackage[ruled,linesnumbered]{algorithm2e}

\SetAlFnt{\small}
\SetAlCapFnt{\small}
\SetAlCapNameFnt{\small}
\SetAlCapHSkip{0pt}
\IncMargin{-\parindent}

\makeatletter
\def\NAT@spacechar{~}%
\makeatother

\newtheorem{theorem}{Theorem}[section]
\newtheorem{corollary}[theorem]{Corollary}
\newtheorem{lemma}[theorem]{Lemma}
\newtheorem{example}{Example}

\theoremstyle{definition}
\newtheorem{definition}{Definition}

\usepackage{authblk}

\usepackage{colortbl}
\usepackage{xcolor}
\definecolor{lightgray}{rgb}{0.8,0.8,0.8}
\newcommand{\graytext}[1]{\fcolorbox{lightgray}{lightgray}{#1}}

\usepackage{lscape}

\newcommand{\comm}[1]{{\color{gray}#1}}

\newcommand{\probCDCACInstance}{\ensuremath{((\electionC, \electionV), A, d, \combRule, \solk)}}

\newcommand{\probSetCover}{\textsc{Set Cover}\xspace}
\newcommand{\probSetCoverInstance}{\ensuremath{(X,\calS, h)}}

\newcommand{\probClique}{\textsc{Clique}\xspace}
\newcommand{\probCVC}{\textsc{Cubic Vertex Cover}\xspace}

\newcommand{\probVC}{\textsc{Vertex Cover}\xspace}

\newcommand{\probColorClique}{\textsc{Multi-Colored Clique}\xspace}
\newcommand{\probMCCInstance}{\ensuremath{(G, h)}}

\newcommand{\np}{{\mathsf{NP}}}
\newcommand{\paranp}{{\mathsf{Para}\textrm{-}\mathsf{NP}}}
\newcommand{\fpt}{{\mathsf{FPT}}}

\newcommand{\xp}{{\mathsf{XP}}}
\newcommand{\wone}{{\mathsf{W[1]}}}
\newcommand{\wtwo}{{\mathsf{W[2]}}}

\newcommand{\p}{{\mathsf{P}}}

\newcommand{\pref}{\ensuremath{\succ}}

\newcommand{\calR}{\ensuremath{\mathcal{R}}}

\newcommand{\calS}{\ensuremath{\mathcal{S}}}
\newcommand{\electionC}{\ensuremath{C}}
\newcommand{\electionV}{\ensuremath{V}}

\newcommand{\combRule}{\ensuremath{\kappa}\xspace}

\newcommand{\revnot}[1]{\overleftarrow{#1}}

\newcommand{\ouralpha}{$0 \leq \alpha \leq 1$}

\newcommand{\constantk}{\ensuremath{t}}
\newcommand{\solk}{\ensuremath{k}}
\newcommand{\park}{\ensuremath{p}}

\newcommand{\probDef}[3]{
  \begin{quote}
   #1\\
  \textbf{Input:} #2\\
  \textbf{Question:} #3
  \end{quote}
}

\newcommand{\mcctech}{Multi-Colored Clique Proof Technique\xspace}
\newcommand{\cvctech}{Cubic Vertex Cover Proof Technique\xspace}
\newcommand{\setcovertech}{Set-Embedding Proof Technique\xspace}
\newcommand{\signaturetech}{Signature Proof Technique\xspace}

\newcommand{\SetCandidateSet}{\ensuremath{\calS_{\mathrm{cand}}}}
\newcommand{\ElementCandidateSet}{\ensuremath{X_{\mathrm{cand}}}}

\newcommand{\shortcite}[1]{\cite{#1}}
\newcommand{\citeA}[1]{\citet{#1}}
\newcommand{\shortciteA}[1]{\citet{#1}}

\begin{document}

\title{Elections with Few Voters: Candidate~Control Can Be Easy\thanks{A preliminary short 
version of this work has been presented at the 29th AAAI Conference on Artificial Intelligence (AAAI~'15),  Austin Texas, January, 2015~\shortcite{CheFalNieTal2015}.
Compared to our conference version, in this long version, we provide all omitted proofs 
and, thanks to  \citeA{MausRot2016}, we fix a flaw in our multi-colored clique proof technique.}}

\author[1]{Jiehua Chen\footnote{\url{jiehua.chen@tu-berlin.de}}}
\author[2]{Piotr Faliszewski\footnote{\url{faliszew@agh.edu.pl}}}
\author[1]{Rolf Niedermeier\footnote{\url{rolf.niedermeier@tu-berlin.de}}}
\author[1]{Nimrod Talmon\footnote{\url{nimrodtalmon77@gmail.com}}}

\affil[1]{Institut f\"ur Softwaretechnik und Theoretische Informatik, TU Berlin, Germany}
\affil[2]{AGH University of Science and Technology, Krakow, Poland}

\date{}

\maketitle

\begin{abstract}
  We study the computational complexity of candidate control in elections with few voters,
  that is, we consider the parameterized complexity of candidate control in elections with respect to the number of voters as a parameter.
  We consider both the standard scenario of
  adding and deleting candidates, where one asks
  whether a given candidate can become a winner
  (or, in the destructive case, can be precluded from winning)
  by adding or deleting few candidates,
  as well as a combinatorial scenario where adding/deleting a candidate
  automatically means adding or deleting a whole group of candidates.
  Considering several fundamental voting rules,
  our results show that the parameterized complexity of candidate control,
  with the number of voters as the parameter,
  is much more varied than in the setting with many voters.
\end{abstract}

\renewcommand{\descriptionlabel}[1]{\hspace{\labelsep}\emph{#1}}

\newcommand{\wonehard}{$\wone$-h}
\newcommand{\wonehardx}{$\wone$-h / $\xp$}
\newcommand{\xpopen}{$?$ / $\xp$}
\newcommand{\paranph}{$\paranp$-h}

\newcommand{\wonehl}{$\wone$-hard}
\newcommand{\paranphl}{$\paranp$-hard}

\section{Introduction}

Election control problems are concerned with affecting the result of
an election by modifying the structure of the election.
Such election modifications could be either introducing some new candidates or voters
or removing some existing candidates or voters from the election
or partitioning candidates or voters~\shortcite{bar-tov-tri:j:control,erd-now-rot:j:sp-av,fal-hem-hem-rot:j:llull,hem-hem-rot:j:destructive-control,men:j:range-voting,men-sin:c:control-schulze,fal-hem-hem:j:multimode,fal-hem-hem-rot:j:single-peaked-preferences,par-xia:strategic-schultze-ranked-pairs}.
We focus on the computational complexity of election control by adding
and deleting candidates (that is, candidate control), for the case
where the election involves only a few voters.  From the viewpoint of
computational complexity, voter control with few voters has not
received sufficient study.  We focus on very simple, practical voting
rules such as Plurality, Veto, and $\constantk$-Approval, but discuss
several more involved rules as well.  To analyze the effect of
allowing only a small number of voters, we use the formal tools of
parameterized complexity
theory~\shortcite{CyFoKoLoMaPiPiSa2015,DF13,flu-gro:b:parameterized-complexity,nie:b:invitation-fpt}.

From the viewpoint of classical complexity theory, most of the
candidate control problems for most of the typically studied voting
rules are $\np$-hard. Indeed, candidate control problems are
$\np$-hard even for the Plurality rule; nonetheless, there are some
natural examples of candidate control problems with polynomial-time
algorithms.  It turns out that for the case of elections with few
voters, that is, for control problems parameterized by the number of
voters, the computational complexity landscape of candidate control is
much more varied and sometimes quite surprising.  We present
a high-level discussion of our results in \autoref{section_discussion}
(to get a quick feel of the nature of the results we obtain, the
reader might also wish to consult \autoref{tab:summary} in~\autoref{section_discussion}).

In addition to the standard candidate control problems, we also study
their \emph{combinatorial} variants, where instead of adding/deleting
candidates one-by-one, we add/delete whole groups of candidates at
unit cost.  In this we continue our previous work, conducted by a
slightly different set of authors, on combinatorial voter
control~\shortcite{bul-che-fal-nie-tal:j:combinatorial-voter-control}
and combinatorial shift bribery~\shortcite{bre-fal-nie-tal:j:combinatorial-shift-bribery}
(we mention that a somewhat similar model of combinatorial control was also
studied by~\citet{erd-hem-hem:t:comb-partition-control}).

The study of the computational complexity of control problems in
elections was initiated by~\citet{bar-tov-tri:j:control} and was
continued by many researchers (see, for example, the surveys
of~\citet{fal-rot:b:control}, of~\citet{fal-hem-hem:j:cacm-survey},
and our related work section).  While many researchers have considered
the case of few candidates---see, e.g., the classic work
of~\citet{con-lan-san:j:when-hard-to-manipulate} on election
manipulation and subsequent papers regarding
control~\shortcite{fal-hem-hem-rot:j:llull,fal-hem-hem:j:multimode,HemaspaandraLM16}---very
little effort was invested into studying the case of few voters
(perhaps the most notable example of a paper focusing on this case is
that of \citet{BrHaKaSe2014}; \citet{bet-sli-uhl:j:mon-cc} also
considered parametrization by the number of voters).  One possible
reason for this situation is that the case of few voters may seem
somewhat less natural. After all, presidential elections, an archetype
of elections, rarely involve more than a few candidates but do involve
millions of voters. We argue that the case of few voters is as natural
and as important to study, especially in the context of artificial
intelligence and multi-agent settings and various non-political
elections.\footnote{We should mention that since the publication of
  this paper's conference version, parameterization by the number of
  voters has become more common. For example,
  \citet{mis-nab-sin:c:minimax-av} considered winner determination
  parametrized by the number of voters under the Minimax Approval
  Voting rule, \citet{fal-sko-sli-tal:c:top-k-counting} did the same
  for a class of committee scoring rules, and
  \citet{bre-fal-nie-tal:c:multiwinner-sb} used this parameter in the
  study of shift bribery for committee elections. In each of these
  papers, parameterization by the number of voters leads to interesting
  $\fpt$ algorithms.}

\subsection{Motivation}\label{motivation}

Let us now argue why we believe that elections with few voters are
natural, and why (combinatorial) candidate control is an important
issue in such elections.  
First, let us look at the following examples, which include few voters
but possibly very many candidates.

\paragraph{{Hiring committee}}

Consider a university department which is going to hire a new faculty
member.  Typically, the committee consists of relatively few faculty
members,
but it may consider hundreds of applications for a given position.
The members of the committee have to aggregate their opinions
regarding the candidates and it is quite natural to assume that at
some point this would be done through voting.

\paragraph{{Holiday planning}}

Consider a group of people who are planning to spend holidays
together.  The group typically would consist of no more than a dozen
persons, but---technically---they have to choose from all possible
options provided by the travel agents, hotels, airlines, etc.  This
example is particularly relevant to the case of multi-agent systems:
one may foresee that in the near future we will delegate the task of
finding the most satisfying holiday location to our personal software
agents that will negotiate with travel agents and other travelers on
our behalf.

\paragraph{{Meta-search engine}}

\citet{dwo-kum-nao-siv:c:rank-aggregation} argued that one can build a
web meta-search engine that queries several other search engines (the
few voters) regarding a given query, aggregates their rankings of the
web pages (the many candidates), and outputs the consensus ranking.
\bigskip

In all these examples, it is clear that prior to holding an
election, the voters, or some particular individual, usually first shrinks the
set of candidates.  In the hiring committee case, most of the
applications are removed from the considerations early in the
evaluation process (based on the number of journal publications, for
example).  The group of people planning holidays first (implicitly)
removes most of the possible holiday options and, then, removes those
candidates that do not fully fit their preferences: for example, they
might remove destinations which are too expensive, or holiday
places by the sea when they are interested in hiking in the mountains,
etc.  The search engines usually disregard those web pages that appear
completely irrelevant to a given query.

This natural process of modifying the candidate set, however, creates
a natural opportunity for manipulating the result. A particularly
crafty agent may remove those candidates that prevent his or her
favorite candidate from winning.  Similarly, after the initial process
of thinning down the candidate set, a voter may request that some
candidates are added back into consideration, possibly to help his or
her favorite candidate. More importantly, it is quite realistic to
assume that the voters in a small group know each other so well as to
reliably predict each others' votes (this is particularly applicable
to the example of the hiring committee).  Thus, it is natural and
relevant to study the computational complexity of candidate control
parameterized by the number of voters. While control problems do not
model the full game-theoretic process of adding/deleting candidates,
they allow agents to compute what effects they might be able to
achieve, and, if the corresponding computational problem is tractable,
also how to achieve their goals.\footnote{To the best of our
  knowledge, game-theoretic aspects of this process of adding/deleting
  candidates have not been studied in detail. There is, however, a
  related line of research regarding \emph{strategic candidacy}, where
  the candidates themselves may decide to run or
  not~\shortcite{dut-jac-leb:j:strategic-candidacy} (see also, for example,
  the work of \citet{lan-mau-pol:c:strategic-candidacy},
  \citet{pol-obr-rab-kru-jen:c:strategic-candidacy}, and
  \citet{obr-elk-pol-rab:c:strategic-candidacy} for recent results).}

Finally, it is quite natural to consider the case where deleting
(adding) a particular candidate means also deleting (adding) a number
of other ones.  For example, if a hiring committee removes some
candidate from consideration, it might have to also remove all those
with weaker publication records; if people planning holidays disregard
some expensive hotel, they might also want to remove those that cost
more. Thus, we also study \emph{combinatorial variants} of candidate
control problems which model such settings.

\subsection{Main Contributions}

Our research sheds light on some surprising patterns that were not
(nearly as) visible in the context of classical complexity analysis of
election control. The two most interesting patterns can be summarized
as follows (by constructive control we mean variants of our problems
where we want to ensure some candidate's victory, whereas by destructive
control we mean cases where the goal is to prevent some candidate from
winning):

\begin{enumerate}\fixlist

\item In the non-combinatorial setting, destructive candidate control
  is easy for all our voting rules, either in the fixed-parameter
  tractability sense or via outright polynomial-time algorithms.

\item In the combinatorial setting, control by deleting candidates appears to be 
  computationally harder than control by adding
  candidates.

\end{enumerate}

We also found an interesting difference in the complexity of
non-combinatorial constructive control by deleting candidates between
Plurality and Veto rules (under Plurality we elect the candidate that
is ranked first most often; under Veto we elect the candidate that is
ranked last least often). This is especially interesting since the
rules are so similar and there is no such difference for the case of
adding candidates.

Our results (see \autoref{section_discussion} and
\autoref{tab:summary}; formal definitions are given in the next
section) are of four types.
For each of our problems we show that it is either:
\begin{enumerate}\fixlist
\item[(1)] in $\p$, 
\item[(2)] in $\fpt$ (that is, is fixed-parameter tractable),
\item[(3)] is $\wone$-hard but is in $\xp$, or
\item[(4)] is $\paranp$-hard.\footnote{There is one exception. For
    $\constantk$-Veto-\textsc{Comb}-DCAC ($\constantk \ge 2$) we only show $\xp$ membership.  Thus,
    from our point of view, the case of $\constantk$-Veto-\textsc{Comb}-DCAC is
    still partially open.}
\end{enumerate}
For each case, the parameterization is by the number of voters.
Naturally, results of Type~(1) are the most positive\footnote{We
  evaluate the results from the computational complexity perspective
  and, hence, regard computational efficiency as positive.} because
they give unconditionally efficient algorithms.  Results of Type~(2)
are quite positive too (the exponential part of the running time of
the corresponding algorithms depends only on the number of voters and
not on the whole input size).  The third kind is less positive (under
the usual assumption that $\fpt \neq \wone$, $\wone$-hardness
precludes the existence of $\fpt$ algorithms, but membership in $\xp$
means that there are algorithms which are polynomial-time if the
number of voters is assumed to be a constant).  Results of Type~(4)
are the most negative ones (they mean that the corresponding problems
are $\np$-hard even for a constant number of voters; this precludes
membership in $\xp$, under the usual assumption that $\p \neq
\np$).

We emphasize that almost all of our results follow by applying proof
techniques that might be useful in further research on the complexity
of election problems with few voters. In particular, our
$\wone$-hardness results follow via reductions from the $\wone$-complete
\probColorClique problem and have quite a universal
structure. Similarly, our $\paranp$-hardness proofs follow either via
reductions from the $\np$-complete \probCVC problem and use a universal trick to
encode graphs into elections with eight votes, or are based on embedding
\textsc{Set Cover} instances in our problems. Our $\fpt$ algorithms
also have a fairly universal structure and are based on what we call the
Signature Proof Technique.  Indeed, we believe that introducing these
proof techniques is an important contribution of this paper.

\subsection{Related Work}

The complexity study of election control was introduced
by~\citet{bar-tov-tri:j:control}, who were later followed by numerous
researchers---mostly working within the field of artificial
intelligence---including, for example,
\citet{hem-hem-rot:j:destructive-control} (who introduced the
destructive variants of the control problems),
\citet{mei-pro-ros-zoh:j:multiwinner} (who considered control in
multiwinner voting rules and introduced a model that unifies the
constructive and the destructive settings),
\citet{con-lan-xia:c:control-combinatorial-domain} (who considered a
form of election control for elections over sets of interrelated
candidates), and many others.  We point the reader to the surveys by
\citet{fal-hem-hem:j:cacm-survey} and \citet{RotSch2013}, the book
chapter of \citet{fal-rot:b:control}, and to several recent papers on
the topic, including those focusing on the Schulze and Ranked Pairs
rules~\shortcite{par-xia:strategic-schultze-ranked-pairs,men-sin:c:control-schulze},
and those focusing on Bucklin and Fallback rule, also including
experimental
studies~\shortcite{erd-fel-rot-sch:j:bucklin-control-theory,erd-fel-rot-sch:j:bucklin-control-experiments}.
Briefly put, it turns out that for standard voting rules, control
problems are typically $\np$-hard (however, it is worth noting that
some of these hardness results disappear in restricted
domains~\shortcite{BrandtBHH15,fal-hem-hem-rot:j:single-peaked-preferences,mag-fal:c:sc-control}).

There is a growing body of research regarding the parameterized
complexity of voting problems (see, for example, the survey
by~\citet{bet-bre-che-nie:b:fpt-voting}), where typical parameters
include the solution size (for example, the number of candidates which
can be added, that is, the budget) and parameters related to the
election size (usually, the number of candidates; or, as is in our case, the number of
voters).  When considering the solution size as the parameter, control
problems typically turn out to be hard (for examples,
see the works
of~\citet{bet-uhl:j:parameterized-complexity-candidate-control},~\citet{fen-liu-lua-zhu:j:parameterized-control},
and~\citet{liu-zhu:j:maximin}).  On the contrary, taking the number of
candidates as the parameter almost always leads to $\fpt$ results
(see, for example, the papers by~\citet{kou-kno-mni:c:candidates}, by
\citet{bre-fal-nie-sko-tal:c:elections-and-covers},
by~\citet{fal-hem-hem:j:multimode}, and
by~\citet{HemaspaandraLM16}).  So far, however,
only~\citet{bet-uhl:j:parameterized-complexity-candidate-control}
considered a \emph{control} problem parameterized by the number of
voters (for the Copeland rule), and~\citet{BrHaKaSe2014} showed
$\np$-hardness results of several winner determination problems even
for constant numbers of voters.  The parameter~``number of voters''
also received some attention in other voting
settings~(see, for example, the paper
of~\citet{bet-guo-nie:j:dodgson-parametrized} on winner determination
for Dodgson and Young rule and the paper of~\citet{dor-sch:j:parameterized-swap-bribery} on swap bribery
problems in $t$-Approval rules), and it is currently receiving
increased
attention~\shortcite{bre-che-fal-nic-nie:j:parametrized-shift-bribery,mis-nab-sin:c:minimax-av,fal-sko-sli-tal:c:top-k-counting,bre-fal-nie-tal:c:multiwinner-sb}.

The study of \emph{combinatorial} control was initiated in our recent
paper~\shortcite{bul-che-fal-nie-tal:j:combinatorial-voter-control},
where we focused on voter control.  A different notion of
combinatorial control was studied
by~\citet{erd-hem-hem:t:comb-partition-control}, and some of us also
considered combinatorial shift
bribery~\cite{bre-fal-nie-tal:j:combinatorial-shift-bribery}.

We stress that our combinatorial view of control is different from the
studies of combinatorial voting domains considered, for example,
by~\citet{bou-bra-dom-hoo-poo:cp-nets},~\citet{xia-con:c:strategy-proof-multi-issue},
and~\citet{mat-pin-ros-ven:c:cp-bribery}. There, the authors consider
sets of candidates of possibly exponential size and voters who express
their preferences succinctly, using special formalisms (such as, for
example, CP-nets). In our case, sets of candidates and preference
orders are expressed directly, and the combinatorial flavor comes from
considering bundles of candidates.

Finally, we mention that the original construction in our
multi-colored clique proofs had a very subtle bug that was noted and
fixed by \citeA{MausRot2016}; we incorporate the fix here.

\subsection{Organization}

The paper is organized as follows. In the next section we provide some
preliminaries regarding the standard election model, formal
definitions of our control problems, and a brief review of relevant
notions from parameterized complexity theory.  Then,
in~\autoref{section_discussion}, we discuss our results. Specifically,
in~\autoref{section_proof_techniques} we discuss our proof techniques
in a high-level fashion, giving the main intuitive ideas, and provide
the most illustrative full proofs in the following sections:
in \autoref{sec:mcc} we discuss $\wone$-hardness proofs based on the \mcctech, 
in \autoref{sec:cvc} we discuss $\paranp$-hardness proofs based on the \cvctech, 
in \autoref{sec:set} we discuss $\paranp$-hardness proofs based on the \setcovertech, 
in \autoref{sec:sig} we discuss $\fpt$ algorithms based on the Signature
Proof Technique, and in \autoref{sec:oth} we consider the remaining
problems.  Finally, in~\autoref{section_outlook}, we present
challenges for future research.  The proofs not presented in the main
text are deferred to an appendix.

\section{Preliminaries}

In this section, we provide some relevant notions concerning elections
and voting rules, control problems, and parameterized complexity.  We
denote the set $\{1, 2, \ldots, z\}$ by $[z]$. For a set $A$, we write
$\mathcal{P}(A)$ to denote the family of all subsets of $A$.

\subsection{Elections}

We consider the standard, ordinal model of elections (see, for example, the
handbook edited by \citet{arr-sen-suz:b:handbook-of-social-choice} for
a general overview of elections and social choice theory, or the
handbook of computational social
choice~\shortcite{bra-con-end-lan-pro:b:comsoc} for a more
computational perspective on the topic).  An election $E = (C,V)$
consists of a set~$C = \{c_1, \ldots, c_m\}$ of candidates and a
collection $V = (v_1, \ldots, v_n)$ of voters. Each voter $v_\ell$ has
a preference order~(that is, a vote), often denoted by $\pref_\ell$,
which ranks the candidates from the one that $v_\ell$ likes most to
the one that $v_\ell$ likes least.  For example, if
$C = \{c_1, c_2, c_3\}$ then voter $i$ with preference order
$c_1 \pref_i c_2 \pref_i c_3$ would most like $c_1$, then $c_2$, and
then $c_3$.  Throughout the text, we use masculine forms to refer to
the candidates and feminine forms to refer to the voters.

For a voter $v_\ell$ and two candidates, $c_i$ and $c_j$, we sometimes
write $v_\ell \colon c_i \pref c_j$ to indicate that $v_\ell$ prefers
$c_i$ to~$c_j$.  For a subset $A$ of candidates, by writing $A$ within
a preference order description (for example, $A \pref a \pref b$,
where $a$ and $b$ are some candidates not in $A$) we mean listing
the members of $A$ in some arbitrary but fixed order.  By writing
$\revnot{A}$ we mean listing the candidates in the reverse of this
arbitrary but fixed order.  Given an election $E = (C,V)$, for each
two candidates, $c_i, c_j \in C$, we define $N_E(c_i,c_j) \coloneqq
|\{v_\ell \in V \mid v_\ell\colon c_i \pref c_j\}|$ (that is,
$N_E(c_i,c_j)$ is the number of voters preferring $c_i$ to $c_j$).

A voting rule $\calR$ is a function that given an election $E = (C,V)$
outputs a set $\calR(E) \subseteq C$ of candidates that tie as winners
(we use the non-unique-winner model where the candidates in $\calR(E)$
are equally successful).  We study the following standard voting rules
(in each case, the candidates who receive the highest number of points
are the winners).

\paragraph{{$\boldsymbol\constantk$-Approval and $\boldsymbol\constantk$-Veto}}

  Under $\constantk$-Approval, $\constantk\geq 1$,
  each candidate gets a point from each voter that ranks him among
  the top~$\constantk$ positions.
  For $m$ candidates, $\constantk$-Veto is a synonym for
  $(m-t)$-Approval (we often view the score of a candidate under
  $t$-Veto as the number of vetoes, that is, the number of times he
  is ranked among the bottom $t$ positions).  We refer to $1$-Approval and
  $1$-Veto as the Plurality rule and the Veto rule, respectively,
  and we jointly refer to the voting rules in this group as approval-based rules.

\paragraph{{Borda rule and Maximin rule}} 

  Under the Borda rule, in election $E = (C,V)$ each
  candidate~$c \in C$ receives $\sum_{d \in C \setminus \{c\}} N_E(c,d)$ points.
  It is also convenient to think that under the Borda rule each 
  voter gives each candidate~$c$ as many points as there are
  candidates that this voter ranks below~$c$.
  Under the Maximin rule, each candidate~$c \in C$
  receives $\min_{d \in C \setminus \{c\}} N_E(c,d)$ points.

\paragraph{{Copeland$^\alpha$ rule}}

Under the Copeland$^\alpha$ rule (where $\alpha$ is a rational number,
$0 \leq \alpha \leq 1$), in election $E = (C,V)$ each candidate~$c$
receives $|\{ d \in C \setminus \{c\} \colon N_E(c,d) > N_E(d,c) \}| +
\alpha \cdot |\{ d \in C \setminus \{c\} \colon N_E(c,d) = N_E(d,c) \}|$
points.  Intuitively, under the Copeland$^\alpha$ rule we conduct a
head-to-head contest among each pair of candidates. For a given pair
of candidates, $c$ and $d$, the candidate who is preferred to the
other one by a majority of voters receives one point. If there is a tie,
both candidates receive $\alpha$ points.

\subsection{Control Problems}

We study \emph{candidate control} in elections, considering both
constructive control (CC) and destructive control (DC), by either
adding candidates (AC) or deleting candidates (DC). Thus, for example,
\textsc{CCAC} refers to constructive control by adding candidates.

For the case of problems of control by adding candidates, we make the
standard assumption that the voters have preference orders over all
the candidates---both those already registered and those that can be
added (that is, the unregistered candidates).
Naturally, when the election is conducted we consider the
preference orders to be restricted only to those candidates that
either were originally registered or were added. Similarly, for the
case of control by deleting candidates, to compute the result of the
election we restrict the preference orders only to those candidates
that were not deleted.

We consider combinatorial variants of our problems, where
adding/deleting a single candidate also automatically adds/deletes a whole
group of other candidates (in this we follow our earlier work on
combinatorial voter
control~\shortcite{bul-che-fal-nie-tal:j:combinatorial-voter-control}; see also
the work of \citet{erd-hem-hem:t:comb-partition-control}). In these
\emph{combinatorial variants} (denoted with a prefix Comb), we use
bundling functions $\combRule$ such that for each candidate $c$,
$\combRule(c)$ is the set of candidates which are also added if $c$ is
added (or, respectively, which are also deleted if $c$ is deleted).
For each candidate~$c$, we require that $c \in \combRule(c)$ and call
$\combRule(c)$ the \emph{bundle} of $c$.
For a given subset~$B$ of candidates,
we write~$\combRule(B)$ to denote $\bigcup_{c \in B}\combRule(c)$.  Bundling
functions are encoded by explicitly listing their values for all
arguments. 

Formally, given a voting rule~$\calR$, our problems are defined as
follows (we only list the combinatorial generalizations; the
non-combinatorial variants can be ``derived'' by using the identity
function as $\combRule$).

\probDef{
  $\calR$-\textsc{Comb}-\textsc{CCAC}
}
{
  An election~$(C,V)$, a set~$A$ of unregistered candidates such that
  the voters from $V$ have preference orders over $C\cup A$, a preferred
  candidate~$p\in C$,
  a bundling function~$\combRule \colon A \to \mathcal{P}(A)$, and a non-negative integer~$\solk$.
}
{
  Is there a subset $A'\subseteq A$ with $|A'|\le k$ such that $p\in \calR(C \cup \combRule(A'),V)$?
}

\probDef{
  $\calR$-\textsc{Comb}-CCDC
}
{
  An election~$(C,V)$, a preferred candidate~$p\in C$,
  a bundling function~$\combRule\colon C \to \mathcal{P}(C)$, and a non-negative integer~$\solk$.
}
{
  Is there a subset~$C' \subseteq C$ with $|C'| \leq k$ such
  that $p \in \calR(C \setminus \combRule(C'),V)$?
}

The destructive variants of our problems, $\calR$-\textsc{Comb}-DCAC and
$\calR$-\textsc{Comb}-DCDC, are defined analogously,
except that we replace the preferred candidate $p$ with the despised candidate $d$,  
and we ask whether it is possible to ensure that $d$ is \emph{not} a winner of the
election.
In the DCDC case, we explicitly disallow deleting any bundle
containing the despised candidate.
In the standard, non-combinatorial variants of control we omit the
prefix ``Comb'' and assume that for each candidate $c$ we have
$\combRule(c) = \{c\}$, omitting the bundling function in the respective
discussions.

We believe that our model of combinatorial candidate control is fairly
simple, but that it captures some of the most important features of
our motivating examples. We view it as simple because in a scenario
with $m$~candidates, there are at most~$m$ corresponding bundles of
candidates that can be added/deleted.  While in real life one might
expect many more (perhaps even exponentially many with respect to
$m$), this is enough to associate each candidate with a ``single
reason'' for which it could be added or deleted (for example, in the hiring
committee the single reason associated with candidate $c$ could be
``having only as many journal papers as $c$ has is not enough to be
considered for the position at hand''). Then, each candidate's bundle
would correspond to a set of candidates who also meet the condition
associated with him.  Note that each candidate could have a
bundle defined through a considerably different condition and, so, the
bundles could have quite an arbitrary structure.  Notably, even
this---quite simple---model turns out to be computationally difficult.

Indeed, one could consider an even simpler model. For
example, one could associate each candidate with a label and consider
actions of adding/deleting all candidates with a given label. 
(This is the model proposed by
\citet{erd-hem-hem:t:comb-partition-control} in the context of voter
control.) This model would not allow for a given candidate to belong
to more than one bundle and, in effect, might be computationally much
simpler than ours. Indeed, our techniques for showing hardness of the
combinatorial control problems would not apply to it.

\subsection{Parameterized Complexity}

A parameterized problem is in the complexity class $\fpt$ (termed fixed-param\-eter
tractable) if there exists an algorithm that, given an instance $I$ of
this problem (with parameter value $\park$ and instance size~$|I|$;
in this paper, the parameter value is always the number of voters
involved), computes an answer for this instance in $f(\park)\cdot
|I|^{O(1)}$ time, where $f$ is some computable function.  
A presumably larger complexity class is $\xp$ which consists of parameterized problems solvable in $O(|I|^{f(\park)})$~time.
Indeed, while the problems from both complexity classes, $\fpt$ and
$\xp$, come under the ``polynomial-time solvable when the
parameter~$\park$ is a constant'' description, it is decisive that in
the case of $\fpt$ the degree of the polynomial does not depend
on~$\park$, which is not the case for~$\xp$.  Problems in $\fpt$ are
viewed as tractable, whereas the class~$\xp$ is rather considered to
be at the high-level of the parameterized intractability hierarchy.
Specifically, it holds that
$\fpt \subseteq \wone \subseteq \wtwo \subseteq \cdots \subseteq \xp$.

Originally, the class $\wone$ was defined in terms of certain
circuit-based computations.  For our purposes, however, it is much
easier to define $\wone$ through its complete problems.  Specifically,
$\wone$ contains those problems that reduce to $\probColorClique$ (see
\autoref{def:mcc} below) in the parameterized sense.  A parameterized
reduction from a parameterized problem~$L$ to a parameterized
problem~$L'$ is a function that, given an instance~$(I,\park)$,
computes in time~$f(\park)\cdot |I|^{O(1)}$ an instance~$(I',\park')$
such that~$\park' \le g (\park)$
and~$(I,\park)\in L \Leftrightarrow (I',\park')\in L'$, where $f$ and
$g$ are some computable functions. We mention that in this paper all
reductions can in fact be performed in polynomial time.

\begin{definition}\label{def:mcc}
  An input instance of $\probColorClique$ consists of an undirected
  graph $G = (V(G),E(G))$ and a non-negative integer~$h$ such that
  the vertex set $V(G)$ is partitioned into $h$ sets,
  $V_1(G), \ldots, V_h(G)$, where each set corresponds to one of the
  $h$~colors (in a one-to-one manner).  We ask whether there exist
  $h$~vertices $v_1, \ldots, v_h$ such that for each $i$,
  $1 \leq i \leq h$, it holds that $v_i \in V_i(G)$, and each pair is
  connected by an edge.  We call the set of these $h$ vertices a
  \emph{multi-colored clique} of order $h$.
\end{definition}

We say that a problem is $\paranp$-hard if there is a proof of its
$\np$-hardness that produces an instance in which the value of the
parameter is upper-bounded by a constant. If a problem is
$\paranp$-hard for some parameter, then it cannot even belong to~$\xp$
for this parameter (unless $\p = \np$).  Similarly, if a problem is
$\wone$-hard, it cannot be in $\fpt$ (unless $\fpt = \wone$).  In our
$\paranp$-hardness proofs we will mostly rely on the following
$\np$-hard problem (and on its restricted variant, the \probCVC
problem; see \autoref{sec:cvc}).

\begin{definition}\label{def:sc}
  An instance of $\probSetCover$ consists of a ground set $X = \{x_1,
  \ldots, x_{n'}\}$, a family $\calS = \{S_1, \ldots, S_{m'}\}$ of subsets
  of $X$, and a non-negative integer~$h$.
  We ask whether it is possible to pick at most $h$ sets from~$\calS$ so
  that their union is~$X$.
  Such a collection of $h$ sets is called a \emph{set cover} of order $h$.
\end{definition}

For more details on parameterized complexity and parameterized
algorithms, we point the readers to the textbooks of \citet{CyFoKoLoMaPiPiSa2015},
\citet{DF13},~\citet{flu-gro:b:parameterized-complexity},
and~\citet{nie:b:invitation-fpt}.

\section{Discussion of Our Results and Proof
  Techniques}\label{section_discussion}

In this section, we review our results, discuss some relevant
patterns we identified in them, and provide a high-level overview of
our proof techniques.  This section can be viewed as a guide for
helping the reader to better understand the implications of our results,
presented in~\autoref{tab:summary}, and to get an intuitive
understanding of our means of obtaining them.  We begin by discussing
the results regarding the approval-based voting rules, then discuss
the results for the other voting rules, and finally describe our proof
techniques.

\newcommand{\multimode}{$\clubsuit$ }
\newcommand{\llull}{$\diamondsuit$ }
\newcommand{\bordacite}{$\heartsuit$ }
\newcommand{\betzler}{$\spadesuit$ }
\newcommand{\betzlerfollow}{$^\spadesuit$ }
\newcommand{\betzlerfollowphantom}{\phantom{$^\spadesuit$} }

\newcommand{\accactl}{$\calR$-CCAC  & \wonehardx & \wonehardx & \wonehardx & \wonehardx}
\newcommand{\accdctl}{$\calR$-CCDC  & $\fpt$    & \wonehardx & \wonehardx & \wonehardx}
\newcommand{\adcactl}{$\calR$-DCAC  & $\fpt$    & $\fpt$    & $\fpt$ & $\fpt$}
\newcommand{\adcdctl}{$\calR$-DCDC  & $\fpt$    & $\fpt$    & $\fpt$ & $\fpt$}

\newcommand{\acccactl}{$\calR$-\textsc{Comb}-CCAC  & \wonehardx & \wonehardx & \wonehardx & \wonehardx}
\newcommand{\acccdctl}{$\calR$-\textsc{Comb}-CCDC  & \paranph~($1$) & \paranph~($1$) & \paranph~($1$) & \paranph~($1$)}
\newcommand{\acdcactl}{$\calR$-\textsc{Comb}-DCAC  & $\fpt$    & $\fpt$ & \wonehardx & \xpopen}
\newcommand{\acdcdctl}{$\calR$-\textsc{Comb}-DCDC  & \paranph~($3$) & \paranph~($1$) & \paranph~($2$) & \paranph~($1$)}

\newcommand{\bccactl}{$\calR$-CCAC  & \paranph~($10$) & \paranph~($10$) & \paranph~($20$)~\betzlerfollow}
\newcommand{\bccdctl}{$\calR$-CCDC  & $\p$~\multimode & \paranph~($10$) & \paranph~($26$)~\betzlerfollow}
\newcommand{\bdcactl}{$\calR$-DCAC  & $\p$~\multimode & $\p$~\bordacite & $\p$~\llull}
\newcommand{\bdcdctl}{$\calR$-DCDC  & $\p$~\multimode & $\p$~\bordacite & $\p$~\llull}

\newcommand{\bcccactl}{$\calR$-\textsc{Comb}-CCAC  & \paranph~($6$) & \paranph~($2$) & \paranph~($3$)~\betzlerfollow}
\newcommand{\bcccdctl}{$\calR$-\textsc{Comb}-CCDC  & \paranph~($1$) & \paranph~($1$) & \paranph~($1$)~\betzlerfollow}
\newcommand{\bcdcactl}{$\calR$-\textsc{Comb}-DCAC  & $\p$ & \paranph~($2$) & \paranph~($3$)~\betzlerfollowphantom}
\newcommand{\bcdcdctl}{$\calR$-\textsc{Comb}-DCDC  & \paranph~($5$) & \paranph~($2$) & \paranph~($3$)~\betzlerfollowphantom}

\begin{table*}[t]
  \setlength{\tabcolsep}{3pt}
  \centering
  \small
  \begin{subtable}{.9\linewidth}
    \centering
      \caption{Approval-based voting rules}
    \begin{tabular}{c|cccc}
    \toprule
    Problem & Plurality & Veto & $\constantk$-Approval & $\constantk$-Veto \\
    \bottomrule
    \accactl \\
    \accdctl \\
    \adcactl \\
    \adcdctl \\
    \midrule
    \acccactl \\
    \acccdctl \\
    \acdcactl \\
    \acdcdctl \\
    \bottomrule
    \end{tabular}
  \end{subtable}

  \vspace{5px}
  
  \begin{subtable}{.9\linewidth}
    \centering
      \caption{Other voting rules}
    \begin{tabular}{c|ccc}
    \toprule
    Problem & Maximin & Borda & Copeland$^\alpha$ \\
    \bottomrule
    \bccactl \\
    \bccdctl \\
    \bdcactl \\
    \bdcdctl \\
    \midrule
    \bcccactl \\
    \bcccdctl \\
    \bcdcactl \\
    \bcdcdctl \\
    \bottomrule
    \end{tabular}
  \end{subtable} 
  \caption{\label{tab:summary}
  The complexity of candidate control 
    (constructive~(CC) and destructive~(DC), adding candidates~(AC) and 
    deleting candidates~(DC)) problems 
    for various voting rules~$\calR$ parameterized by the number of voters (for $\constantk$-Approval and $\constantk$-Veto
    we mean $\constantk \geq 2$; for Copeland$^\alpha$, we mean \ouralpha (which we define to always be a rational number);
    note that the results by~\protect\citet{bet-uhl:j:parameterized-complexity-candidate-control} hold only for $\alpha \in \{0, 1\}$).
    Results marked with \multimode and \llull are due to~\protect\citet{fal-hem-hem-rot:j:llull,fal-hem-hem:j:multimode},
    those marked with \bordacite are due to~\protect\citet{lor-nar-ros-bre-wal:c:replacing-candidates},
    and those marked with \betzlerfollow follow from the work of Betzler and Uhlmann for $\alpha \in \{0,1\}$ and
    are due to this paper for the remaining values.
    Cells containing statements of the form ``\paranph~($z$)'' mean that the relevant problem is $\np$-hard even with only $z$~voters.
    The question mark~($?$) means that the computational complexity is still open.}
\end{table*}

\subsection{Results for Approval-Based Voting Rules}

Approval-based rules are perhaps the simplest and the most frequently
used ones, so results regarding them are of particular interest.
Further, they exhibit quite interesting behavior with respect to the
complexity of candidate control parameterized by the number of voters.
For the approval-based voting rules discussed here,
the results regarding Plurality (and, to some extent, Veto),
are by far the most important ones.

In terms of standard complexity theory, all constructive and
destructive candidate control problems for Plurality and Veto are
$\np$-complete (see the works of \citet{bar-tov-tri:j:control} and
\citet{hem-hem-rot:j:destructive-control} for the results regarding
the Plurality rule; the results for Veto are easy to derive based on
those for Plurality and, for example,~\citet{elk-fal-sli:j:cloning} showed that Veto-\textsc{CCAC} is
$\np$-complete\footnote{For the cases of $t$-Approval and $t$-Veto, $t \geq 2$, the results
  are due to~\citet{lin:thesis:hard-election-problems}.}). 
Yet, if we consider parameterization by the
number of voters, the results change quite drastically. We make the
following observations:
\begin{enumerate}
\item The results for Plurality and Veto are no longer the same
  (specifically, Plurality-CCDC is in $\fpt$ whereas Veto-CCDC is
  $\wone$-hard).  This is quite surprising given both the similarities
  between these rules and the fact that their standard
  complexity-theoretic results for control are identical (yet, we mention
  that some results for them were known to be different previously;
  for example, weighted coalitional manipulation problem is in $\p$
  for the Plurality rule and is $\np$-complete for the Veto
  rule~\shortcite{hem-hem:j:dichotomy}).
\item For all the $t$-Approval and $t$-Veto rules (including Plurality
  and Veto), the destructive non-combinatorial candidate control
  problems are fixed parameter tractable. The
  constructive variants of these problems---with the exception of
  Plurality-CCDC---are $\wone$-hard.
\item For the combinatorial setting there is a sharp difference
  between control by adding candidates and control by deleting
  candidates.  Specifically, for both the Plurality rule and the Veto rule, 
  \textsc{Comb}-DCAC is fixed-parameter tractable and \textsc{Comb}-\textsc{CCAC}
  is $\wone$-hard, whereas \textsc{Comb}-CCDC and \textsc{Comb}-DCDC are
  $\paranp$-hard. For $t$-Approval and $t$-Veto with $t \geq 2$, the
  patterns are simpler. All the adding-candidates cases are
  $\wone$-hard (with one open case), and all the deleting-candidates cases are
  $\paranp$-hard.
\end{enumerate}

We conclude by noting that in each of the $\wone$-hard cases discussed
here we also provide an $\xp$ algorithm.  This means that,
under the assumption $\p \neq \np$,
these cases
cannot be strengthened to $\paranp$-hardness results and, thus, in
some sense our results are tight.  It is quite interesting that in the
non-combinatorial setting the demarcation line between fixed-parameter
tractable and $\wone$-hard problems goes along the
constructive-vs-destructive axis, whereas for the combinatorial
setting the line between $\wone$-hard (or, in-$\fpt$ for Plurality and
Veto) and $\paranp$-hard problems goes along the
adding-vs-deleting-candidates axis.

\subsection{Results for Maximin, Borda, and Copeland}

The results for Maximin, Borda, and Copeland$^\alpha$ rules are quite
different from those for $\constantk$-Approval and
$\constantk$-Veto. %
Here, instead of $\fpt$ and $\wone$-hardness results we find
polynomial-time algorithms and $\paranp$-hardness
results. Specifically, it has already been reported in the literature
that there are polynomial-time algorithms for destructive candidate
control under the Borda
rule~\shortcite{lor-nar-ros-bre-wal:c:replacing-candidates}, the
Copeland$^\alpha$ rule~\shortcite{fal-hem-hem-rot:j:llull}, and the Maximin
rule~\shortcite{fal-hem-hem:j:multimode}.  For constructive candidate
control, $\paranp$-hardness was already known for Copeland$^0$ and
Copeland$^1$~\shortcite{bet-uhl:j:parameterized-complexity-candidate-control},
while in this paper we establish the same $\paranp$-hardness for the
remaining values of $\alpha$ (that is, for $0 < \alpha < 1$), for the
Borda rule, and for the Maximin rule (in the latter case, only for
\textsc{CCAC}; CCDC is known to be in $\p$~\shortcite{fal-hem-hem:j:multimode}).

For the combinatorial setting, almost all of our problems are
$\paranp$-hard. The only exception is Maximin-\textsc{Comb}-DCAC,
which can be solved in polynomial time using an algorithm that, in
essence, is identical to the one for the non-combinatorial setting.
Our proofs of the $\paranp$-hardness results mostly rely on a
set-embedding technique (see the next section), which more-or-less
directly embeds instances of \textsc{Set Cover} in our control
problems. Due to the generality of this approach, we also prove that
for every voting rule $\calR$ that satisfies the unanimity principle
(that is, for every voting rule that chooses as the unique winner the
candidate that is ranked first by all the voters, if such a candidate exists),
$\calR$-\textsc{Comb}-CCDC is $\paranp$-hard.

Summarizing the discussion above, for our more involved voting rules,
the high-level intuition is that constructive candidate control is
$\paranp$-hard even in the non-combinatorial setting, whereas destructive
candidate control is polynomial-time solvable in the non-com\-bi\-na\-to\-ri\-al settings, but
becomes $\paranp$-hard in the combinatorial ones. The only exception
from this rule is Maximin for the destructive control by adding
candidates (Maximin-(\textsc{Comb})-DCAC).

\subsection{Proof Techniques}\label{section_proof_techniques}

We believe that one of the most important contributions of this paper
comes from identifying four very general proof techniques for
establishing our results.  Indeed, we believe that these techniques
might be useful in studying the complexity of other election problems
(especially, parameterized by the number of voters) and below we provide
their high-level, intuitive descriptions.

\paragraph{{Overview of the \mcctech}}

This is a technique used for establishing $\wone$-hard\-ness results.
The main idea is to provide a reduction from the \probColorClique (MCC)
problem parameterized by the order of the desired clique.  This
problem, which is a variant of the standard \probClique problem, is
more naturally suited for parameterized complexity analysis than
\probClique itself~\shortcite{FHRV09}.  Specifically, each vertex is
associated with one color out of $h$ colors overall and we seek a clique of order
$h$ containing one vertex of each color.  The high-level description
of our technique is as follows. We provide a reduction that, given an
MCC-instance, introduces a candidate for each vertex and two
candidates for each edge.
We have to ensure that the only successful control actions add exactly
the candidates (delete all but exactly the candidates) which correspond to a
multi-colored clique (we mean both the candidates corresponding to the
vertices of the clique and the candidates corresponding to the edges
between them).  We enforce this constraint using pairs of carefully
crafted votes such that if we have two vertices but not an edge
between them, then some candidate receives one more point than it
should have for our preferred candidate to win.  Note that the colors
help us to upper-bound the number of voters needed for the
construction (specifically, the number of voters created in the
construction depends only on the parameter, therefore giving rise to
the parameterized hardness).  Formal proofs using this technique are
given in~\autoref{section_formal_proofs_mcc}.

\paragraph{{Overview of the \cvctech}}

This is a technique used for establishing $\paranp$-hard\-ness results
for non-combinatorial constructive candidate control problems.  The
main idea is to give a reduction from the \probCVC (CVC) problem.
This problem is a variant of the standard \probVC problem where the
input graph is guaranteed to be cubic (that is, each vertex is incident to exactly three edges).  The crucial observation used in
this technique is that the edges of a cubic graph can be partitioned
into four disjoint matchings.  This fact allows us to encode all
information regarding the graph in a constant number of votes, in a
way that ensures that the actions of adding/deleting candidates
correspond to covering edges.  Formal proofs using this technique are
given in~\autoref{section_formal_proofs_cvc}.

\paragraph{{Overview of the \setcovertech}}

This is a fairly simple technique for showing $\paranp$-hard\-ness
results for combinatorial control by adding/deleting candidates. The
idea is to reduce from the \probSetCover problem using the
bundling function to encode the given sets.  Due to the power of
bundling, a constant number of voters suffices for the reduction.
Formal proofs using this technique are given
in~\autoref{section_formal_proofs_sc}.

\paragraph{{Overview of the \signaturetech}}

This is a group of two similar techniques for showing $\fpt$ results
for the case of destructive control under $t$-Approval/$t$-Veto rules.
The first technique in the group works for problems of destructive
control by adding candidates.  The main idea is to group candidates by
finding some equivalences between them in the sense that it does not
make a difference which candidate from the set we add.  In other
words, often it is possible to limit the number of candidates that one
has to consider by identifying their most crucial properties (such as
the subsets of voters where the candidates are ranked ahead of some
given candidate); we refer to these properties as signatures.
Building upon this idea, and by upper-bounding the number of such
groups with function solely depending on the number of voters, we
achieve fixed-parameter tractability results.  The second technique is
of similar nature and applies to problems of destructive control by
deleting candidates.  Formal proofs using this technique are given
in~\autoref{section_formal_proofs_signatures}.  \bigskip

Almost all of the proofs in this paper follow by applying one of the
above four techniques.  The few remaining ones follow by direct
arguments and are given in \autoref{sec:other-proofs}.  The following
sections are ordered by the proof technique employed and present the
most notable results. Proofs omitted here are presented in the
appendix.

\section{\mcctech}\label{section_formal_proofs_mcc}
\label{sec:mcc}

\newcommand{\vertexsetsize}{\ensuremath{n'}}
\newcommand{\edgesetsize}{\ensuremath{m'}}

We start our technical discussion by describing the Multi-Colored
Clique proof technique.  We apply it to showing $\wone$-hardness of
candidate control problems for $t$-Approval/$t$-Veto rules. Indeed,
all $\wone$-hardness proofs in this paper are based on this
technique.  Specifically, we prove the following statements (all
results are for the parameterization by the number of voters):
  
\begin{enumerate}
\item For each fixed integer~$\constantk \geq 1$ and for each voting rule {$\calR \in
    \{$}$\constantk$-Approval, $\constantk$-Veto{$\}$}, $\calR$-\textsc{CCAC} (and
  therefore also $\calR$-\textsc{Comb}-\textsc{CCAC}) is \wonehl.
\item For each fixed integer~$\constantk \ge 2$ 
and for each voting rule {$\calR \in
    \{$}Veto, $\constantk$-Approval, $\constantk$-Veto{$\}$}, $\calR$-CCDC is \wonehl.
\item For each fixed integer~$\constantk \ge 2$, $t$-Veto-\textsc{Comb}-DCAC is \wonehl.
\end{enumerate}

All these results follow by reductions from $\probColorClique$ (hence
the name of the technique) and are quite similar in spirit. Thus, we
start by providing some common notation and observations for all of
them.

Let $I=\probMCCInstance$ be a $\probColorClique$ instance with graph
$G$ and non-negative integer $h$ (recall \autoref{def:mcc}).
The vertices of $G$ are partitioned into
$h$ sets, $V_1(G), \ldots, V_h(G)$, each containing all vertices
colored with a given color.  Without loss of generality, we assume
that $h > 2$ and each $V_i(G)$ contains the same number of vertices, denoted by
$\vertexsetsize$, and we rename the vertices so that for each color
$i$, $1 \leq i \leq h$, we have $V_i(G) = \{v^{(i)}_1, \ldots,
v^{(i)}_{\vertexsetsize}\}$.  The task is to decide whether there is a
clique of order~$h$ where each vertex comes from a different set
$V_i(G)$.  Moreover, and without loss of generality, we assume that
each edge in $G$ connects vertices with different colors, that the
input graph contains at least two vertices, and that this graph is
connected.

In our reductions, given an instance $I=\probMCCInstance$, 
we build elections with the
following candidates related to the graph~$G$ (in addition to the
candidates specific to each particular reduction).  For each vertex $v
\in V(G)$, we introduce a candidate denoted by the same symbol. For
each edge~$e = \{u,v\}$, we introduce two candidates $(u,v)$ and
$(v,u)$ (while our original graph is undirected, for our construction
we treat each undirected edge as two directed ones, one in each
direction).

In the description of our preference orders, we will use the following
orders over subsets of candidates:
\begin{enumerate}
  \item For each color $i$, when we write
  $V_i(G)$ in a preference order, we mean the order
  \begin{align*}
    V_i(G)\colon~ v^{(i)}_1 \pref v^{(i)}_2 \pref \cdots \pref v^{(i)}_{n'}\text{.}
  \end{align*}
  \item For each color $i$, each vertex $v^{(i)}_\ell \in
  V_i(G)$, and each color $j$ with $j \neq i$, we write
  $L(v^{(i)}_{\ell},j)$ to denote the order obtained from
  \[
  L(v^{(i)}_{\ell},j)\colon~ (v^{(i)}_{\ell},v^{(j)}_1) \pref \cdots
  \pref (v^{(i)}_{\ell},v^{(j)}_{\vertexsetsize})\]
  by removing those candidates $(v^{(i)}_{\ell},v^{(j)}_h)$ for which
  there is no edge~$\{v^{(i)}_{\ell},v^{(j)}_h\}$ in~$G$.
  Intuitively, $L(v^{(i)}_{\ell},j)$ lists all edge candidates for
  edges which have endpoint $v^{(i)}_{\ell}$ and go to vertices of
  color $j$ (the particular order of these edges in
  $L(v^{(i)}_{\ell},j)$ is irrelevant for our constructions).

  \item The following two preference orders are crucial for the
  $\probColorClique$ technique.  For each pair of colors, $i,j$ ($1 \leq i,j
  \leq h$, $i \neq j$) we define $R(i,j)$ and $R'(i,j)$ as follows:
  \begin{align*}
    R(i,j) \colon & v^{(i)}_1 \pref L(v^{(i)}_1,j)
                    \pref v^{(i)}_2 \pref L(v^{(i)}_2,j) 
                    \pref \cdots
                    \pref v^{(i)}_{\vertexsetsize} \pref L(v^{(i)}_{\vertexsetsize},j), \\[2mm]
    R'(i,j) \colon &
                     L(v^{(i)}_1,j) \pref v^{(i)}_1 
                     \pref L(v^{(i)}_2,j) \pref v^{(i)}_2
                     \pref \cdots
                     \pref L(v^{(i)}_{\vertexsetsize},j) \pref v^{(i)}_{\vertexsetsize}.
  \end{align*}
  
  The idea behind $R(i,j)$ and $R'(i,j)$ is as follows. Consider a
  setting where $u$ is a vertex of color $i$ and $v$ is a vertex of
  color $j$ (that is, $u \in V_i(G)$ and $v \in V_j(G)$).  Note that
  $R(i,j)$ and $R'(i,j)$ contain all candidates from $V_i(G)$ and
  $E(i,j)$. If we restrict these two preference orders to candidates~$u$ and
  $(u,v)$, then they will become $u \pref (u,v)$ and $(u,v) \pref u$.
  That is, in this case they are reverses of each other. However, if we
  restrict them to $u$ and some candidate $(u',v')$ with $u \neq u'$,
  then either they will be both $u \pref (u',v')$ or they will
  be both $(u',v') \pref u$. 
  Using this effect is at the heart of our
  constructions.

\item
  For each two colors $i$ and $j$, ($1 \leq i < j \leq h$, $i \neq j$),
  let $e(\{i,j\}) = (e^{\{i,j\}}_1, \ldots, e^{\{i,j\}}_t)$ denote
  some fixed sequence of all the edges between the vertices from
  $V_i(G)$ and $V_j(G)$ (the specific order of the edges in this
  sequence is irrelevant and we pick an easily computable one). Given
  an edge $e^{\{i,j\}}_\ell = \{u,v\}$ from this sequence, with
  $u \in V_i(G)$ and $v \in V_j(G)$, by $e^{(i,j)}_\ell$ we mean the
  candidate $(u,v)$ and by $e^{(j,i)}_\ell$ we mean the candidate
  $(v,u)$. We define the following two
  preference orders:
  \begin{align*}
    E(i,j) \colon& \;\;
                   e^{(i,j)}_1 \pref e^{(j,i)}_1 \;\; 
             \pref e^{(i,j)}_2 \pref e^{(j,i)}_2 \;\; 
             \pref \;\;\cdots\;\; \pref \;\;
                   e^{(i,j)}_t \pref e^{(j,i)}_t, \\[2mm]
    E(j,i) \colon& \;\;
                   e^{(j,i)}_1 \pref e^{(i,j)}_1 \;\;
             \pref e^{(j,i)}_2 \pref e^{(i,j)}_2 \;\; 
             \pref\;\; \cdots\;\; \pref \;\;
                   e^{(j,i)}_t \pref e^{(i,j)}_t. 
  \end{align*}
  Both $E(i,j)$ and $E(j,i)$ list all the candidates that correspond
  to the edges between vertices of colors $i$ and $j$ and the
  difference is that for each edge of the form
  $\{v^{(i)}_\ell,v^{(j)}_h\}$, in $E(i,j)$ we have
  $(v^{(i)}_\ell, v^{(j)}_h) \pref (v^{(j)}_h, v^{(i)}_\ell)$ and in
  $E(j,i)$ we have
  $(v^{(j)}_h, v^{(i)}_\ell) \pref (v^{(i)}_\ell, v^{(j)}_h)$. This
  construction of~$E(i,j)$ and $E(j,i)$ is due to
  \shortciteA{MausRot2016}.  We thank them for pointing out the flaw
  in our original construction and repairing it
  (we explain why the current construction works within the proofs).

\end{enumerate}

With the above setup, we are ready to prove the results of this
section. Here we give the most interesting examples of proofs; the
remaining ones are in Appendix~\ref{app:mcc}.

\begin{theorem}\label{lem:plurality-ccac}
  Plurality-\textsc{CCAC}, parameterized by the number of voters, is $\wone$-hard.
\end{theorem}

\begin{proof}
  Let $I=\probMCCInstance$ be our input instance of $\probColorClique$
  with graph~$G$ and non-negative integer $h$. Let the notation be the
  same as described above the theorem statement. We form an instance
  $I'$ of Plurality-\textsc{CCAC} as follows. Let the registered
  candidate set~$C$ consist of two candidates, $p$ and $d$, and let
  the set~$A$ of unregistered candidates contain all vertex
  candidates and all edge candidates for~$G$. Let $p$ be the
  preferred candidate. We construct the election such that the current
  winner is~$d$.  We introduce the following groups of voters (when we
  write ``$\cdots$'' in a preference order, we mean listing all the
  remaining candidates in some arbitrary order; we illustrate the
  construction in Example~\ref{example:mcc} after the proof):
  
  \begin{enumerate}
  \item For each color $i$ ($1 \leq i \leq h$), 
  we have one voter with preference order of the form
    \[ V_i(G) \pref d \pref \cdots \pref p.\]
  \item For each pair of colors $i,j$ ($1 \leq i,j \leq h$, $i \neq
    j$), we have $h-1$ voters with preference order of the form
    \[E(i,j) \pref d
    \pref \cdots \pref p.\]
  \item For each pair of colors $i,j$ ($1 \leq i,j \leq h$, $i \neq
    j$) we have two voters, one with
    preference order of the form 
    \[R(i,j) \pref d \pref \cdots \pref p,\]
    and one with preference
    order of the form
    \[R'(i,j) \pref d \pref \cdots \pref p.\]
  \item We have $h$~voters with preference order of the form
  \[d \pref \cdots \pref p,\]
  and
    $h$~voters with preference order of the form
    \[p \pref \cdots \pref d.\]
  \end{enumerate}

  Note that the total number of voters is
  $3h+2(h+1)\cdot {h \choose 2}$ and that the current winner is $d$
  with the score of $(2h+2(h+1)\cdot {h \choose 2})$~points.  We set
  the budget $k \coloneqq h + 2{h \choose 2} = h^2$.  This completes
  the construction (which is a parameterized reduction).

  We now claim that it is possible to ensure that $p$ becomes a winner by
  adding at most $k$ candidates if and only if $I$ is
  a ``yes''-instance.

  First, assume that $I$ is a ``yes''-instance of $\probColorClique$
  and let $Q$ be a size-$h$ subset of vertices that forms a multi-colored
  clique in $I$. If we add to our election the
  $h$ vertex candidates from $Q$ and all edge candidates that correspond
  to edges between the candidates from $Q$, then, in the resulting
  election, each candidate (including $p$ and $d$) will have $h$~points
  (for example, each of the added vertex candidates will receive one
  point from the first group of voters and $h-1$ points from the third
  group of voters). Thus, everyone will win.

  Now, assume that it is possible to ensure $p$'s victory by adding at
  most $k$~candidates. Let $A'$ be a subset of candidates such that
  $|A'| \leq k = h + 2{h \choose 2}$ and adding the candidates from
  $A'$ to the election ensures that $p$ is a winner. Irrespective of
  the contents of the set~$A'$, in the resulting election $p$ will
  have $h$~points. Thus, it follows that $d$ must lose all the points
  from the first three groups of voters. 
  This implies the following facts:
  \begin{enumerate}
  \item[(i)] For each color $i$, $1 \leq i \leq h$, $A'$ contains at least
    one vertex from $V_i(G)$. Otherwise $d$ would not lose all the
    points from the first group of voters. (This also ensures that $d$ loses
    all the points from the third group.)

  \item[(ii)] For each pair of colors $i,j$ ($1 \leq i < j \leq h$),
    $A'$ contains at least one edge candidate $(u,v)$ such that
    $u \in V_i(G)$ and $v \in V_j(G)$, and at least one edge candidate
    $(u',v')$ such that $u' \in V_j(G)$ and $v' \in V_i(G)$. Why is
    this the case? If for some pair of colors $i$ and $j$ neither of
    such edge candidates were included in $A'$, then $d$ would receive
    at least one point from the second group of voters, which would
    preclude $p$'s victory. On the contrary, if for some colors $i$
    and $j$, $A'$ contained edge candidate $(u,v)$ with $u \in V_i(G)$
    and $v \in V_j(G)$ but no edge candidate $(u',v')$ with
    $u' \in V_j(G)$ and $v' \in V_i(G)$, then $(u,v)$ would recevie
    $2(h-1)$ points from the second group of voters and $p$, again,
    would not be a winner.
  \end{enumerate}
  By a simple counting argument, the above two facts lead to the
  conclusion that for each color $i$, $1 \leq i \leq h$, $A'$ contains
  exactly one candidate from $V_i(G)$ and for each pair of colors
  $i,j$ ($1 \leq i, j \leq h$, $i \neq j$), $A'$ contains exactly one
  edge candidate $(u,v)$ such that $u \in V_i(G)$ and $v \in V_j(G)$.
  This leads to yet another observation:
  \begin{enumerate}
  \item[(iii)] For each pair of colors $i,j$ ($1 \leq i < j \leq h$),
    if $(u,v)$ and $(u',v')$ are the two candidates from $A'$ such
    that $u \in V_i(G)$, $v \in V_j(G)$ and
    $u' \in V_j(G), v' \in V_i(G)$, then it must be the case that
    $(u',v') = (v,u)$. Indeed, if this were not the case then,
    analogously to the reasoning in item (ii) above, either $(u,v)$ or
    $(u',v')$ would receive $2(h-1)$ points from the second group of
    voters and $p$ would not be a winner of the resulting election.
  \end{enumerate}

  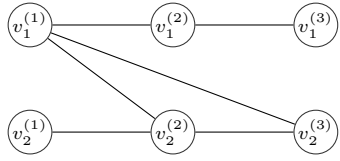
\begin{figure}
  \begin{center}
    \begin{tikzpicture}[scale=0.95]
      \scriptsize
      \tikzset{node/.style={circle,draw=black!80,minimum size=12pt,inner sep=0pt}}

      \node[node] (v11) at (0,1.5) {$v^{(1)}_1$};
      \node[node] (v21) at (2,1.5) {$v^{(2)}_1$};
      \node[node] (v31) at (4,1.5) {$v^{(3)}_1$};
      \node[node] (v12) at (0,0)   {$v^{(1)}_2$};
      \node[node] (v22) at (2,0)   {$v^{(2)}_2$};
      \node[node] (v32) at (4,0)   {$v^{(3)}_2$};
      
      \draw (v11) to node [auto,near start] {} (v21);
      \draw (v11) to node [auto,near     start] {} (v32);
      \draw (v11) to node [auto,near start] {} (v22);
      \draw (v21) to node [auto,near start] {} (v31);
      \draw (v12) to node [auto,near start] {} (v22);
      \draw (v22) to node [auto,near start] {} (v32);
    \end{tikzpicture}
  \end{center}
  \caption{The input graph considered in our example for~\autoref{lem:plurality-ccac}.}
  \label{figureexamplemcc}
  \end{figure}

  Now, to show that $G$ has a multi-colored clique of order~$h$, it
  suffices to show that for each two vertices~$u,v\in A'$ we have that
  $(u,v)\in A'$.  Consider two arbitrary vertices~$u,v\in A'$ and let
  $i,j$ be two colors with $u\in V_i(G)$ and $v\in V_j(G)$.  As per
  our previous reasoning, let $(u',v'), (v',u')\in A$ denote the pair
  of edge candidates that corresponds to the color pair $(i,j)$, that
  is, $u'\in V_i(G)$ and $v'\in V_j(G)$.  We show that
  $(u,v) = (u',v')$.  Suppose for the sake of contradiction that
  $(u,v)\neq (u',v')$, implying that either $u\neq u'$ or $v\neq v'$.
  Let us assume, without loss of generality, that $u\neq u'$ (the case
  with $v\neq v'$ is analogous).  First, observe that there is a total
  of $h+2(h+1)\cdot {h \choose 2}= h^3 =h\cdot k$ voters in the first
  three voter groups and they all give points to the newly added
  candidates.  Since each added candidate can have at most $h$~points,
  it follows that $|A'|=k$ and each added candidate receives exactly
  $h$~points.  By the observations regarding preference orders
  $R(i,j)$ and $R'(i,j)$ (in their definition, before this proof), and
  by obsevation (iii) above,
  either vertex candidate~$u$ or edge candidate~$(u',v')$ is ranked
  first by at least two voters from the third group.
  If this were the case for edge candidate~$(u',v')$, then---including
  the voters from the second group---this candidate would have more
  than $h$~points and $p$ would not be a winner. If this were the case
  for vertex candidate~$u$ (and none of the edge candidates were
  ranked first by more than one of the voters in the third group),
  then this vertex candidate would receive at least $h$~points from
  the voters in the third group and one point from the voters in the
  first group. Again, $p$ would not be a winner. Thus it must be that
  $(u,v) = (u',v')$, implying that for each $u, v \in A'$, there is an
  edge $\{u,v\}$ in the graph and, so, $G$ contains a multi-colored
  clique of order $h$.
\end{proof}

\begin{example}\label{example:mcc}
  We provide an example for the reduction described in the proof
  of~\autoref{lem:plurality-ccac}.  The input graph is depicted
  in~\autoref{figureexamplemcc}, we take $h \coloneqq 3$, and the
  election constructed by the reduction is given in
  \autoref{tabularexamplemcc}. In this table, the registered
  candidates are typeset in bold, the unregistered ones are typeset
  normally, and the added candidates are marked with gray background.
  These added candidates, who correspond to the multi-colored clique
  $\{v^{(1)}_1, v^{(2)}_2, v^{(3)}_2\}$ are:
  \begin{align*}
      & v^{(1)}_1, v^{(2)}_2, v^{(3)}_2, \\
      & (v^{(1)}_1, v^{(2)}_2), (v^{(1)}_1, v^{(3)}_2), (v^{(2)}_2, v^{(3)}_2), \\
      & (v^{(2)}_2, v^{(1)}_1), (v^{(3)}_2, v^{(1)}_1), (v^{(3)}_2, v^{(2)}_2).
  \end{align*}
  We see that $p$, with three points, is among the winners of this
  election (as are all the other candidates).

  \newcommand{\breakline}{\\ &&\\[-2.2ex]}
  \begin{table}
    \centering
    \resizebox{1\textwidth}{!}{
      \begin{tabular}{@{}l|ll@{}}
        Group, & Preference order \\
        Color(s) & \\
        \midrule
        $1$, $(1)$   & \graytext{$v^{(1)}_1$} $ \pref v^{(1)}_2 \pref \mathbf{d} \pref \cdots$ \breakline
        $1$, $(2)$   & $v^{(2)}_1 \pref$ \graytext{$v^{(2)}_2$} $\pref \mathbf{d} \pref \cdots$ \breakline
        $1$, $(3)$   & $v^{(3)}_1 \pref $ \graytext{$v^{(3)}_2$} $\pref \mathbf{d} \pref \cdots$ \breakline
        &\\[-1.6ex]
        \hline
        &\\[-1.6ex]
        $2$, $(1,2)$ & $(v^{(1)}_1, v^{(2)}_1) \pref (v^{(2)}_1, v^{(1)}_1) \pref $ \graytext{$(v^{(1)}_1, v^{(2)}_2) $} $\pref$ \graytext{$(v^{(2)}_2, v^{(1)}_1)$} $\pref (v^{(1)}_2, v^{(2)}_2)  \pref (v^{(2)}_2, v^{(1)}_2)\pref \mathbf{d} \pref \cdots$ & ($2$ copies) \breakline
        $2$, $(2,1)$ & $(v^{(2)}_1, v^{(1)}_1) \pref (v^{(1)}_1, v^{(2)}_1) \pref$ \graytext{$(v^{(2)}_2, v^{(1)}_1)$} $\pref$ \graytext{$(v^{(1)}_1, v^{(2)}_2)$} $\pref (v^{(2)}_2, v^{(1)}_2) \pref (v^{(1)}_2, v^{(2)}_2) \pref \mathbf{d} \pref \cdots$ & ($2$ copies)\breakline
        $2$, $(1,3)$ & \graytext{$(v^{(1)}_1, v^{(3)}_2)$} $\pref$ \graytext{$(v^{(3)}_2, v^{(1)}_1)$} $\pref \mathbf{d} \pref \cdots$  &($2$ copies) \breakline
        $2$, $(3,1)$ & \graytext{$(v^{(3)}_2, v^{(1)}_1)$} $\pref$ \graytext{$(v^{(1)}_1, v^{(3)}_2)$} $\pref \mathbf{d} \pref \cdots$ &($2$ copies) \breakline
        $2$, $(2,3)$ & $(v^{(2)}_1, v^{(3)}_1) \pref (v^{(3)}_1, v^{(2)}_1) \pref $ \graytext{$(v^{(2)}_2, v^{(3)}_2)$} $\pref$ \graytext{$(v^{(3)}_2, v^{(2)}_2)$} $\pref \mathbf{d} \pref \cdots$ &($2$ copies) \breakline
        $2$, $(3,2)$ & $(v^{(3)}_1, v^{(2)}_1) \pref (v^{(2)}_1, v^{(3)}_1) \pref$ \graytext{$(v^{(3)}_2, v^{(2)}_2)$} $\pref$ \graytext{$(v^{(2)}_2, v^{(3)}_2)$} $\pref \mathbf{d} \pref \cdots$ & ($2$ copies) \breakline
        &&\\[-1.6ex]
        \hline
        &&\\[-1.6ex]
        $3$, $(1,2)$ & \graytext{$v^{(1)}_1$} $\pref (v^{(1)}_1, v^{(2)}_1) \pref$ \graytext{$(v^{(1)}_1, v^{(2)}_2)$} $\pref v^{(1)}_2 \pref (v^{(1)}_2, v^{(2)}_2) \pref \mathbf{d} \pref \cdots$ \breakline
        $3$, $(1,2)$ & $(v^{(1)}_1, v^{(2)}_1) \pref$ \graytext{$(v^{(1)}_1, v^{(2)}_2)$} $\pref$ \graytext{$v^{(1)}_1$} $\pref (v^{(1)}_2, v^{(2)}_2) \pref v^{(1)}_2 \pref \mathbf{d} \pref \cdots$ \breakline
        $3$, $(2,1)$ & $v^{(2)}_1$ $\pref (v^{(2)}_1, v^{(1)}_1) \pref$ \graytext{$v^{(2)}_2$} $\pref$ \graytext{$(v^{(2)}_2, v^{(1)}_1)$} $\pref (v^{(2)}_2, v^{(1)}_2) \pref \mathbf{d} \pref \cdots$ \breakline
        $3$, $(2,1)$ & $(v^{(2)}_1, v^{(1)}_1) \pref v^{(2)}_1 \pref$ \graytext{$(v^{(2)}_2, v^{(1)}_1)$} $\pref (v^{(2)}_2, v^{(1)}_2) \pref$ \graytext{$v^{(2)}_2$} $\pref \mathbf{d} \pref \cdots$ \breakline
        $3$, $(1,3)$ & \graytext{$v^{(1)}_1$} $\pref$ \graytext{$(v^{(1)}_1, v^{(3)}_2)$} $\pref v^{(1)}_2 \pref \mathbf{d} \pref \cdots$ \breakline
        $3$, $(1,3)$ & \graytext{$(v^{(1)}_1, v^{(3)}_2)$} $\pref$ \graytext{$v^{(1)}_1$} $\pref v^{(1)}_2 \pref \mathbf{d} \pref \cdots$ \breakline
        $3$, $(3,1)$ & $v^{(3)}_1 \pref$ \graytext{$v^{(3)}_2$} $\pref$ \graytext{$(v^{(3)}_2, v^{(1)}_1)$} $\pref \mathbf{d} \pref \cdots$ \breakline
        $3$, $(3,1)$ & $v^{(3)}_1 \pref$ \graytext{$(v^{(3)}_2, v^{(1)}_1)$} $\pref$ \graytext{$v^{(3)}_2$} $\pref \mathbf{d} \pref \cdots$ \breakline
        $3$, $(2,3)$ & $v^{(2)}_1 \pref (v^{(2)}_1, v^{(3)}_1) \pref$ \graytext{$v^{(2)}_2$} $\pref$ \graytext{$(v^{(2)}_2, v^{(3)}_2)$} $\pref \mathbf{d} \pref \cdots$ \breakline
        $3$, $(2,3)$ & $(v^{(2)}_1, v^{(3)}_1) \pref v^{(2)}_1 \pref$ \graytext{$(v^{(2)}_2, v^{(3)}_2)$} $\pref$ \graytext{$v^{(2)}_2$} $\pref \mathbf{d} \pref \cdots$ \breakline
        $3$, $(3,2)$ & $v^{(3)}_1 \pref (v^{(3)}_1, v^{(2)}_1) \pref$ \graytext{$v^{(3)}_2$} $\pref$ \graytext{$(v^{(3)}_2, v^{(2)}_2)$} $\pref \mathbf{d} \pref \cdots$ \breakline
        $3$, $(3,2)$ & $(v^{(3)}_1, v^{(2)}_1) \pref v^{(3)}_1 \pref$ \graytext{$(v^{(3)}_2, v^{(2)}_2)$} $\pref$ \graytext{$v^{(3)}_2$} $\pref \mathbf{d} \pref \cdots$ \breakline
        &&\\[-1.6ex]
        \hline
        &&\\[-1.6ex]
        $4$         & $\mathbf{d} \pref \cdots$ & ($3$ copies) \breakline        
        $4$         & $\mathbf{p} \pref \cdots$ & ($3$ copies)        
      \end{tabular}
    }
    \caption{The election constructed in the proof of
      \autoref{lem:plurality-ccac} for our example, that is, for the
      input graph depicted in~\autoref{figureexamplemcc}.  The
      registered candidates (typeset in bold) are $d$ and~$p$.  The added candidates
      corresponding to picking the multi-colored clique $\{v^{(1)}_1,
      v^{(2)}_2, v^{(3)}_2\}$ are typeset with gray background.}
    \label{tabularexamplemcc}
\end{table}

\end{example}

We now consider the Veto-\textsc{CCAC} case,
which,
despite being a simple modification of the last proof,
is quite intriguing.

\begin{theorem}\label{lem:veto-ccac}
  Veto-\textsc{CCAC},
  parameterized by the number of voters,
  is $\wone$-hard.
\end{theorem}

\begin{proof}
  One can use the same construction (and proof) as for the Plurality-\textsc{CCAC}
  case (Theorem~\ref{lem:plurality-ccac}),
  but with the following modifications
  (note that the order is important,
  that is,
  we perform the second modification only after we have performed the first modification):
  \begin{enumerate}
    \item swap the occurrences of $p$ and $d$ in every vote, and
    \item reverse each vote.
  \end{enumerate}
  In effect, prior to adding candidates, $p$ is vetoed by all but $h$
  voters and $d$ is vetoed by exactly $h$ voters. If we add vertex candidates and edge candidates that correspond
  to a multi-colored clique, then every candidate in the election is
  vetoed by exactly $h$ voters and all the candidates are winners. 

  For the reverse direction, analogously as in the Plurality case
  (Theorem~\ref{lem:plurality-ccac}), we note that we have to add
  exactly one vertex candidate of each color and exactly one edge
  candidate for each (ordered) pair of colors (otherwise $p$ would
  receive more than $h$ vetoes, or one of the added candidates would
  receice fewer than $h$ vetoes). To argue that for each pair of
  vertex candidates $u$ and $v$ that we add, we also have to add edge
  candidate $(u,v)$, we use the same reasoning as in the Plurality
  case, but pointing out that if some candidate receives two vetoes
  from the third group of voters, then some other one receives,
  altogether, fewer than $h$ vetoes and $p$ is not a winner.
\end{proof}

To see why this result is intriguing, let us consider the following
voting rule, that we call \emph{TrueVeto}.  Under TrueVeto, a
candidate~$c$ is a winner if none of the voters ranks $c$ last. TrueVeto-\textsc{CCAC} is indeed $\np$-complete
(this can be proved by a reduction
from \textsc{Set Cover}, for example), but it is also in
$\fpt$ (when parameterized by the number of voters; an algorithm
similar to that for Plurality-DCAC, based on our signatures technique,
works; see \autoref{section_proof_techniques}).  If a Veto
election contained more candidates than voters, then at least one
candidate would never be vetoed and, in effect, the election would be
held according to the TrueVeto rule.  This means that in the proof
which shows that Veto-\textsc{CCAC} is $\wone$-hard, the election has fewer
candidates than voters, even after adding the candidates (and keep in
mind that the number of voters is the parameter!). Thus, the hardness
of the problem lays in picking a few spoiler candidates to add from a
large group of them. If we were adding more candidates than we had
voters, then the problem would be in $\fpt$.

Now, we move on to the deleting candidates case.  We will give a
detailed proof for Veto-CCDC (on the one hand, Plurality-CCDC is in
$\fpt$, and, on the other hand, it is instructive to see a detailed
proof for the case of Veto).  The proof still follows the general
ideas of the Multi-colored Clique technique, but since we delete
candidates, we have to adapt the approach.

\newcommand{\hforvetoone}{|V(G)|-h + 2|E(G)| - H}
\begin{theorem}\label{lem:veto-ccdc}
  Veto-\textsc{CCDC},
  parameterized by the number of voters,
  is $\wone$-hard.
\end{theorem}

\begin{proof}
  We provide a parameterized reduction from the $\probColorClique$
  problem.  Let $I=\probMCCInstance$ be our input instance with graph
  $G$ and non-negative integer $h$, and let the notation be as
  described in the introduction to the section. We form an instance
  $I'$ of Veto-CCDC as follows. Let the registered candidate set~$C$
  consist of all vertex candidates plus all edge candidates for~$G$, plus
  the preferred candidate~$p$.  We construct the following groups of
  voters (set $H = 2{h \choose 2} = h\cdot (h-1)$):
  \begin{enumerate}
  \item For each color $i$, $1 \leq i \leq h$, we introduce $2H-(h-1)$
    voters with preference order of the form
    \[\cdots \pref p \pref V_i(G).\]

  \item For each pair of colors $i,j$ ($1 \leq i,j \leq h$, $i \neq j$) we
  introduce $2H-1$ voters with preference order of the form
    \[\cdots \pref p \pref E(i,j).\]
  \item For each pair of colors, $i,j$ ($1 \leq i,j \leq h$, $i \neq j$)
    we introduce two voters, one with preference order of the form
    \[\cdots \pref p
    \pref R(i,j),\] and one with preference order of the form
    \[\cdots \pref p \pref
    R'(i,j).\]
  \item We introduce $2H$ voters with preference order of the form $\cdots \pref
    p$.
  \end{enumerate}
  We set the number~$k$ of candidates that can be deleted to
  $\hforvetoone$ (with the intention that one should delete all the
  candidates except for $p$ and those corresponding to the vertices
  and edges of the multi-colored clique of order~$h$).  This completes
  the construction.  Note that the total number of voters is
  \begin{align*}
      (2H-(h-1))\cdot h + (2H-1)\cdot H + H \cdot 2 + 2H \cdot 1 = 2H\cdot (H+h+1).
  \end{align*}
  Since the input graph is connected and contains at least two
  vertices (which means that the election has more than $H+h+1$
  candidates), there is at least one candidate, either a vertex
  candidate or an edge candidate, which has fewer than $2H$ vetoes.
  Thus, $p$ is currently not a winner.

  We claim that $p$ can become a winner by deleting at most $k$~candidates 
  if and only if $I$ is a ``yes''-instance.
  First, if $G$ contains an order-$h$ multi-colored clique and $Q$ is
  the set of $h$ vertices that form such a clique, then we can ensure
  that $p$ is a winner.  It suffices to delete all candidates from
  $V(G) \setminus Q$ and all edge candidates except the ones of the
  form $(u,v)$, where both $u$ and $v$ belong to $Q$. In effect, each
  remaining candidate will have $2H$~vetoes and all the candidates
  will tie for victory.  To see this, note that after deleting the
  candidates, $p$ still receives $2H$ vetoes from the last group of
  voters. Now, for each color $i$, $1 \leq i \leq h$, consider the
  remaining vertex candidate of color $i$ (call this vertex
  $v^{(i)}$). This candidate receives $2H-(h-1)$ vetoes from the first
  group of voters. Further, there are exactly $h-1$ voters in the
  third group that each give one veto to $v^{(i)}$ (these are the
  voters that correspond to the edges that connect $v^{(i)}$ with the
  other vertices of the clique). No other voter vetoes $v^{(i)}$.
  Now, for each pair of colors $i$ and $j$, $1 \leq i, j \leq h$,
  $i \neq j$, consider the two edge candidates, call them $(u,v)$ and
  $(v,u)$, whose corresponding edges are incident to the vertices of
  color $i$ (candidate $u$) and color $j$ (candidate $v$).  Both
  $(u,v)$ and $(v,u)$ still get $2H-1$ vetoes from the second group of
  voters. Each of them receives one veto from the third group of
  voters (for the case of $(u,v)$, this veto comes from the first
  voter corresponding to the color choice $(i,j)$, and in the case of
  $v$, this veto comes from the first voter corresponding to the color
  choice $(j,i)$).

  Now we come to the reverse direction. Assume that it is
  possible to ensure $p$'s victory by deleting at most
  $k$~candidates.
  Prior to deleting any candidates, $p$ has $2H$ vetoes and,
  of course, deleting candidates cannot decrease this number.
  Thus, we have to ensure that each non-deleted candidate has at least $2H$ vetoes.

  Consider two colors $i$ and $j$, $1 \leq i$, $j \leq h$, $i \neq
  j$. Each edge candidate $(u,v)$ (where the corresponding vertex~$u$ has
  color $i$ and the corresponding vertex~$v$
  has color~$j$) appears below~$p$ in $2H-1$ votes from the second
  group of voters and in two votes from the third one. If we keep two
  edge candidates, say $(u',v')$ and $(u'',v'')$ (where $u', u'' \in
  V_i(G)$ and $v', v'' \in V_j(G)$),
  then they are both ranked below~$p$
  in the same $2H-1$ votes from the second group and in the same
  two votes from the third one. If neither $(u',v')$ nor $(u'',v'')$
  is deleted, then one of them will receive fewer than $2H$
  vetoes. This means that for each pair of colors $i$ and $j$, we have to
  delete all except possibly one edge candidate of the form $(u,v)$,
  where $u \in V_i(G)$ and $v \in V_j(G)$.

  Similarly, for each color $i$, $1 \leq i \leq h$, each
  vertex candidate from $V_i(G)$ appears below $p$ in $2H-(h-1)$
  votes from the first group of voters and in $2(h-1)$ votes from the
  third group. Each two candidates of the same color are ranked below
  $p$ in the same votes in the first group. Thus, if two
  vertex candidates of the same color were left in the election (after
  deleting candidates), then at least one of them would have fewer
  than $2H$ vetoes.

  In consequence, and since we can delete at most $k=\hforvetoone$
  candidates, which means that at least $h+H$~candidates except $p$
  must remain in the final election, if $p$ is to become a winner,
  then after deleting the candidates the election must contain exactly
  one vertex candidate of each color, and exactly one edge candidate
  for each ordered pair of colors.

  Assume that $p$ is among the winners after deleting candidates and
  consider two remaining vertex candidates $u$ and $v$, $u \in V_i(G)$
  and $v \in V_j(G)$ ($i \neq j$); they must exist by the previous
  observations. We claim that edge candidates $(u,v)$ and $(v,u)$ also
  must remain. Due to symmetry, it suffices to consider
  $(u,v)$. Careful inspection of voters in the third group shows that
  if $(u,v)$ is not among the remaining candidates, then (using the
  observation regarding orders $R(i,j)$ and $R'(i,j)$) we have that
  the two voters from the third group that correspond to the color
  pair $(i,j$) either both rank $u$ last or both rank the same edge
  candidate last.  In either case, a simple counting argument shows
  that either $u$ has fewer than $2H$ vetoes or the edge candidate
  corresponding to the ordered color pair $(i,j)$ has fewer than $2H$
  vetoes.  In either case, $p$ is not a winner. This shows that the
  remaining candidates correspond to an order-$h$ multi-colored
  clique.
\end{proof}

Our final example of the application of the multi-colored clique
technique is for $\constantk$-Approval-\textsc{Comb}-DCAC for $\constantk \geq
2$.  We use an approach very similar to the one used in the preceding
proofs, but since we are in the combinatorial setting, we use the
bundling function to ensure consistency between the added edge
candidates and the added vertex candidates.  This is crucial since
$\constantk$-Approval-DCAC is in $\fpt$.

\begin{theorem}\label{lem:approval-comb-dcac}
  For each fixed integer~$\constantk \geq 2$,
  $\constantk$-Approval-\textsc{Comb}-\textsc{DCAC},
  parameterized by the number of voters,
  is $\wone$-hard.
\end{theorem}

\begin{proof}
  Given an instance~$(G,h)$ for the \probColorClique problem, we construct an instance of
  $\constantk$-Approval-\textsc{Comb}-DCAC.  For the combinatorial setting it is more natural
  to create only one candidate for each edge, and not two ``directed''
  ones. %
  We let the set of registered candidates be $C = \{p, d\} \cup D$,
  where $D$ is the following set of dummy candidates:
  \begin{align*}
    D =\ &\{d^{\{i,j\}}_{z} \mid i,j \in [h], i \neq j, z \in [t-1]\} \\ 
      \cup\ &\{d^{(i)}_{z} \mid i \in [h], z \in [t-1]\}\\
      \cup\ &\{e^{(i)}_{z} \mid i \in [h], z \in [t-1]\}.
    \end{align*}

  Candidate~$d$ is the despised one whose victory we want to
  preclude.  We let the set of the additional (unregistered)
  candidates be
  \[
    A = V(G) \cup E(G).
  \]
  That is, $A$ contains all vertex candidates and all edge
  candidates.  We set the bundling function $\combRule$ so that for
  each edge candidate~$e$ whose corresponding edge is incident to $u$ and $v$,
  we have $\combRule(e) = \{e,u,v\}$,
  and for each vertex candidate $v$ we have $\combRule(v) = \{v\}$.
  We introduce the following voters:
  \begin{enumerate}
  \item For each pair~$i, j$, $i \in [h]$, $j \in [h]$, $i\neq j$,
    of distinct colors,
    we have one voter with the following preference order,
    where we write $E(\{i,j\})$ to mean an arbitrarily chosen order
    over the edge candidates that link vertices of color $i$ with
    those of color $j$; the first occurrence of ``$\cdots$'' regards
    the candidates in $\{d^{\{i,j\}}_z \mid z \in [t-1]\}$ only:
    \[
      E(\{i,j\}) \pref d^{\{i,j\}}_1 \pref \cdots \pref d^{\{i,j\}}_{t-1} \pref d \pref \cdots.
    \]   
    Note that in the initial election,
    $d$ gets a point from this voter, 
    but it is sufficient (and we will make sure that it is also
    necessary) to add one candidate from $E(\{i,j\})$ to prevent~$d$ from
    getting this point.
  \item For each color $i$, $1 \leq i \leq h$, we have a voter with
    the following preference order (recall that $V_{i}(G)$ consists of all
    vertex candidates that correspond to the vertices of the same color~$i$;
    the first occurrence of ``$\cdots$'' regards
    the candidates in $\{d^{(i)}_z \mid z \in [t-2]\}$ only):
    \[
    V_i(G) \pref d^{(i)}_1 \pref \cdots \pref d^{(i)}_{t-2} \pref p
    \pref d^{(i)}_{t-1} \pref \cdots.
    \]
    Note that in the initial election $p$ gets a point from this
    voter, but if more than one candidate from $V_i(G)$ is added, then
    $p$ does not gain this point.
  \item For each number $i \in [h]$, we have a voter with the
    following preference order (the first occurrence of ``$\cdots$''
    regards the candidates in $\{e^{(i)}_z \mid z \in [t-1]\}$ only):
    \[
      d \pref e^{(i)}_1 \pref \cdots \pref e^{(i)}_{t-1} \pref \cdots.
    \]
    Note that, altogether, $d$ gets $h$ points from the voters in this group.
  \end{enumerate}
  First, prior to adding any candidates, $d$ has
  $h+\binom{h}{2}$~points while $p$ has $h$~points, and each of the
  dummy candidates has one point.  We show next that it is possible to
  ensure that $d$ is not a winner of this election by adding at most
  $k \coloneqq \binom{h}{2}$~(bundles of) candidates if and only if $G$ has a
  multi-colored clique of order~$h$.

  It follows now that if there is a multi-colored
  clique in $G$, then adding the edge candidates corresponding to the
  edges of this clique ensures that $d$ is not a winner.  

  For the reverse direction, assume that it is possible to ensure that
  $d$ is not a winner by adding at most $\binom{h}{2}$~(bundles of) candidates.
  $p$ is the only candidate that can reach a higher score than $d$ this way.
  For this to happen, $d$ must lose all the
  points that $d$ initially got from the first group of voters, and
  $p$ must get all the points from the second group of voters.
  Moreover, adding voters corresponding to vertices does not help.
  Thus, this must correspond to adding $\binom{h}{2}$ edge candidates
  whose bundles do not add two vertices of the same color. That is,
  these $\binom{h}{2}$ added edge candidates must correspond to a
  multi-colored clique of order $h$.
\end{proof}

We conclude this section by mentioning that the following results also
follow by applying the Multi-Colored Clique technique. The
proofs are available in Appendix~\ref{app:mcc}.

\newcommand{\cortapprovalccac}{For each fixed integer $t$, $t \geq 2$,
  $t$-Approval-\textsc{CCAC}, parameterized by the number of voters, is
  $\wone$-hard.}

\begin{theorem}\label{cor:t-approval-ccac}
  \cortapprovalccac
\end{theorem}

\newcommand{\cortvetoccac}{
  For each fixed integer $t$, $t \geq 2$, $t$-Veto-\textsc{CCAC}, parameterized
  by the number of voters, is $\wone$-hard.
}

\begin{theorem}\label{cor:t-veto-ccac}
  \cortvetoccac
\end{theorem}

\newcommand{\lemkvetoccdc}{For each fixed integer~$\constantk\ge 1$,
  $\constantk$-Veto-\textsc{CCDC}, parameterized by the number of voters, is
  $\wone$-hard.}

\begin{theorem}\label{lem:k-veto-ccdc}
  \lemkvetoccdc
\end{theorem}

\newcommand{\lemtwoapprovalccdc}{$2$-Approval-\textsc{CCDC}, parameterized by
  the number of voters, is $\wone$-hard.}

\begin{theorem}\label{lem:2-approval-ccdc}
  \lemtwoapprovalccdc
\end{theorem}

\newcommand{\lemkapprovalccdc}{For each fixed integer~$\constantk$,
  $\constantk \geq 3$, $\constantk$-Approval-\textsc{CCDC}, parameterized by
  the number of voters, is $\wone$-hard.}

\begin{theorem}\label{lem:k-approval-ccdc}
  \lemkapprovalccdc
\end{theorem}

\section{\cvctech}\label{section_formal_proofs_cvc}
\label{sec:cvc}

We now move on to the Cubic Vertex Cover proof
technique. Specifically, we use it to obtain the following results
(again, all results are for the parameterization by the number of
voters):
\begin{enumerate}
\item Borda-\textsc{CCAC} and Borda-CCDC are $\np$-hard (this holds already
  for elections with ten voters).
\item For each rational $\alpha$, $0 \leq \alpha \leq 1$,
  Copeland$^\alpha$-\textsc{CCAC} and Copeland$^\alpha$-CCDC are $\np$-hard
  (this holds already for elections with twenty and twenty-six voters,
  respectively).
\item Maximin-\textsc{CCAC} is $\np$-hard (this holds already for elections
  with ten voters).
\end{enumerate}
In other words, we use the Cubic Vertex Cover technique for all our
non-combinatorial $\paranp$-hardness results.
In this section we provide proofs for the cases of Borda-CCDV
and Maximin-\textsc{CCAC}, while the remaining ones are in Appendix~\ref{app:cvc}.  We
made this choice because the proofs for Borda-CCDV and Maximin-\textsc{CCAC}
illustrate the essential elements of the technique (as applied both to
an adding-candidates case and to a deleting-candidates case, and both
to a scoring rule and a Condorcet consistent rule).

The general idea of the Cubic Vertex Cover technique is to prove
$\paranp$-hardness via reductions from the \probCVC problem (known to
be $\np$-hard~\shortcite{GareyJohnsonStockmeyer1976}), using the fact that cubic
graphs\footnote{In a cubic graph, each vertex is of degree exactly
  three, that is, it has exactly three neighbors.}  can be encoded
using a constant number of votes. Formally, the \probCVC problem is
defined as follows.

\begin{definition}
  An instance of \probCVC consists of an undirected graph
  $G = (V(G),E(G))$,
  where each vertex of $G$ has degree exactly three, and a non-negative
  integer $h$. We ask if there is a subset (\emph{vertex cover}) of at most
  $h$~vertices such that each edge is incident to at least one vertex
  in the subset. 
\end{definition}

All the reductions in this section use the following common setup. Let
$I$ be an instance of \probCVC with a graph $G$ and non-negative
integer $h$. From a classic result by~\citet{vizing1965critical}, we
know that there is an edge-coloring of $G$ with four colors (that is,
it is possible to assign one out of four colors to each edge so that
no two edges incident to the same vertex have the same
color). Further, it is possible to compute this coloring in polynomial
time~\shortcite{mis-gri:j:vizing}. This is equivalent to saying that it is
possible to decompose the set of $G$'s edges into four disjoint
matchings.  Our reductions start by computing this decomposition.  We
rename the edges of $G$ so that these four disjoint matchings are:
\begin{align*}
  E^{(1)} & = \{e^{(1)}_{1}, \ldots, e^{(1)}_{m_1}\},  \\
  E^{(2)} & = \{e^{(2)}_{1}, \ldots, e^{(2)}_{m_2}\},  \\
  E^{(3)} & = \{e^{(3)}_{1}, \ldots, e^{(3)}_{m_3}\},  \\
  E^{(4)} & = \{e^{(4)}_{1}, \ldots, e^{(4)}_{m_4}\}.  
\end{align*}
We set $\edgesetsize = m_1 + m_2 + m_3 + m_4 = |E(G)|$ and $\vertexsetsize = |V(G)|$.
For each edge~$e$ of the graph, we arbitrarily order its vertices and
we write $v'(e)$ and $v''(e)$ to refer to the first vertex and to the
second vertex, respectively.
For each $\ell$, $1 \leq \ell \leq 4$, we write $E^{(-\ell)}$ to mean
$E(G) \setminus E^{(\ell)}$. We write $V^{(-\ell)}$ to mean the set of
vertices that are not incident to any of the edges in $E^{(\ell)}$.

The crucial point of our approach is to use the above decomposition to
create eight votes (two for each matching) that encode the graph.  We
will now provide useful notation for describing these eight votes. For
each edge $e$ of the graph, we define the following four orders over
$e$, $v'(e)$, and $v''(e)$:
\begin{align*}
 P(e) \colon  & e \pref v'(e) \pref v''(e), \\
 P'(e) \colon & e \pref v''(e) \pref v'(e), \\
 Q(e) \colon  & v'(e) \pref v''(e) \pref e, \\
 Q'(e) \colon  & v''(e) \pref v'(e) \pref e.
\end{align*}
For each $\ell$, $1 \leq \ell \leq 4$, we define the following orders
over $V(G) \cup E(G)$:
\begin{align*}
  A(\ell) \colon& P(e^{(\ell)}_1) \pref P(e^{(\ell)}_2) \pref \cdots \pref P(e^{(\ell)}_{m_\ell}), \\
  A'(\ell) \colon& P'(e^{(\ell)}_{m_\ell}) \pref \cdots \pref P'(e^{(\ell)}_2) \pref \cdots \pref P'(e^{(\ell)}_{1}),\\
  B(\ell) \colon& Q(e^{(\ell)}_1) \pref Q(e^{(\ell)}_2) \pref \cdots \pref Q(e^{(\ell)}_{m_\ell}), \\
  B'(\ell) \colon& Q'(e^{(\ell)}_{m_\ell}) \pref \cdots \pref Q'(e^{(\ell)}_2) \pref \cdots \pref Q'(e^{(\ell)}_{1}).
\end{align*}
Note that since each $E^{(\ell)}$ is a matching, each of the above
orders is well-defined.  The first two of these families of orders
(that is, $A(\ell)$ and $A'(\ell)$) will be useful in the hardness
proofs for the cases of deleting candidates, and the latter two (that
is, $B(\ell)$ and $B'(\ell)$) in the hardness proofs for the cases of
adding candidates.
The intuitive idea behind orders $A(\ell)$ and $A'(\ell)$ (or $B(\ell)$
and $B'(\ell)$) is that, at a high level, they are reverses of each
other, but they treat edges and their endpoints in a slightly asymmetric way
(we will describe this in detail in the respective proofs).

We are ready to show examples of applying the Cubic Vertex Cover
technique. We start with the case of Borda-CCDC (we present the
theorem and its proof first, and right after that, we show an example of
applying the reduction).

\begin{theorem}\label{lem:borda-ccdc}
  Borda-\textsc{CCDC} is $\np$-hard,
  even for elections with only ten voters.
\end{theorem}

\begin{proof}
  We give a reduction from the \probCVC problem. Let $I$ be our
  input instance that contains graph $G = (V(G),E(G))$ and
  non-negative integer $h$. We use the notation introduced in the
  beginning of the section.
  We form an election $E = (C,V)$, where $C = \{p,d\} \cup V(G) \cup E(G)$.
  We introduce the following ten voters:
  \begin{enumerate}
  \item For each $\ell$, $1 \leq \ell \leq 4$, we have the following two voters:
    \begin{align*}
      \mu(\ell) \colon & A(\ell) \pref E^{(-\ell)} \pref V^{(-\ell)} \pref d \pref p, \\
      \mu'(\ell) \colon & p \pref d \pref \revnot{V^{(-\ell)}} \pref \revnot{E^{(-\ell)}} \pref A'(\ell). 
    \end{align*}
  \item We have one voter with preference order $p \pref d \pref V(G)
    \pref E(G)$ and one voter with preference order $\revnot{E(G)} \pref
    \revnot{V(G)} \pref p \pref d$.
  \end{enumerate}
  We claim that $p$ can become a winner of this election by deleting
  at most $k \coloneqq h$~candidates if and only if there is a vertex cover of
  size~$h$ for~$G$.

  Let us first calculate the scores of all the candidates:
  \begin{enumerate}
  \item Candidate $p$ has $5(\vertexsetsize+\edgesetsize)+6$ points
    (that is, $4(\vertexsetsize+\edgesetsize+1)$ points from the first
    eight voters and $\vertexsetsize+\edgesetsize+2$ points from the
    last two voters).

  \item Each vertex candidate $v$ has
    $5(\vertexsetsize+\edgesetsize)+2$ points (for each of the three
    pairs of voters $\mu(\ell)$, $\mu'(\ell)$, $1 \leq \ell \leq 4$,
    such that $v$ is incident to some edge in $E^{(\ell)}$, $v$ gets
    $\vertexsetsize+\edgesetsize$ points; $v$ gets
    $\vertexsetsize+\edgesetsize+1$ points from the remaining pair of
    voters in the first group and, additionally,
    $\vertexsetsize+\edgesetsize+1$ points from the last two voters).

  \item Each edge candidate $e$ has $5(\vertexsetsize+\edgesetsize)+7$
    points, that is, $\vertexsetsize+\edgesetsize+3$ points from the
    pair of voters $\mu(\ell)$, $\mu'(\ell)$ such that $e \in
    E^{(\ell)}$, $\vertexsetsize+\edgesetsize+1$ points from each pair of
    the remaining three pairs of voters in the first group, and
    $\vertexsetsize+\edgesetsize+1$ points from the last two voters.

  \item Candidate $d$ has $5(\vertexsetsize+\edgesetsize)+4$ points
    (that is, $4(\vertexsetsize+\edgesetsize+1)$ points from the
    voters in the first group and $\vertexsetsize+\edgesetsize$ points
    from the last two voters.
  \end{enumerate}
  Prior to deleting any of the candidates, $p$ is not a
  winner because the edge candidates have higher scores. However, the
  score of~$p$ is higher than the score of the vertex candidates and
  the score of~$d$.

  We now describe how deleting candidates affects the scores of the
  candidates. Let $v$ be some vertex candidate. Deleting $v$ from our
  election causes the following effects: The score of each edge
  candidate $e$ such that $v = v'(e)$ or $v = v''(e)$ decreases by
  six; the score of each remaining candidate decreases by five.
  This means that if we delete $h$ vertex candidates that correspond
  to a vertex cover of $G$, then the scores of $p$, $d$, and all the
  vertex candidates decrease by $5h$, while the scores of all edge
  candidates decrease by at least $5h+1$.  As a result, we have $p$ as
  a winner of the election.

  For the reverse direction, assume that it is possible to ensure $p$'s
  victory by deleting at most $h$ candidates. Deleting candidate $d$
  decreases the score of $p$ by six, whereas it decreases the scores
  of every other candidate by five. Thus, we can assume that
  there is a solution that does not delete $d$. Similarly, one can verify 
  that if there is a solution that deletes some edge $e$, then
  a solution that is identical but instead of $e$ deletes either
  $v'(e)$ or $v''(e)$ (it is irrelevant which one) is also correct. We
  conclude that it is possible to ensure $p$'s victory by deleting at
  most $h$ vertex candidates.  However, by the discussion of the
  effects of deleting vertex candidates and the fact that prior to any
  deleting each edge candidate has one point more than $p$, we have
  that these at-most-$h$ deleted vertex candidates must correspond to
  a vertex cover of $G$. This completes the proof.
\end{proof}

\begin{example}
  We provide an example for the reduction described in the proof
  of~\autoref{lem:borda-ccdc}.  The input graph is depicted
  in~\autoref{figureexamplecvc} and we take $h := 4$. We present the
  constructed election in~\autoref{tabularexamplecvc}. The election
  that results from deleting candidates $\{v_1, v_3, v_4, v_6\}$ that
  correspond to a vertex cover is presented
  in~\autoref{tabular2examplecvc}.

  \begin{figure}
  \begin{center}
    \begin{tikzpicture}[scale=0.95]
      \scriptsize
      \tikzset{node/.style={circle,draw=black!80,minimum size=12pt,inner sep=0pt}}

      \node[node] (v1) at (0,2) {$v_1$};
      \node[node] (v2) at (6,2) {$v_2$};
      \node[node] (v3) at (2,1) {$v_3$};
      \node[node] (v4) at (4,1) {$v_4$};
      \node[node] (v5) at (0,0) {$v_5$};
      \node[node] (v6) at (6,0) {$v_6$};
      
      \draw (v1) to node [auto] {$(1)$} (v2);
      \draw (v5) to node [auto] {$(1)$} (v3);
      \draw (v4) to node [auto] {$(1)$} (v6);
      
      \draw (v3) to node [auto] {$(2)$} (v4);
      \draw (v5) to node [below] {$(2)$} (v6);
      
      \draw (v3) to node [auto] {$(3)$} (v1);
      \draw (v2) to node [auto] {$(3)$} (v4);
      
      \draw (v1) to node [left] {$(4)$} (v5);
      \draw (v2) to node [right] {$(4)$} (v6);
    \end{tikzpicture}
  \end{center}
  \caption{The input graph for our example
    for~\autoref{lem:borda-ccdc}.  The numbers in parentheses
    represent the colors of the edges according to an assumed
    partition to four colors.  For example, the edge $\{v_1, v_3\}$ has
    the third color.  }
  \label{figureexamplecvc}
  \end{figure}
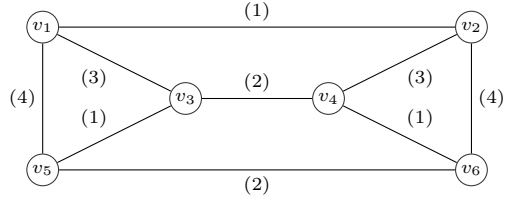

  \begin{table}
    \begin{center}
      \begin{align*}
      \mu(1)  \colon & \{v_1, v_2\} \pref v_1 \pref v_2 \pref \{v_3, v_5\} \pref v_3 \pref v_5 \pref \{v_4, v_6\} \pref v_4 \pref v_6 \pref \\
                     & E^{(-1)} \pref V^{(-1)} \pref d \pref p \\[2mm]
      \mu'(1) \colon & p \pref d \pref \revnot{V^{(-1)}} \pref \revnot{E^{(-1)}} \pref \\
                     & \{v_4, v_6\} \pref v_6 \pref v_4 \pref \{v_3, v_5\} \pref v_5 \pref v_3 \pref \{v_1, v_2\} \pref v_2 \pref v_1 \\[2mm]
      \mu(2)  \colon & \{v_3, v_4\} \pref v_3 \pref v_4 \pref \{v_5, v_6\} \pref v_5 \pref v_6  \pref \\
                     & E^{(-2)} \pref V^{(-2)} \pref d \pref p \\[2mm]
      \mu'(2) \colon & p \pref d \pref \revnot{V^{(-2)}} \pref \revnot{E^{(-2)}} \pref \\
                     & \{v_5, v_6\} \pref v_6 \pref v_5 \pref \{v_3, v_4\} \pref v_4 \pref v_3 \\[2mm]
      \mu(3)  \colon & \{v_1, v_3\} \pref v_1 \pref v_3 \pref \{v_2, v_4\} \pref v_2 \pref v_4 \pref \\
                     & E^{(-3)} \pref V^{(-3)} \pref d \pref p \\[2mm]
      \mu'(3) \colon & p \pref d \pref \revnot{V^{(-3)}} \pref \revnot{E^{(-3)}} \pref \\
                     & \{v_2, v_4\} \pref v_4 \pref v_2 \pref \{v_1, v_3\} \pref v_3 \pref v_1 \\[2mm]
      \mu(4)  \colon & \{v_1, v_5\} \pref v_1 \pref v_5 \pref \{v_2, v_6\} \pref v_2 \pref v_6 \pref \\
                     & E^{(-4)} \pref V^{(-4)} \pref d \pref p \\[2mm]
      \mu'(4) \colon & p \pref d \pref \revnot{V^{(-4)}} \pref \revnot{E^{(-4)}} \pref \\
                     & \{v_1, v_5\} \pref v_5 \pref v_1 \pref \{v_2, v_6\} \pref v_6 \pref v_2 \\[2mm]
      \text{one voter } \colon               & p \pref d \pref V(G) \pref E(G) \\
      \text{one voter } \colon               & \revnot{E(G)} \pref \revnot{V(G)} \pref p \pref d
      \end{align*}
    \end{center}
    \caption{The election constructed in the proof of
      \autoref{lem:borda-ccdc} for the input graph from
      \autoref{figureexamplecvc}.  }
    \label{tabularexamplecvc}
\end{table}

  \begin{table}
    \begin{center}
      \begin{align*}
      \mu(1)  \colon & \{v_1, v_2\} \pref \xcancel{v_1} \pref v_2 \pref \{v_3, v_5\} \pref \xcancel{v_3} \pref v_5 \pref \{v_4, v_6\} \pref \xcancel{v_4} \pref \xcancel{v_6} \pref \\
                     & E^{(-1)} \pref V^{(-1)} \pref d \pref p \\[2mm]
      \mu'(1) \colon & p \pref d \pref \revnot{V^{(-1)}} \pref \revnot{E^{(-1)}} \pref \\
                     & \{v_4, v_6\} \pref \xcancel{v_6} \pref \xcancel{v_4} \pref \{v_3, v_5\} \pref v_5 \pref \xcancel{v_3} \pref \{v_1, v_2\} \pref v_2 \pref \xcancel{v_1} \\[2mm]
      \mu(2)  \colon & \{v_3, v_4\} \pref \xcancel{v_3} \pref \xcancel{v_4} \pref \{v_5, v_6\} \pref v_5 \pref \xcancel{v_6} \pref \\
                     & E^{(-2)} \pref V^{(-2)} \pref d \pref p \\[2mm]
      \mu'(2) \colon & p \pref d \pref \revnot{V^{(-2)}} \pref \revnot{E^{(-2)}} \pref \\
                     & \{v_5, v_6\} \pref \xcancel{v_6} \pref v_5 \pref \{v_3, v_4\} \pref \xcancel{v_4} \pref \xcancel{v_3} \\[2mm]
      \mu(3)  \colon & \{v_1, v_3\} \pref \xcancel{v_1} \pref \xcancel{v_3} \pref \{v_2, v_4\} \pref v_2 \pref \xcancel{v_4} \pref \\
                     & E^{(-3)} \pref V^{(-3)} \pref d \pref p \\[2mm]
      \mu'(3) \colon & p \pref d \pref \revnot{V^{(-3)}} \pref \revnot{E^{(-3)}} \pref \\
                     & \{v_2, v_4\} \pref \xcancel{v_4} \pref v_2 \pref \{v_1, v_3\} \pref \xcancel{v_3} \pref \xcancel{v_1} \\[2mm]
      \mu(4)  \colon & \{v_1, v_5\} \pref \xcancel{v_1} \pref v_5 \pref \{v_2, v_6\} \pref v_2 \pref \xcancel{v_6} \pref \\
                     & E^{(-4)} \pref V^{(-4)} \pref d \pref p \\[2mm]
      \mu'(4) \colon & p \pref d \pref \revnot{V^{(-4)}} \pref \revnot{E^{(-4)}} \pref \\
                     & \{v_1, v_5\} \pref v_5 \pref \xcancel{v_1} \pref \{v_2, v_6\} \pref \xcancel{v_6} \pref v_2 \\[2mm]
      \text{one voter } \colon               & p \pref d \pref V(G) \pref E(G) \\
      \text{one voter } \colon               & \revnot{E(G)} \pref \revnot{V(G)} \pref p \pref d
      \end{align*}
    \end{center}
    \caption{The election from \autoref{tabularexamplecvc} with the candidates corresponding
  to the vertex cover $\{v_1, v_3, v_4, v_6\}$ deleted.
    }
    \label{tabular2examplecvc}
\end{table}
\end{example}

Let us now show an application of the Cubic Vertex Cover technique to
the case of adding candidates. Specifically, we consider Maximin-\textsc{CCAC}.

\begin{theorem}\label{lem:maximin-ccac}
  Maximin-\textsc{CCAC} is $\np$-hard,
  even for elections with only ten voters.
\end{theorem}

\begin{proof}
  We give a reduction from \probCVC (we use the notation as provided
  at the beginning of this section).  Given an instance~$(G,h)$ for
  \probCVC, we construct an instance for Maximin-\textsc{CCAC}.  We
  let the registered candidate set~$C$ be $\{p\} \cup E(G)$, and we
  let $V(G)$ be the set of unregistered candidates.  We construct ten
  voters:
  \begin{enumerate}
  \item For each $\ell$, $1 \leq \ell \leq 4$, we have the following
    two voters:
    \begin{align*}
      \mu(\ell) \colon & B(\ell) \pref E^{(-\ell)} \pref V^{(-\ell)} \pref p, \\
      \mu'(\ell) \colon & p \pref \revnot{V^{(-\ell)}} \pref \revnot{E^{(-\ell)}} \pref B'(\ell). 
    \end{align*}
  \item We have one voter with preference order $E(G) \pref p \pref
    V(G)$ and one voter with preference order $\revnot{E(G)} \pref p
    \pref \revnot{V(G)}$.
  \end{enumerate}
  
  Let $E$ be the thus-constructed election (including all 
  registered and unregistered candidates). We have the following values
  of the $N_E(\cdot,\cdot)$ function (recall that this function represents the head-to-head contests):
  \begin{enumerate}
  \item For each vertex $v \in V(G)$, we have $N_E(p,v) = 6$ (and thus, $N_E(v,p) = 4$).
  \item For each edge $e \in E(G)$, we have $N_E(p,e) = 4$ (and thus, $N_E(e,p) = 6$).
  \item For each vertex $v \in V(G)$ and each edge $e \in E(G)$ we
    have the following: If $v$ is an endpoint of $e$, then $N_E(v,e) =
    6$ (so $N_E(e,v) = 4$), and otherwise we have $N_E(v,e) = 5$ (so
    $N_E(e,v) = 5$).
  \item For each pair of vertices, $v',v'' \in V(G)$, $N_E(v',v'') = 5$.
  \item For each pair of edges, $e',e'' \in E(G)$,  $N_E(e',e'') = 5$.
  \end{enumerate}
  In effect, prior to adding the candidates, the score of~$p$ is four
  and the score of each edge candidate is five.  Adding a vertex
  candidate $v$ to the election does not change the score of $p$, but
  decreases the score of each edge candidate that has $v$ as an
  endpoint to four. Further, this added vertex candidate has score
  four as well. Thus, it is possible to ensure
  $p$'s victory by adding at most $h$ candidates if and only if there
  is a size-$h$ vertex cover for $G$.
\end{proof}

We conclude the section by mentioning that the following results,
whose proofs are in Appendix~\ref{app:cvc}, also follow by applying the Cubic
Vertex Cover technique.

\newcommand{\lembordaccac}{Borda-\textsc{CCAC} is $\np$-hard, even for
  elections with only ten voters.}

\begin{theorem}\label{lem:borda-ccac}
\lembordaccac
\end{theorem}

\newcommand{\lemcopelandccac}{ For each rational number~$\alpha$, $0
  \leq \alpha \leq 1$, Copeland$^\alpha$-\textsc{CCAC} is $\np$-hard, even for
  elections with only twenty voters.}

\begin{theorem}\label{lem:copeland-ccac}
  \lemcopelandccac
\end{theorem}

\newcommand{\lemcopelandccdc}{For each rational number~$\alpha$, $0
  \leq \alpha \leq 1$, Copeland$^\alpha$-\textsc{CCDC} is $\np$-hard, even for
  elections with only twenty-six voters.}

\begin{theorem}\label{lem:copeland-ccdc}
  \lemcopelandccdc
\end{theorem}

\section{\setcovertech for Combinatorial Variants}\label{section_formal_proofs_sc}
\label{sec:set}

In this section we present our Set-Embedding proof technique for the
combinatorial variants of our control problems.  Specifically, we
prove the following statements (again, all results are for the
parameterization by the number of voters):
\begin{enumerate}
  \item For each fixed integer~$\constantk \geq 1$
  and for each voting rule {$\calR \in \{$}$\constantk$-Approval, $\constantk$-Veto, Borda, Copeland$^\alpha$ (for \ouralpha), Maximin{$\}$},
  both $\calR$-\textsc{Comb}-CCDC and $\calR$-\textsc{Comb}-DCDC
  are \paranphl.
\item For each voting rule {$\calR \in \{$}Borda, Copeland$^\alpha$
  (for \ouralpha){$\}$}, Maximin{$\}$}, $\calR$-\textsc{Comb}-\textsc{CCAC}
  is \paranphl.
  \item For each voting rule {$\calR \in \{$}Borda, Copeland$^\alpha$ (for \ouralpha){$\}$},
  $\calR$-\textsc{Comb}-DCAC
  is \paranphl.
\end{enumerate}
That is, in this section we provide all our $\paranp$-hardness results
for the combinatorial variants of our problems.

All proofs follow by reducing the \probSetCover problem (recall Definition~\ref{def:sc}) to the respective problem in a way which uses
the bundling function to encode the sets from the \probSetCover instances
(hence the name of the technique).  We start by providing some common
notation and observations common to all of these results.

Let $I = (X, \calS, h)$ be an input instance of \probSetCover (which is $\np$-hard~\shortcite{gar-joh:b:int}). 
We construct elections with candidate sets that include
the elements from~$X$ and the sets from~$\calS$. Specifically, for
each element $x_i \in X$, we introduce a candidate with the same name,
and for each set $S_j \in \calS$, we introduce a candidate named $s_j$.  We
denote the set of all element candidates by $\ElementCandidateSet$ and
denote the set of all set candidates by $\SetCandidateSet$.
Further, we will typically have
candidates $p$ and $d$.  For the constructive cases, $p$ will be the
preferred candidate while for the destructive cases, $d$ will be the
despised one.

Unless stated otherwise, in each of our proofs we use a bundling
function $\combRule$ defined as follows: for each set candidate $s_j$,
we have $\combRule(s_j) = \{s_j\} \cup \{x_i \mid x_i \in S_j\}$, and
for each non-set candidate $c$, we have $\combRule(c) = \{c\}$.  We
refer to this bundling function as the \emph{set-embedding bundling
  function}.

The general idea of our proofs is that to ensure $p$'s victory (for
the constructive cases) or $d$'s defeat (for the destructive cases),
one has to add/delete all the candidates from $\ElementCandidateSet$,
and due to the bound on the number of candidates that we can
add/delete, this has to be achieved by adding/deleting the candidates
from $\SetCandidateSet$ and relying on the bundling function.

With the above setup ready, we move on to proving our results.  Most
of the proofs are in Appendix~\ref{app:set}, but for each type of
problem (\textsc{Comb}-\textsc{CCAC}, \textsc{Comb}-CCDC, \textsc{Comb}-DCAC,
\textsc{Comb}-DCDC) we give one sample proof.

\subsection{Constructive Control by Deleting Candidates}
We start by looking at constructive control by deleting candidates
because in this case we obtain a very general hardness result that
applies to all the voting rules which satisfy the unanimity
principle.  A rule satisfies the \emph{unanimity principle} if in each
election where a unique candidate $c$ is ranked first by all the
voters, this candidate $c$ is the unique winner.

\newcommand{\lemrcccdc}{Let $\calR$ be a voting rule that satisfies the unanimity principle.
  $\calR$-\textsc{Comb}-\textsc{CCDC} is $\np$-hard, even for the case of elections with just
  a single voter.}

\begin{theorem}\label{lem:r-cccdc}
 \lemrcccdc
\end{theorem}

\begin{proof}
  Let the notation be as in the introduction to this
  section.  Given an instance~$I \coloneqq \probSetCoverInstance$ for
  \probSetCover, we create an instance $I'$ of $\calR$-\textsc{Comb}-CCDC as
  follows.  We construct an election $E = (C,V)$ where $C = \{p\} \cup
  \ElementCandidateSet \cup \SetCandidateSet$ and where $V$~contains a
  single voter with the following preference order:
  \[
    \ElementCandidateSet \pref p \pref \SetCandidateSet.
  \]
  We use the set-embedding bundling function. We claim that $I$ is a
  ``yes''-instance of \probSetCover if and only if it is possible to
  ensure $p$'s victory by deleting at most $h$ (bundles of) candidates.

  On one hand, if $I$ is a ``yes''-instance of $\probSetCover$, then
  $I'$ is a ``yes''-instance of $\calR$-\textsc{Comb}-CCDC.  Indeed, if
  $\calS'$ is a subfamily of $\calS$ such that $|\calS'| \leq h$ and
  $\bigcup_{S_j \in \calS'}S_j = X$, then it suffices to delete the
  candidates $C'$ that correspond to the sets in $\calS'$ from the
  election to ensure that $p$ is ranked first (and, by the unanimity
  of $\calR$, is a winner).

  On the other hand, assume that $I'$ is a ``yes''-instance of
  $\calR$-\textsc{Comb}-CCDC.  Since $\calR$ satisfies the unanimity
  property, the candidate ranked first by the only voter in our
  election is always the unique winner.  This means that if $I'$ is a
  ``yes''-instance of $\calR$-\textsc{Comb}-CCDC, then there is a
  subset~$C'$ of candidates such that $p \notin C'$ and $X \subseteq
  \bigcup_{c \in C'}\combRule(c)$.  Without loss of generality, we can
  assume that $C'$ contains only candidates from the set~$\SetCandidateSet$ (if $C'$ contained some candidate $x_i$, we could
  replace $x_i$ with an arbitrary candidate $s_j$ such that $x_i \in
  S_j$).  However, this immediately implies that setting
  $\calS' \coloneqq \{S_j \mid s_j \in C'\}$ results in a set cover of size at
  most $h$.  Therefore $I$ is a ``yes''-instance of~$I$.
\end{proof}

As Plurality, Borda, Copeland$^\alpha$, and Maximin all satisfy the
unanimity property, we conclude the following.

\begin{corollary}\label{lem:plurality-borda-copeland-maximin-comb-ccdc-np-h-1}
  For each voting rule
  {$\calR \in \{$}Plurality, Borda, Copeland$^\alpha$, Maximin{$\}$},
  $\calR$-\textsc{Comb}-CCDC is $\np$-hard, even for elections with only a single voter.
\end{corollary}

By applying minor tweaks to the above construction, we obtain the
following results (for the proofs see Appendix~\ref{app:set}).

\newcommand{\lemapprovalcombccdcnphone}{For each fixed integer~$\constantk \ge 2$,
  $\constantk$-Approval-\textsc{Comb}-\textsc{CCDC} is $\np$-hard, even for
  elections with only a single voter.}

\begin{theorem}\label{lem:approval-comb-ccdc-np-h-1}
  \lemapprovalcombccdcnphone
\end{theorem}

\newcommand{\lemvetocombccdcnphone}{For each fixed integer~$\constantk
  \geq 1$, $\constantk$-Veto-\textsc{Comb}-\textsc{CCDC} is $\np$-hard, even for
  elections with only a single voter.}

\begin{theorem}\label{lem:veto-comb-ccdc-np-h-1}
  \lemvetocombccdcnphone
\end{theorem}

\subsection{Destructive Control by Deleting Candidates}
While the very general proof for the combinatorial variant of
constructive control by deleting candidates is very simple,
occasionally our set-embedding proofs become slightly more
involved. For example, our proof that Maximin-\textsc{Comb}-DCDC is
$\np$-hard even for elections with only few voters requires a bit more
care.

\begin{theorem}\label{sc18}
  Maximin-\textsc{Comb}-\textsc{DCDC} is $\np$-hard,
  even for elections with only five voters.
\end{theorem}

\begin{proof}
  Given an instance~$\probSetCoverInstance$ for \probSetCover, we
  construct an instance~$(E=(C, V), k)$ for Maximin-\textsc{Comb}-DCDC.
  We construct an election~$E=(C, V)$ where $C \coloneqq \{p,d,e\} \cup \ElementCandidateSet \cup \SetCandidateSet$ and where the voter set consists of the following five voters:
  \begin{align*}
    \text{one voter}\colon & p\pref d \pref \ElementCandidateSet \pref e \pref \SetCandidateSet, \\
    \text{two voters}\colon &  d\pref \ElementCandidateSet \pref p \pref e \pref \SetCandidateSet,  \\
    \text{two voters}\colon &  e \pref \overleftarrow{\ElementCandidateSet} \pref p \pref d \pref \overleftarrow{\SetCandidateSet}. 
  \end{align*}
  We set $k\coloneqq h$.
  We use the set-embedding bundling functions. We claim that $I$ is a
  ``yes''-instance of \probSetCover if and only if it is possible to
  ensure that $d$ is not a winner by deleting at most $k=h$ (bundles of)
  candidates.

  The values of the $N_E(\cdot,\cdot)$ function are given in the table
  below (the entry for row~$a$ and column~$b$ gives the value of
  $N_E(a,b)$; we assume $i' \neq i''$ and $j' \neq j''$).
  \begin{center}
  \begin{tabular}[h]{l|ccccc}
    & $p$ & $d$ & $e$ & $x_{i'}$ & $s_{j'}$\\\hline
    $p$ & - & $3$ & $3$ & $1$ & $5$\\
    $d$ & $2$ & - & $3$ & $3$ & $5$\\
    $e$ & $2$ & $2$ & - & $2$ & $5$\\
    $x_{i''}$ & $4$ & $2$ & $3$ & $2$ or $3$ & $5$\\
    $s_{j''}$ & $0$ & $0$ & $0$ & $0$ & $2$ or $3$\\
  \end{tabular}
  \end{center}
  We have the following scores of the candidates: $p$ has one point
  (because of the members of~$\ElementCandidateSet$), $d$ has two
  points (because of $p$), $e$ has two points (because of $p$, $d$,
  and the members of~$\ElementCandidateSet$), the members of
  $\ElementCandidateSet$ have two points each (because of $d$), and
  the members of $\SetCandidateSet$ have zero points each (because of
  all other candidates).

  Now, if there is a set cover for $I$ of size
  $h$, then deleting the set candidates corresponding to the cover
  deletes all members of $\ElementCandidateSet$ and ensures that
  $p$ has three points, whereas $d$ has only two. In effect, $d$
  certainly is not a winner.

  Now consider the other direction.  Since deleting a candidate can
  never decrease the score of any remaining candidate, the only way of
  making $d$ lose is to increase some remaining candidate's score.

  Since for each candidate other than $p$, at least three voters
  prefer $d$ to this candidate, only $p$ has any chance of getting
  a higher score than $d$.  For this to happen, we need to ensure that
  all members of $\ElementCandidateSet$ disappear. As in the previous
  set-embedding proofs, this is possible to do by deleting at most~$h$
  candidates only if there is a set cover of size at most $h$ for~$I$.
\end{proof}

Yet, for most of the other results it suffices to use proofs very
similar to that for Theorem~\ref{lem:r-cccdc}. However, for the case
of $\constantk$-Approval-\textsc{Comb}-DCDC we have to use either two
voters (if $t \geq 2$) or three voters (if $t=1$ and we are dealing
with Plurality). The reason is that if we have a single voter and
candidate $d$ is a $\constantk$-Approval winner, then it is impossible
to prevent $d$ from winning by deleting candidates (no matter what we
do, $d$ will still have the highest possible score, one). A similar
reasoning applies to the case of Plurality and two voters.  We omit
the proofs of the following results (they are available in
Appendix~\ref{app:set}).

\newcommand{\lempluralitycombdcdcnphthree}{Plurality-\textsc{Comb}-\textsc{DCDC} is $\np$-hard, even for elections with only three voters.}

\begin{theorem}\label{lem:plurality-comb-dcdc-np-h-3}
  \lempluralitycombdcdcnphthree
\end{theorem}

\newcommand{\lemapprovalcombdcdcnphtwo}{For each fixed
  integer~$\constantk \ge 2$, $\constantk$-Approval-\textsc{Comb}-\textsc{DCDC}
  is $\np$-hard, even for elections with only two voters.}

\begin{theorem}\label{lem:approval-comb-dcdc-np-h-2}
  \lemapprovalcombdcdcnphtwo
\end{theorem}

\newcommand{\vetocombdcdcnphone}{For each fixed integer~$\constantk
  \geq 1$, $\constantk$-Veto-\textsc{Comb}-\textsc{DCDC} is $\np$-hard, even for
  elections with only a single voter.}

\begin{theorem}\label{veto-comb-dcdc-np-h-1}
  For each fixed integer~$\constantk \geq 1$,
  $\constantk$-Veto-\textsc{Comb}-\textsc{DCDC}
  is $\np$-hard, even for elections with only a single voter.
\end{theorem}

\newcommand{\lembordacombdcdcnphtwo}{Borda-\textsc{Comb}-\textsc{DCDC} is $\np$-hard, even for elections with only two voters.}

\begin{theorem}\label{lem:borda-comb-dcdc-np-h-2}
  \lembordacombdcdcnphtwo
\end{theorem}

\newcommand{\lemcopelandcombdcdcnphthree}{Copeland$^\alpha$-\textsc{Comb}-\textsc{DCDC} is $\np$-hard, even for elections with only three voters.}

\begin{theorem}\label{lem:copeland-comb-dcdc-np-h-3}
  \lemcopelandcombdcdcnphthree
\end{theorem}

\subsection{Constructive  and Destructive Control by Adding Candidates}

For the case of combinatorial control by adding candidates, we give
sample proofs for the cases of Borda. %

\begin{theorem}\label{lem:borda-comb-dcac-np-h-2}
  Borda-\textsc{Comb}-\textsc{CCAC} and Borda-\textsc{Comb}-\textsc{DCAC} are both
  $\np$-hard, even for elections with only two voters.
\end{theorem}

\begin{proof}
  We first show the $\np$-hardness result for Borda-\textsc{Comb}-\textsc{CCAC} and then show how to modify the proof for Borda-\textsc{Comb}-DCAC.

  Let the notation be as in the introduction to this
  section.  Given an instance $I \coloneqq \probSetCoverInstance$ for
  \probSetCover with $n'\coloneqq|\ElementCandidateSet|$,
  we create an
  instance $I'$ of Borda-\textsc{Comb}-\textsc{CCAC} as follows.  We construct
  the registered candidate set $C = \{d, p\}\cup D$, where $D =
  \{d_1, \ldots, d_{n'}\}$.  We construct the unregistered
  candidate set~$A = \ElementCandidateSet\cup \SetCandidateSet$.  We
  construct two voters with the following preference orders:
  \begin{align*}
    & d \pref D \pref p \pref \SetCandidateSet \pref \ElementCandidateSet \pref \cdots  \text{\, and }  \\
    & p \pref \overleftarrow{\ElementCandidateSet} \pref d \pref \SetCandidateSet \pref \overleftarrow{D} \pref \cdots .
  \end{align*}
  We use the set-embedding bundling function.  We claim that $I$ is a
  ``yes''-instance of \probSetCover if and only if it is possible to
  ensure $p$'s victory by adding at most $h$ (bundles of) candidates.
  Note that $d$ gets $n'$ points more than $p$ from the first voter.
  Given a set cover of size $h$, we add the corresponding $s_j$'s to
  the election. Simple calculation shows that in this case $p$ and $d$
  tie as winners.

  For the reverse direction, note that the relative scores of $p$ and
  $d$ in the first vote do not change irrespective which candidates we
  add.  The relative scores of $p$ and $d$, however, do
  change in the second vote in the following way: For each
  unregistered candidate~$x_i$ added to the election, $p$'s score
  increases by one but $d$'s score remains unchanged.  Thus, the only
  way to ensure that $p$ is a winner is by bringing all the candidates
  from $\ElementCandidateSet$ to the election.  Doing so by adding at
  most $h$ candidates is possible only if there is a size-$h$ cover
  for $I$.

  The construction for Borda-\textsc{Comb}-DCAC is the same, except
  that, first, we do not want $p$ to win but $d$ to lose (that is, we
  define $d$ to be the despised candidate, and second, we define $D$ to
  have only $n-1$ dummy candidates.
\end{proof}

The proofs for the remaining results are available in
Appendix~\ref{app:set}.  Note that technically similar
$\paranp$-hardness results already follow from our discussion of the
non-combinatorial variants; using the set-embedding technique we can
give proofs that use fewer voters.

\newcommand{\lemcopelandcombdcacccacnphthree}{For each rational number~$\alpha$, $0 \le \alpha \le 1$, Copeland$^\alpha$-\textsc{Comb}-\textsc{DCAC}
  and Copeland$^\alpha$-\textsc{Comb}-\textsc{CCAC} are $\np$-hard, even for
  elections with only three voters.}
 
\begin{theorem}\label{lem:copeland-comb-dcac-ccac-np-h-3}
  \lemcopelandcombdcacccacnphthree
\end{theorem}

\newcommand{\lemmaximincombccacnphsix}{Maximin-\textsc{Comb}-\textsc{CCAC} is
  $\np$-hard, even for elections with only six voters.}

\begin{theorem}\label{lem:maximin-comb-ccac-np-h-6}
  \lemmaximincombccacnphsix
\end{theorem}

\section{\signaturetech for Destructive~Control}\label{section_formal_proofs_signatures}
\label{sec:sig}

\newcommand{\signature}{\ensuremath{\vec{\gamma}}}
\newcommand{\dptable}{\ensuremath{\vec{\mathcal{Z}}}}
\newcommand{\sig}[1]{\ensuremath{{\gamma}_{#1}}}
\newcommand{\arr}[1]{\ensuremath{{\mathcal{Z}}_{#1}}}
\newcommand{\scorevec}{\ensuremath{\vec{\delta}}}
\newcommand{\svec}[1]{\ensuremath{{\delta}_{#1}}}
\newcommand{\relevantregisteredset}{\ensuremath{C_{\mathrm{R}}}}
\newcommand{\relevantunregisteredset}{\ensuremath{A_{\mathrm{R}}}}
\newcommand{\typevec}{\ensuremath{\vec{\tau}}}
\newcommand{\type}[1]{\ensuremath{\tau_{#1}}}

We now move on to the discussion of the signatures technique of
obtaining $\fpt$ algorithms.  Specifically, in this section we show
the following results:
\begin{enumerate}
\item For each fixed integer~$\constantk \geq 1$ and for each voting
  rule {$\calR \in \{$}$\constantk$-Approval, $\constantk$-Veto{$\}$},
  $\calR$-DCAC is in $\fpt$.
\item Plurality-\textsc{Comb}-DCAC and Veto-\textsc{Comb}-DCAC are in $\fpt$.
\item For each fixed integer~$\constantk \geq 1$ and for each voting
  rule {$\calR \in \{$}$\constantk$-Approval, $\constantk$-Veto{$\}$},
  $\calR$-DCDC is in $\fpt$.
\end{enumerate}
That is, we apply the technique to the case of destructive candidate
control, under $\constantk$-Approval and $\constantk$-Veto elections.
The main idea of the signature technique is to identify certain
properties of the candidates to be added/deleted that allow us to
treat them as equivalent. We refer to these properties as signatures
and we build $\fpt$ algorithms based on the observations that the
number of different signatures (in a given context) is a function of
the number of voters only.

The results from this section apply, in particular, to the cases of
Plurality-DCDC and Veto-DCDC. However, these problems are simple
enough that direct algorithms for them that are easier and faster; we
provide such algorithms in \autoref{sec:other-fpt}.

\subsection{Destructive Control by Adding Candidates}

Let us consider an instance of the Destructive Control by Adding
Candidates problem for the case of $t$-Approval/$t$-Veto (for now we
focus on the non-combinatorial variant).  The instance consists of the
set $C$ of registered candidates, the set $A$ of unregistered
candidates, the collection $V$ of $n$ voters, the despised candidate
$d \in C$, and an integer~$k$ bounding the number of candidates that
we can add. We assume that $d$ is a winner in election $(C,V)$
(otherwise we can trivially accept the input).
The general scheme for our $\fpt$ algorithm (parameterized by the
number $n$ of the voters) is as follows:

\begin{enumerate}
\item We guess a candidate~$p \in C \cup A$ ($p \neq d$). The role
  of~$p$ is to defeat~$d$, that is, to obtain more points than $d$.
  Altogether there are $m\coloneqq|C|+|A|$ candidates and we repeat our
  algorithm for each possible choice of $p$.
\item For each choice of $p$, we ``kernelize'' the input instance,
  that is, we bound the number of ``relevant''
  candidates %
  by a function of the parameter $n$, and search for an optimal
  solution in a brute-force manner over this ``kernel'' (indeed, the
  signature technique regards computing this kernel).\footnote{We
    mention that this kind of kernelization is called \emph{Turing
      kernelization}.  See 
    \citet{Bin-Rai-Fer-Fom-Fok-Sau-Vil-2012} and of
    \citet{Sch-Kom-Mos-Nie-2012} for examples of this concept in the
    context of graph problems.}
\item We accept if the best solution found adds at most $k$ candidates
  and we reject otherwise.
\end{enumerate}

We now describe how to perform the kernelization step. Let us consider
the registered candidates first. It turns out that it suffices to
focus on a few relevant ones only.

\begin{definition}[Relevant registered candidates]
  Fix an integer $\constantk$, $\constantk \geq 1$, and consider an
  instance of $\constantk$-Approval-DCAC. We call a registered
  candidate \emph{relevant} if this candidate receives \emph{at least}
  one point. For the case of $\constantk$-Veto-DCAC, we call a
  registered candidate \emph{relevant} if this candidate receives
  \emph{at least} one veto. We refer to those candidates that are not
  relevant as \emph{irrelevant}.
\end{definition}

For the case of $t$-Approval-DCAC, we can safely remove all the
irrelevant registered candidates. This is so for two reasons: First,
removing an irrelevant candidate does not change the score of any other
registered (or later-added) candidate. Second, an irrelevant
candidate can never obtain score higher than the despised one because,
on the one hand, initially this candidate has score zero, and, on the
other hand, under $t$-Approval adding candidates never increases the
scores of those already registered.

For the case of $t$-Veto-DCAC, if there is an irrelevant candidate
then there are two possibilities. If $d$ receives at least one veto,
then $d$ already is not a winner (because the irrelevant candidate
receives no vetoes and, thus, defeats $d$). If $d$ does not receive any
vetoes, then this stays so, irrespective what candidates we add, and $d$
remains a winner. In either case, we can immediately output the
correct answer.

Thus, from now on we assume that all the registered candidates are
relevant.  For each $\constantk$-Approval-DCAC instance
($\constantk$-Veto-DCAC instance) with $n$ voters, at most $t \cdot n$
candidates are relevant.

To deal with the unregistered candidates, we introduce the notion of a
$\{d,p\}$-signature (recall that $d$ is the despised candidate and $p$
is a candidate whose goal is to defeat $d$).  Let $c$ be some
unregistered candidate. Each voter can rank $c$ in three different
ways relative to $p$ and $d$: This voter can rank $c$ ahead of both
$p$ and $d$, below both $p$ and $d$, or between $p$ and~$d$ (a finer
distinction is not necessary).  A $\{d,p\}$-signature for~$c$ is a
vector in $[3]^n$ that for each voter indicates which of these cases
holds. Formally, we use the following definition.

\begin{definition}[$\{d,p\}$-Signature]\label{def:signature}
  Consider an election~$(C \cup A,V)$, a candidate $d \in C$, and a
  candidate $p \in C \cup A$. Let $n$ be the number of voters in $V$.
  A $\{d,p\}$-signature of candidate $c \in (C \cup A)\setminus
  \{d,p\}$ is a size-$n$~vector~$\signature =
  (\sig{1},\sig{2},\ldots,\sig{n}) \in [3]^n$ such that for each
  voter~$v_i \in V$, it holds that:
   \begin{align*}
    \sig{i} = 
    \begin{cases}
      3, & \text{ if } v_i \text{ prefers $c$ to both $p$ and $d$},\\
      1, & \text{ if } v_i \text{ prefers both $p$ and $d$ to $c$},\\
      2, & \text{ otherwise}.
    \end{cases}
  \end{align*}
\end{definition}

Now, observe that for a given choice of $p$, adding exactly~$\constantk$ 
candidates with the same $\{d,p\}$-signature has the same effect
on the relative scores of $p$ and $d$ as adding more than $\constantk$
such 
candidates.

\begin{lemma}\label{lem:approval-dcac-bound-unregistered-candidates}
  Consider an instance $I\coloneqq((\electionC, \electionV), A, d \in \electionC, \solk)$ of $\constantk$-Approval-DCAC (of
  $\constantk$-Veto-DCAC), with the despised candidate $d$, and with
  some arbitrarily selected candidate $p\in C \cup A$. Let $\signature$ be some
  $\{d,p\}$-signature for this election.  Adding $t$ unregistered
  candidates with signature $\signature$ has the same effect on the
  relative scores of $p$ and $d$ as adding more than $t$ candidates
  with this signature.
\end{lemma}
\begin{proof}
  Let us focus on the case of $t$-Approval-DCAC.  Let $n$ be the
  number of voters in instance $I$. We have $\signature = (\sig{1},
  \ldots, \sig{n})$.  Consider the $i$th voter.
  \begin{enumerate}
  \item If $\sig{i} = 3$, then after adding $t$ candidates with
    signature $\signature$, the $i$th voter will give $0$~points to
    both $p$ and $d$.

  \item If $\sig{i} = 1$, then the $i$th voter will give the same
    number of points to $p$ (respectively, to~$d$) as prior to adding candidates, irrespective how many
    candidates with signature $\signature$ we add.

  \item If $\sig{i} = 2$, then for each candidate~$c$ with signature~$\signature$, 
  the $i$th voter ranks $c$ between candidates~$p$ and $d$.
  Since this voter may have preferences either $p \pref d$ or $d \pref p$, 
  this implies that either for each candidate $c$ with
  signature $\signature$, the $i$th voter has preference order~$p \pref c \pref d$, or for each candidate $c$ with signature
  $\signature$, the $i$th voter has preference order $d \pref c
  \pref p$. In the first case, adding $t$ (or more) candidates with
  signature~$\signature$ will ensure that the $i$th voter gives
  zero points to $d$ and gives the same number of points to~$p$ as
  prior to adding the candidates.  In the second case, the situation is
  the same, but with the roles of $p$ and $d$ swapped.
  \end{enumerate}
  Summing over the points provided by all voters, this proves that
  adding $t$ candidates with a given signature $\signature$ has the
  same effect on the relative scores of $p$ and $d$ as adding any more
  such candidates.
  
  The argument for the case of $t$-Veto-DCAC is
  analogous.~\end{proof}

In effect, it suffices to keep at most $\constantk$
candidates with each signature.  This results in having at most
$\constantk \cdot 3^n$ 
unregistered candidates. We can now formally describe our
kernelization process.

\begin{theorem}\label{lem:approval-dcac-turing-kernel}
  For each fixed integer~$\constantk\ge 1$, $\constantk$-Approval-\textsc{DCAC}
  and $\constantk$-Veto-\textsc{DCAC} admit Turing kernels of size
  $O(\constantk \cdot 3^n)$, where $n$ is the number of voters.
\end{theorem}

\begin{proof}
  Consider an instance $I$ of $\constantk$-Approval-DCAC (of
  $\constantk$-Veto-DCAC). Let $d$ be the despised candidate and let
  $n$ be the number of voters in the instance. As per our discussion,
  we can assume that the instances are non-trivial and that all 
  registered candidates are relevant. Thus, there are at most $t\cdot
  n$ registered candidates.  By
  Lemma~\ref{lem:approval-dcac-bound-unregistered-candidates}, for
  each choice of $p$ it suffices to consider $3^n$
  $\{d,p\}$-signatures, and for each signature at most $t$
  candidates. Altogether, for each choice of candidate $p$ among the
  registered and unregistered candidates, we produce an instance of
  $\constantk$-Approval-DCAC (of $\constantk$-Veto-DCAC), with at most
  $t\cdot n$ registered candidates and at most $t\cdot 3^n$
  unregistered ones (for each possible signature we keep up to $t$
  arbitrarily chosen unregistered candidates); in each instance we can
  add either the same number of candidates as in $I$, or one less, if
  $p$ is an ``added'' candidate already.  It is possible to preclude
  $d$ from winning in the original instance if and only if it is
  possible to do so in one of the produced instances.
\end{proof}

Using a brute-force approach on top of the kernelization given by
Theorem~\ref{lem:approval-dcac-turing-kernel}, it is possible to solve
both $\constantk$-Approval-DCAC and $\constantk$-Veto-DCAC in $\fpt$
time: straightforward application of a brute-force search to each
instance produced by Theorem~\ref{lem:approval-dcac-turing-kernel} gives
running time $O^*({t \cdot 3^{n} \choose \solk})$,
where the $O^*$ notation suppresses polynomial terms. However, it never makes sense to add more than $t\cdot n$ candidates
(intuitively, if we added more than $t\cdot n$ candidates, then at least one
would be irrelevant and we could as well not add him).  Thus we
can assume that $\solk \le t\cdot n$.  In effect, the straightforward
brute-force algorithm running on top of
Theorem~\ref{lem:approval-dcac-turing-kernel} has running time
$O^*((t \cdot 3^n)^{t\cdot n})$. However, if we are willing to spend more space,
then we can obtain significantly better running times.

\begin{theorem}\label{lem:plurality-dcac-fpt}
  Plurality-\textsc{DCAC} can be solved in time $O(m\cdot n \cdot 2^{n})$,
  using $O^*(2^n)$ space, where $m$ is the total number of candidates
  and $n$ is the number of voters.
\end{theorem}

\begin{proof}
  Our algorithm uses a similar general structure as before.  We assume
  that we are given a non-trivial instance, where all the registered
  candidates are relevant.  First, we guess a candidate~$p$ whose goal
  is to defeat $d$ and from now on we focus on a situation where we
  have both $p$ and $d$, and the goal is to ensure that $p$ gets more
  points than $d$. (If $p$ is an unregistered candidate, then we add
  $p$ to the election, decrease the number of candidates that we can
  add by one, and proceed as if $p$ were a registered candidate to
  begin with.)
  
  We define a simplified notion of a candidate's signature.  A
  \emph{signature} for an unregistered candidate~$c$ is a size-$n$
  binary vector~$\typevec=(\type1, \ldots, \type{n}) \in \{0,1\}^{n}$ such that the following hold:
  \begin{enumerate}
  \item We have $\type{i}=1$ if the $i$th voter ranks candidate $c$  ahead of all 
  the registered candidates.
\item We have $\type{i}=0$ if the $i$th voter ranks candidate $c$
  below some registered candidate.
  \end{enumerate}
  We define the \emph{signature}~$\typevec$ of a set $A'$ of unregistered candidates
  analogously: Value one at a given position means that some candidate
  from $A'$ is ranked ahead of all the registered candidates and value
  zero means that some registered candidate is ranked ahead of all the
  members of $A'$.

  Let $k$ be the number of candidates that we are allowed to add.
  Using the new notion of signatures, we maintain a size-$2^{n}$ table
  $\dptable\coloneqq(\arr{\typevec}) \in [\solk+1]^{2^{n}}$, which for each
  signature~$\typevec \in \{0,1\}^{n}$ stores the minimum
  number~$\arr{\typevec}$ of unregistered candidates such that there
  is a size-$\arr{\typevec}$ ($\arr{\typevec} \leq k$)
  candidate subset~$A^{(\typevec)}\subseteq A\setminus \{p\}$ with
  signature~$\typevec$ (value $\solk+1$ indicates that there is no
  such set). 

  To describe our algorithm for computing table $\dptable$, we need
  one more piece of notation. For each pair of signatures $\typevec$ and
  $\typevec'$, we define a ``merged'' signature $\typevec
  \oplus\typevec' = (\max\{\type{i},\type{i}'\})_{i\in [n]}$.  In
  other words, we apply the coordinate-wise $\max$ operator.
  We compute table $\dptable$ as follows (our algorithm is slightly
  more complicated than necessary for the case of Plurality rule, but
  we will also use it as a base for more involved settings):
  \begin{enumerate}
  \item We initiate the table by setting $\arr{\typevec}\coloneqq1$ if there
    is at least one unregistered candidate with signature~$\typevec$,
    and we set $\arr{\typevec}\coloneqq\solk+1$ otherwise.
  \item For each unregistered candidate $a$ we perform the following
    operations:
    \begin{enumerate}
    \item We compute $a$'s signature $\typevec_a$.
    \item We compute a new table $\arr{}'$, by setting, for each signature $\typevec$:
    \[
    \arr{\typevec}' = \min( \{ \arr{\typevec} \} \cup \{ \arr{\typevec'}+1 \mid \typevec = \typevec' \oplus\typevec_a\} \cup \{k+1\} ).
    \]  
  \item We copy the contents of $\arr{}'$ to $\arr{}$.  (At this
    point, for each signature $\typevec$, $\arr{\typevec}$ is the
    number of candidates in the smallest set (of size up to $k$)
    composed of the so-far processed candidates that jointly have this
    signature, or is $k+1$ if no such set exists.)
    \end{enumerate}
  \item We pick a signature $\typevec$ such that $\arr{\typevec}$ has
    a minimum value and adding the candidate set~$A_{\typevec}$ that
    implements this signature ensures that $p$ has more points than
    $d$ (note that this last condition is easy to check: Given a
    signature $\typevec$, if the $i$th component $\type{i}$ is zero,
    then the $i$th voter gives one point to whomever this voter ranks
    first among the registered candidates; if $\type{i}$ is one, then
    the point goes to a candidate from $A_{\typevec}$, that is,
    neither to $p$ nor $d$). If $\arr{\typevec}$ is smaller or equal
    to $k$, the number of candidates that we can add, then we
    accept. Otherwise we reject (for this choice of $p$).
  \end{enumerate}

  Let us first consider the algorithm's running time. The most
  time-consuming part of the algorithm is the loop in the second step
  of the procedure computing the table $\arr{}$.  For each out of at
  most $m$ candidates, computing $\arr{}'$ requires filling in $O(2^n)$ entries of
  the table. If we first copy the then-current contents of $\arr{}$ to $\arr{}'$,
  and then perform the remaining updates, this can be done in time
  $O(m\cdot n \cdot 2^{n})$. This dominates the running time of the remaining
  parts of the algorithm.  

  Now let us consider the correctness of the algorithm. Assume that we
  have guessed the correct candidate $p$ and that there is a subset of
  unregistered candidates $A' = \{a_1, \ldots, $ $ a_\ell\}$ such that
  $p$ has more points than $d$ after we add the candidates from~$A'$
  and $\ell \leq k$. If $\typevec$ is the signature of the set $A'$,
  then notice that the algorithm indeed computes value
  $\arr{\typevec} \leq \ell$. Further, if the algorithm accepts, then
  it must have found a solution. Thus the algorithm is correct.
\end{proof}

We can apply the above ideas to the case of $\constantk$-Approval and
$\constantk$-Veto as well.
The proofs are in Appendix~\ref{app:sig}.

\newcommand{\lemapprovaldcacfpt}{For each fixed integer~$\constantk
  \ge 2$, $\constantk$-Approval-\textsc{DCAC} can be solved in time
  $\min\{O(m \cdot (\constantk \cdot 3^n)^{t\cdot n}), O(m\cdot n \cdot t \cdot(t+1)^{t\cdot n})\}$,
  where $m$ is the total number of
  candidates and $n$ is the number of voters.}

\begin{theorem}\label{lem:approval-dcac-fpt}
  \lemapprovaldcacfpt
\end{theorem}

Adapting the algorithms in a straightforward way (basically by
inverting, or reversing, the signatures) used for
Theorem~\ref{lem:plurality-dcac-fpt} and
Theorem~\ref{lem:approval-dcac-fpt}, we can show analogous results for
the case of $t$-Veto-DCAC.

\begin{corollary}
  \label{cor:veto-based-dcac-fpt}
  For each fixed integer~$\constantk \ge 1$, $\constantk$-Veto-DCAC
  can be solved in time \linebreak
  $\min\{O(m \cdot (\constantk \cdot 3^n)^{t\cdot n}), O(m\cdot n \cdot t \cdot(t+1)^{t\cdot n})\}$, where $m$ is the total number of
  candidates and $n$ is the number of voters.
\end{corollary}

To conclude the discussion of the signature technique for the case of
destructive control by adding candidates, we consider the
combinatorial variant of the problem.  The situation is more
complicated because now we are adding bundles of candidates instead of
individual candidates. In effect, we cannot upper-bound the number of
bundles to add in $\constantk$-Approval-\textsc{Comb}-DCAC (or
$\constantk$-Veto-\textsc{Comb}-DCAC).  This is so because bundles with the
same signature but with different sizes may have different effects on
the score difference between the despised candidate~$d$ and a specific
guessed candidate~$p$ (indeed, $\constantk$-Approval-\textsc{Comb}-DCAC is $\wone$-hard for $\constantk \geq 2$,
as shown in~\autoref{section_formal_proofs_mcc}).  Yet, for Plurality
and for Veto only the first (or the last) position gets a point (a
veto).  This structural observation allows us to use our
non-combinatorial algorithms.

\begin{corollary}\label{cor:plurality-veto-comb-dcac-fpt}
  Plurality-\textsc{Comb}-DCAC and Veto-\textsc{Comb}-DCAC,
  parameterized by the number of voters,
  are fixed-parameter trac\-ta\-ble.
\end{corollary}

\begin{proof}
  For the case of Plurality, it suffices to use, for example, 
  the same algorithm as in Theorem~\ref{lem:plurality-dcac-fpt},
  but with the following changes: 
  \begin{enumerate}
  \item For each choice of candidate~$p$, we also consider each way of
    adding $p$ to the election if $p$ were unregistered ($p$ might
    belong to several different bundles and we try each possibility).
  \item Each unregistered candidate's signature is replaced by the
    signature of the set of candidates in its bundle.
  \end{enumerate}
  Since under Plurality each voter gives a point only to whomever this
  voter ranks first, this strategy suffices. The case of the Veto rule is
  handled analogously.
\end{proof}

\subsection{Destructive Control by Deleting Candidates}

Let us now move on to the case of destructive control by deleting
candidates.  The (Turing) kernelization approach from the previous
section cannot be easily transferred to the case of deleting
candidates.  This is because we cannot upper-bound the number of
candidates that have to be deleted in terms of the number $n$ of the
voters in the election.  However, applying our signature technique
followed by casting the remaining task as an integer linear
program~(ILP), we can show fixed-parameter tractability (for our
parameterization by the number of voters). We present the proof of the
following result in Appendix~\ref{app:sig} (while the proof is quite
interesting technically and we encourage the reader to read it, we
also believe that the previous proofs have presented the signatures
technique sufficiently well).

\newcommand{\lemapprovaldcdcfpt}{For each fixed
  integer~$\constantk\ge 1$, both $\constantk$-Approval-\textsc{DCDC} and $\constantk$-Veto-\textsc{DCDC} 
  can be solved in time $O^*(m \cdot 4^n \cdot (3^n)^{3^n})$, where $m$ is the total number of
  candidates and $n$ is the number of voters.}

\begin{theorem}\label{lem:approval-dcdc-fpt}
  \lemapprovaldcdcfpt
\end{theorem}

We leave it as an open question whether there is a direct
(combinatorial) $\fpt$-algorithm that avoids ILPs as used in
Theorem~\ref{lem:approval-dcdc-fpt}; notably, for many voting problems which can be shown to be $\fpt$ via ILPs,
it is the case that no other, direct algorithm is known~\shortcite{bredereck2014parameterizedgua}.

\section{Remaining Results: Membership in  $\xp$, $\fpt$, and $\p$}
\label{sec:other-proofs}
\label{sec:oth}

\newcommand{\bruteforcevec}{\ensuremath{\vec{b}}}
\newcommand{\bfvec}[1]{\ensuremath{b_{#1}}}
\newcommand{\CollectCands}[1]{\texttt{CollectCands(#1)}}
\newcommand{\Search}[1]{\texttt{BruteForceSearch(#1)}}

In this section we present our remaining algorithms that show
membership of our problems in $\xp$, $\fpt$, and $\p$. In the cases of
$\xp$ and $\fpt$ membership, our algorithms use a simple brute-force
approach.

\subsection{Fixed-Parameter Tractability Results}\label{sec:other-fpt}

We now show simple, fast $\fpt$ algorithms for Plurality-CCDC,
Plurality-DCDC, and Veto-DCDC.  The main idea for the algorithms in
this section is to guess a subset of voters that will give a specific
candidate one point under either Plurality or Veto.  The key
observation is that in the case of deleting candidates, after guessing
this subset of voters, it is trivial to find the set of candidates to
delete to ``implement'' this guess.

\begin{theorem}\label{lem:plurality-ccdc-fpt}
  Plurality-\textsc{CCDC} can be solved in $O(m \cdot n \cdot 2^n)$ time,
  where $n$ is the number of voters and $m$ is the number of candidates in the input election.
\end{theorem}

\begin{proof}
  Let $I\coloneqq((C,V), p, \solk)$ be a Plurality-CCDC instance.  If $I$ is
  a yes-instance, then after deleting at most $\solk$~candidates,
  there must be a subset of voters who each give candidate~$p$ one point, and no
  other candidate has more points than $p$.  Observe that in order to
  let $p$ gain one point from a voter, one has to delete all the
  candidates this voter prefers to $p$. Our algorithm, based on these
  observations, proceeds as follows.

  We consider all $2^n$~subsets of $n$ voters. For each considered 
  set~$V'$ of voters we do the following: For each voter $v' \in V'$, we
  delete all candidates that $v'$ prefers to~$p$. In effect, all
  members of $V'$ rank $p$ first.  Then, we keep deleting all
  candidates that have more than $|V'|$ points (note that deleting
  some candidate that has more than $|V'|$~points may result in some other
  candidate exceeding this bound). If in the end no candidate has more
  than $|V'|$ points and we deleted at most $\solk$ candidates, then we
  accept. Otherwise, we proceed to the next subset of voters. If we
  did not accept after going over all subsets of voters, then we reject.
  
  To see why the algorithm is correct, note that whenever it accepts,
  it has constructed a correct solution. If, however,
  there is a correct solution in which, after deleting the candidates,
  $p$ gets points exactly from the voters in some subset $V'$, then the algorithm will accept when considering this
  subset. Establishing the running time is straightforward.
\end{proof}

It is straightforward to see how to adapt the algorithm from
Theorem~\ref{lem:plurality-ccdc-fpt} to the destructive case. In
essence, it suffices to try all choices of a candidate $p$ whose goal
is to defeat the despised candidate $d$ and for each such choice guess
a subset of voters that are to give points to $p$. If after deleting
the candidates that these voters prefer to $p$ (assuming that none
of them prefers $d$ to $p$) the despised candidate $d$ has fewer
points than $p$, then we accept. In the destructive case there is no
need to have the final loop of deleting candidates scoring higher than
$p$.

\begin{corollary}\label{cor:plurality-dcdc-fpt}
  Plurality-DCDC can be solved in $O(m^2 \cdot n \cdot 2^n)$ time,
  where $n$ is the number of voters and $m$ is the number of candidates in the input election.
\end{corollary}

We provide an analogous result for the case of Veto.

\begin{theorem}\label{lem:veto-dcdc-fpt}
   Veto-\textsc{DCDC} can be solved in $O(m\cdot n\cdot 2^n)$ time,
   where $n$ is the number of voters and $m$ is the number of candidates in the input election.
\end{theorem}

\begin{proof}
  We use almost the same approach as for
  Theorem~\ref{lem:plurality-ccdc-fpt}. First, we guess candidate $p$
  whose goal is to have fewer vetoes than $d$. Deleting candidates can
  only increase the number of vetoes that a remaining candidate
  has. Thus, our algorithm proceeds as follows.

  We consider every subset $V'$ of voters that prefer $p$ to $d$ in
  the election. For each voter $v'$ in the guessed subset, we delete
  all candidates that this voter ranks below $d$ (by choice of~$V'$, $p$ is never deleted).
  If in effect $d$ has more vetoes than
  $p$, we accept. Otherwise we try the next subset of voters.  If we
  do not accept after processing all subsets of voters, then we reject.
  
  Establishing the correctness and the running time of the algorithm
  is straightforward.
\end{proof}

\subsection{XP Results}

In this section, we establish $\xp$ results for all our $\wone$-hard
problems.  This implies that if the number of voters is a constant, then
the problems are polynomial-time solvable.
 
\begin{theorem}\label{lem:approval-CCAC-CCDC-xp}
  For each fixed integer~$\constantk$, $\constantk\ge 1$, and for each
  control type $\mathcal{K} \in \{\text{CCAC},\allowbreak \text{CCDC}\}$,
  $\constantk$-Approval-$\mathcal{K}$ and
  $\constantk$-Veto-$\mathcal{K}$ can be solved in time
  $O(m^{\constantk n}\cdot m \cdot n)$, where $m$ is the number of candidates and
  $n$ is the number of voters.
\end{theorem}

\begin{proof}
  We consider the \textsc{CCAC} and the CCDC cases jointly, in parallel for
  both $\constantk$-Approval and $\constantk$-Veto.  Our algorithm
  first guesses for each voter the set of $t$ candidates that this
  voter will rank first (for the case of $\constantk$-Approval) or
  last (for the case of $\constantk$-Veto). There are $O(m^{\constantk n})$ possible different guesses.
    
  For each guess, for each voter, we
  verify which candidates have to be added (for the case of \textsc{CCAC}) or
  deleted (for the case of DCAC) to ensure that the voter ranks the
  guessed $t$ candidates on top. If it suffices to add/delete $k$
  candidates to implement the guess, and in effect of implementing the
  guess our preferred candidate is a winner, then we accept. Otherwise
  we proceed to the next guess. If no guess leads to acceptance, then
  we reject.  Establishing the correctness and the running time of the
  algorithm is immediate.
\end{proof}

\begin{theorem}\label{lem:approval-Comb-CCAC-CCDC-xp}
  For each fixed integer~$\constantk$, $\constantk\ge 1$, and each
  control type $\mathcal{K} \in \{\text{CCAC},\text{DCAC}\}$,
  $\constantk$-Approval-\textsc{Comb}-$\mathcal{K}$ and
  $\constantk$-Veto-\textsc{Comb}-$\mathcal{K}$ can be solved in time
  $O(m^{2\constantk n}\cdot m \cdot n)$, where $m$ is the total number of candidates and
  $n$ is the number of voters.
\end{theorem}

\begin{proof}
  We use the same approach as described in the proof of
  Theorem~\ref{lem:approval-CCAC-CCDC-xp}, but in addition to guessing
  the first $\constantk$ candidates for each vote, we also guess for
  each added candidate $c$ the candidate to whose bundle $c$ belongs.
\end{proof}

\subsection{Polynomial-time Solvable Case}

There is one case that is still missing,
which we show next.

\begin{theorem}\label{thm:maximincombdcac}
  Maximin-\textsc{Comb}-\textsc{DCAC} can be solved in $O(m^3 \cdot n)$ time, where $m$ is the number of
  candidates and $n$ is the number of voters.
\end{theorem}

\begin{proof}
  It was shown by~\citet{fal-hem-hem:j:multimode} that Maximin-DCAC is
  polynomial-time solvable.  The same strategy can be applied for the
  combinatorial case as well. To this end, let $\probCDCACInstance$ be an Instance for Maximin-\textsc{Comb}-DCAC.
  
  The algorithm is simple and can be described as follows: We
  guess up to two candidates, add their bundles to the election, and
  check whether the despised candidate~$d$ is no longer a winner; if so, we
  accept and otherwise we reject.

  To see why this simple algorithm is correct, consider a solution, that is, a set~$A'$ of at most $k$ unregistered candidates whose bundles are to be added.
  If $A'$ consists of at most two candidates, then we are done.
  Otherwise, let us take a closer look at the set~$A'$. 
  It is clear that in the resulting election~$E' \coloneqq (C\cup \combRule(A'), V)$,
  $d$ is not a winner.
  Therefore, there must be
  at least one other candidate $p$ that has a higher score than~$d$.
  Consider the bundle~$b_p$ of some candidate in $\combRule(A')$ such that $b_p$ includes $p$
  (indeed, there might be several such candidates, and we can choose any
  one of them arbitrarily; it is also possible that $p$ is present in
  the original election, in which case we take $b_p$ to be an
  ``empty'' bundle).  Further, consider some candidate $z$ in $\combRule(A')$ such that
  the Maximin score of candidate $d$ in the election $E'$ is exactly $N_{E'}(d,z)$. There may be
  several such candidates and we choose one arbitrarily. Finally, we
  choose an arbitrary candidate from $A'$ whose bundle $b_z$ includes~$z$
  (in fact, it is possible that $z$ is present in the original
  election, in which case we take $b_z$ to be an ``empty'' bundle).

  It is clear that $p$ defeats $d$ in the election~$(C\cup \combRule(x,y), V)$ where
  $x$ and $y$ are the leaders of the bundles $b_p$ and $b_z$, respectively (if either of these bundles is ``empty'', then we simply disregard it). Thus, each ``yes''-instance of Maximin-DCAC has
  a solution that consists of at most two candidates and, consequently, it is
  enough to guess and test at most two unregistered candidates.
\end{proof}

\section{Outlook}\label{section_outlook}

Our work motivates several possible research directions,
some of which are listed below.

\begin{enumerate}
 \item We still do not know the exact complexity of $2$-Veto-\textsc{Comb}-DCAC.
 This open question is marked by a question mark ($?$) in~\autoref{tab:summary}.

\item It is natural to consider an even more diverse set of voting
  rules.  This might allow for understanding our general techniques
  better, and might help in devising new techniques as well.  For
  example, it would be interesting to consider the Bucklin rule. On
  the one hand, candidate control problems for Bucklin are
  $\np$-complete~\shortcite{erd-fel-rot-sch:j:bucklin-control-theory} and,
  on the other hand, the rule can be seen as an adaptive variant of
  $t$-Approval. Thus it would be interesting to see if our techniques
  can be applied to the case of Bucklin.
 
\item One can experiment with real-world elections to understand the
  practical relevance of our theoretical findings or heuristically
  solve our proven worst-case intractable cases.  
  One possible starting point for
  such an analysis would be the experimental paper of
  \citet{erd-fel-rot-sch:j:bucklin-control-experiments}.
  
\item It is interesting to consider some game-theoretic aspects of
  candidate control, where several agents perform the control
  actions. So far, doing this even for the simplest rules such as
  Plurality was hampered by the fact that these control problems are
  $\np$-hard.  Our (partial) tractability results might help in
  overcoming this obstacle. %
  
\item It might be worthwhile to study the multimode control framework
  of \citeA{fal-hem-hem:j:multimode} for the case of
  few voters. In multimode control one can perform control actions of
  several types at the same time (for example, one can add candidates and
  delete voters). Faliszewski et al.\ expected that combining two types
  of easy control actions would lead to a possibly computationally
  hard multimode control problem, but they did not observe such
  effects among natural voting rules. One possible explanation for
  this fact is that they did not have enough easy control problems
  available to combine. We have shown that many candidate control
  problems become easy (in $\fpt$) when they are parameterized by the
  number of voters and, thus, there are more opportunities for
  studying multimode control problems.
\end{enumerate}

Generally, we believe that the case of few voters did not receive
sufficient attention in the computational social choice literature and
many other problems can (and should) be studied with respect to this
parameter.  The main two reasons for the study of this parameter are
as follows.  First, it is very well motivated, as discussed
in~\autoref{motivation}.  Second, in our control problems we observe a
rich (parameterized) complexity landscape.  We hope that such rich
landscapes exist for other voting problems, when parameterized by the
number of voters.
Indeed, this turned out to be the case in several recent studies
conducted by some of the authors of this
paper~\cite{fal-sko-sli-tal:c:top-k-counting,bre-fal-nie-tal:c:multiwinner-sb}
as well as by other researchers~\cite{mis-nab-sin:c:minimax-av}.

\subsection*{Acknowledgments}

Piotr Faliszewski was supported by DFG project PAWS (NI 369/10) and by AGH University grant 11.11.230.124 (statutory research).
Nimrod Talmon was supported by DFG, Research Training Group ``Methods for Discrete Structures''~(GRK~1408),
while the author was affiliated with TU Berlin, where most of the work was done.
This work was also partly supported by COST Action IC1205 on Computational Social Choice.

A preliminary short version of this work appears in the Proceedings of the
Twenty-Ninth Conference on Artificial Intelligence (AAAI'15),
pages 2045-2051, AAAI Press.

\bibliographystyle{abbrvnat}
\bibliography{bibliography}

\appendix

\section{Deferred Proofs for the \mcctech}
\label{app:mcc}

\newtheorem*{thmcortapprovalccac*}{Theorem~\ref{cor:t-approval-ccac}}

\begin{thmcortapprovalccac*}%
  \cortapprovalccac
\end{thmcortapprovalccac*}

\begin{proof}
  One can use the same proof as in the case of
  Theorem~\ref{lem:plurality-ccac}, but for each voter
  we introduce $t-1$ additional registered dummy candidates which this
  voter ranks first (each voter ranks all the remaining dummy
  candidates last).
  In this way, each dummy candidate has exactly one point. 
  The reasoning for the proof of correctness works in the same way.
\end{proof}

\newtheorem*{thmcortvetoccac*}{Theorem~\ref{cor:t-veto-ccac}}

\begin{thmcortvetoccac*}%
  \cortvetoccac
\end{thmcortvetoccac*}

\begin{proof}  
  One can use the same proof as in Theorem~\ref{lem:veto-ccac},
  but we introduce $t-1$~additional registered dummy candidates whom every voter ranks last.
  In this way,
  each dummy candidate receives exactly one veto from each voter,
  while $p$ and $d$ receive the same number of vetoes as in
  the election constructed in the proof for Theorem~\ref{lem:veto-ccac}.
  
  One can verify that the arguments from that proof apply here as well. 
\end{proof}

\newtheorem*{thmlemkvetoccdc}{Theorem~\ref{lem:k-veto-ccdc}}

\begin{thmlemkvetoccdc}%
  \lemkvetoccdc
\end{thmlemkvetoccdc}

\begin{proof}
  We use almost the same proof as in Theorem~\ref{lem:veto-ccdc}, but
  we add sufficiently many dummy (padding) candidates to ensure that
  we can only delete vertex and edge candidates.  Let $I=(G, h)$ be an
  input instance of $\probColorClique$. Let $E' = (C',V')$ be the
  election created by the reduction from the proof of
  Theorem~\ref{lem:veto-ccdc} on input $I$ and set $k\coloneqq\hforvetoone$.

  We modify this election by extending $C'$ to contain a set $D$ of
  $t$ dummy candidates, $D = \{d_1, \ldots, d_{t}\}$,
  and modify the voter collection $V'$ as follows (recall that the
  number~$|V'|$ of voters is a function polynomially bounded by $h$;
  set $n'\coloneqq |V'|$):
  \begin{enumerate}
  \item For each voter~$v$ in $V'$ except the last group of voters, 
    we modify $v$'s preference order to
    rank the dummies $d_1, \ldots, d_{t-1}$ last and $d_t$ first.
  \item For each voter $v$ in the last group of $V'$,
    we rank all candidates from $D$ such that $v$ will have a preference order of the form
    \[d_t \pref \cdots \pref (D\setminus \{d_t\}) \pref p.\]
  \item We add $n'$ voters, all with preference order of the form
    \[ \cdots \pref p \pref D. \]
  \end{enumerate}
  One can verify that each newly added candidate~$d_i$, $1\le i \le t-1$,
  has $2n'$ vetoes and $d_t$ has $n'$ vetoes.
  Since we assume the input graph to be connected and to have at least two vertices,
  at least one candidate from the edge and vertex candidates receives fewer vetoes than~$p$.
  Thus, $p$ is not a winner initially.

  We claim that $p$ (the preferred candidate from the proof of
  Theorem~\ref{lem:veto-ccdc}) can become a winner by deleting at most
  $k$ candidates if and only if $I$ is a ``yes''-instance.

  First, we
  note that if we delete any of the new dummy candidates from $D\setminus \{d_t\}$, 
  then $p$
  certainly does not become a winner since $p$ will have at least
  $n'+2H$ vetoes and $d_t$ will have exactly $n'$ vetoes. 
  If we delete dummy candidate~$d_t$, 
  then $p$ will receive $2n'$ vetoes,
  but there is at least one remaining vertex or edge candidate which is not vetoed by the last group of voters and has hence less than $2n'$ vetoes.
  In consequence,
  no dummy candidate can be deleted.
  Thus, 
  none of them
  will have fewer vetoes than $p$ and (ignoring the dummy candidates) the election will behave as if
  it was held according to the Veto rule. The argument from the proof of
  correctness in Theorem~\ref{lem:veto-ccdc} holds.
\end{proof}

\newtheorem*{thmlemtwoapprovalccdc}{Theorem~\ref{lem:2-approval-ccdc}}

\begin{thmlemtwoapprovalccdc}%
  \lemtwoapprovalccdc
\end{thmlemtwoapprovalccdc}

\begin{proof}
  The proof is quite similar to that for the case of Veto-CCDC, but
  now the construction is a bit more involved. Again, we give a
  parameterized reduction from the $\probColorClique$ problem.  Let
  $I=\probMCCInstance$ be our input instance with graph~$G$ and
  non-negative integer $h$, and let the notation be as described in
  the introduction to Section~\ref{sec:mcc}.  We form an instance $I'$ of
  $2$-Approval-CCDC based on $I$.  We build our candidate set $C$ as
  follows,
  where we set $T = |V(G)| + |E(G)|$ with the intended meaning
  that $T$ is an integer larger than the number of candidates that we
  can delete; we set $H\coloneqq 2 {h \choose 2} = (h-1)\cdot h$:
  \begin{enumerate}
  \item Introduce the preferred candidate $p$.
  \item Introduce $T$ candidates $B = b_1, \ldots, b_T$ (these are the
    \emph{blocker} candidates whose role, on the one hand, is to ensure
    that $p$ has to obtain a given number of points and, on the other
    hand, is to prevent deleting too many candidates of other types).
  \item For each vertex $v \in V(G)$, introduce candidate~$v$.
  \item For each edge $\{u,v\} \in E(G)$, introduce two candidates,
    $(u,v)$ and $(v,u)$.
  \item Introduce two sets $D = \{d_1, \ldots d_h\}$ and $F=\{f_{(i,j)} \mid 1\le i,j \le h, i \neq j\}$ of dummy candidates.
  \end{enumerate}
  Before we describe the set of voters, we need some additional notation.
  By writing $B$ in a preference order, we mean 
  \[
    b_1 \pref b_2 \pref \cdots \pref b_T.
  \]
  For each two colors $i,j$ ($1 \leq i,j \leq h$, $i \neq j$), by
  $F(i,j)$ we mean an arbitrary (but fixed) ranking of all the
  candidates of the form $(u,v)$, where $u \in V_i(G)$ and
  $v \in V_j(G)$.  The set of voters consists of the following groups:
  \begin{enumerate}
  \item
    We have $h+3H$ voters,
    each with preference
    order of the form 
    \[B  \pref \cdots \pref p.\]
  \item
    For each color $i$,
    $1 \leq i \leq h$,
    there are $3H + 1$ voters, 
    where the first of them has preference order of the form
    \[
      V_i(G) \pref p \pref B \pref \cdots,
    \]
    and the remaining ones have preference order of the form
    \[
      V_i(G) \pref d_i \pref B \pref \cdots.
    \]

  \item
    For each pair $i,j$ of distinct colors ($1 \leq i, j \leq h$, $i \neq j$),  
    there are $3H + h - 1$ voters, 
    where the first of them has preference order of the form
    \[
      F(i,j) \pref p \pref B \pref \cdots,
    \]
    the next two have preference orders of the form
    \[
      f_{(i,j)} \pref E(i,j) \pref B \pref \cdots, 
    \]
    and the remaining ones have preference orders of the form
    \[
    F(i,j) \pref f_{(i,j)} \pref B \pref \cdots. 
    \]
  \item
    For each pair $i,j$ of distinct colors
    ($1 \leq i,j \leq h$, $i \neq j$),
    we introduce two voters with the following preference orders of the forms
    \begin{align*}
      p \pref R(i,j) \pref B \pref \cdots, \\
      p \pref R'(i,j) \pref B \pref \cdots.
    \end{align*}
  \end{enumerate}
  Note that the total number of constructed voters is polynomially bounded by $h$:  
  \begin{align*}
     \quad h+3H+(3H+1)\cdot h + (3H+h-1)\cdot H + 2H 
  = & \quad 2h+ 4H + 4H\cdot h + 3H^2. 
  \end{align*}

  We set the number of candidates that can be deleted to $k \coloneqq \hforvetoone$, 
  with the intention that
  $p$ can become a winner if and only if it is possible to delete all of
  the vertex candidates and all of the edge candidates
  except for the ones corresponding to a multi-colored clique of order~$h$.
  We note that if $G$ indeed contains
  an order-$h$ multi-colored clique~$Q$, 
  then deleting all candidates in $V(G) \setminus
  Q$ and all edge candidates of the form $(u,v)$ where either $u
  \notin Q$ or $v \notin Q$ indeed ensures that $p$ is a winner
  (in this case
  $p$, and  all of the vertex and edge candidates have $h+3H$ points each,
  and all of the blocker candidates have at most $h+3H$ points each).

  Now we come to the reverse direction.  Assume that it is possible to
  ensure $p$'s victory by deleting at most $k$ candidates and let
  $C' \subseteq C$ be a set of at most $k$ candidates such that $p$ is
  a winner of $E' = (C \setminus C',V)$. Note that $k < T-1$ and so
  there are at least two blocker candidates that receive $h + 3H$
  points each from the first group of voters. The only voters from
  whom $p$ can obtain additional points after deleting at most
  $k$~candidates are the ones in the second and third group and there
  are exactly $h + H$ of them ($h$ in the second group and $H$ in the
  third group).  However, $p$ can obtain the points from the second
  and the third groups of voters without, at the same time, increasing
  the score of the highest-scoring blocker candidate if and only if:
  (a) we delete all-but-one vertex candidates of each color, and (b)
  for each pair $i,j$ of distinct colors ($1 \leq i,j \leq h$,
  $i \neq j$) all-but-one edge candidates of the form $(u,v)$, where
  $u \in V_i(G)$ and $v \in V_j(G)$.  This means deleting exactly $k$
  candidates.

  Let us now argue that for each pair of colors $i$ and $j$
  ($1 \leq i,j \leq h$, $i \neq j$), the two remaining edge candidates
  $(u,v)$ and $(u',v')$ such that $u \in V_i(G), v \in V_j(G)$ and
  $u' \in V_j(G), v' \in V_i(G)$ regard the same edge, that is,
  $\{u,v\} = \{u',v'\}$. Indeed, if it were not the case, then one of
  the edge candidates $(u,v)$ and $(u',v')$ would receive at least
  $3H+h+1$ points (from the third group of the voters, for color
  choices $(i,j)$ and $(j,i)$) and $p$ would not be a winner.  Thus
  our claim holds.

  Further, we claim that if $p$ is a winner of $E'$, then for each
  pair of not-deleted vertex candidates $u$ and $v$, it must be the
  case that both edge candidates $(u,v)$ and $(v,u)$ are still in the
  election, meaning that there is an edge between $u$ and $v$ in the
  original graph.  It suffices to consider the case of $(u,v)$ (the
  case of $(v,u)$ is symmetric).  If instead of $(u,v)$ the only
  not-deleted edge candidate for the pair of colors of $u$ and $v$ is
  some edge candidate $(u',v')$ (where $(u',v') \neq (u,v)$), then one
  of the two following cases must happen: either $u$ and $v$ would
  receive more than $h - 1$ points from the fourth group and,
  therefore, would have more than $h + 3H$ points altogether, causing
  $p$ not to be a winner, or $(u',v')$ would receive more than one
  point from the fourth group, again causing $p$ to not be a winner
  (this latter case holds because from the previous paragraph we know
  that $(v',u')$ must be included in the election).  Thus $p$ can
  become a winner by deleting at most $H$ candidates if and only if
  $G$ contains a multi-colored clique of order~$h$.

  It is clear that the given reduction can be computed in polynomial
  time and that it is a parameterized reduction,
  therefore the proof is complete.
\end{proof}

\newtheorem*{thmlemkapprovalccdc}{Theorem~\ref{lem:k-approval-ccdc}}

\begin{thmlemkapprovalccdc}%
  \lemkapprovalccdc
\end{thmlemkapprovalccdc}

\begin{proof}
  Let $E'=(C', V')$ be the election constructed in the proof for
  Theorem~\ref{lem:2-approval-ccdc}.  One can use the same proof as
  for Theorem~\ref{lem:2-approval-ccdc} except that now for each
  voter~$v_i\in V'$ we introduce a group of $t-2$ new dummy
  candidates, $d^i_1, d^i_2, \ldots, d^i_{t-2}$, that are ranked
  first, and for each such introduced group, we introduce two new
  dummies, $c^i_1$ and $c^i_2$, and one voter %
  with preference
  order of the form (we write $D_i$ to refer to the preference
  order~$d^i_1 \pref d^i_2 \pref \ldots \pref d^i_{t-2}$):
  \[ D_i \pref c^i_1 \pref c^i_2 \pref B \pref \cdots.\]
  These voters ensure that none of the new dummy
  candidates can be deleted without increasing the score of the
  highest-scoring blocker candidate. If a score of a highest-scoring
  blocker candidate increases, 
  then the preferred candidate no longer has any chance of winning.
  If none of the new dummy candidates can be deleted,
then the correctness proof works the same as the one given for Theorem~\ref{lem:2-approval-ccdc}.
  
  The number of voters is still polynomially bounded by the clique order~$h$.
\end{proof}

\section{Deferred Proofs for the \cvctech}
\label{app:cvc}

\newtheorem*{thmlembordaccac}{Theorem~\ref{lem:borda-ccac}}

\begin{thmlembordaccac}%
  \lembordaccac
\end{thmlembordaccac}

\begin{proof}
  We give a reduction from \probCVC (we use the notation as
  provided at the beginning of Section~\ref{sec:cvc}).  Given an instance~$I = (G,h)$
  for \probCVC, we construct an instance for Borda-\textsc{CCAC}. We let the
  registered candidate set~$C$ be $\{p,d\} \cup E(G)$, and we let $V(G)$ be
  the set of unregistered candidates.  We construct the following
  voters:
  \begin{enumerate}
  \item For each $\ell$, $1 \leq \ell \leq 3$, we have the following
    two voters:
    \begin{align*}
      \mu(\ell) \colon & B(\ell) \pref E^{(-\ell)} \pref V^{(-\ell)} \pref d \pref p, \\
      \mu'(\ell) \colon & p \pref d \pref \revnot{V^{(-\ell)}} \pref
      \revnot{E^{(-\ell)}} \pref B'(\ell).
    \end{align*}
  \item For $\ell = 4$, we have the following two voters:
    \begin{align*}
      \mu(\ell) \colon & B(\ell) \pref E^{(-\ell)} \pref V^{(-\ell)} \pref d \pref p, \\
      \mu'(\ell) \colon & d \pref p \pref \revnot{V^{(-\ell)}} \pref \revnot{E^{(-\ell)}} \pref B'(\ell). 
    \end{align*}
  \item We have two voters with preference orders 
    \begin{align*}
      E(G)  \pref    p \pref d \pref V(G), \\
      p \pref \revnot{E(G)} \pref d \pref \revnot{V(G)}.
    \end{align*}
  \end{enumerate}
  We claim that it is possible to ensure $p$'s victory by adding $k \coloneqq h$
  candidates if and only if there is a vertex cover of size~$h$ 
  for~$G$.
  
  Let $m' \coloneqq |E(G)|$ be the number of edges in $E(G)$.
  Note that at the beginning,
  $p$ has $5\edgesetsize+5$ points, $d$ has $4\edgesetsize+5$
  points, and each edge candidate has $5m'+6$ points. Thus $p$ is not a
  winner.  Adding each unregistered vertex candidate~$v$ causes the
  scores of all candidates to increase: For the edge candidates
  that include~$v$ as an endpoint this increase is by five points,
  whereas for all other candidates this increase is by six
  points. 
  Note that the last two voters always prefer the registered candidates to any vertex candidate.
  Thus, by simple counting, 
  each of these $h$ vertex candidates may obtain at most $4\edgesetsize+5h+7$~points and
  will never obtain more points than $p$ as long as $\edgesetsize+h\ge 2$.
  
  Thus, if we have a vertex cover of size $h$, then it is possible to
  ensure $p$'s victory by adding all vertex candidates that
  correspond to this vertex cover.  For the other direction, assume
  that it is possible to ensure $p$'s victory by adding at most $h$
  candidates and let $S$ be such a set of candidates.  For the sake of
  contradiction, assume that there is an edge candidate $e$ which is
  not covered by some vertex candidate in~$S$. It follows that the
  score of $e$ is greater than the score of $p$, which is a
  contradiction.  Thus $S$ must correspond to a vertex cover in $G$.
\end{proof}

\newtheorem*{thmlemcopelandccac}{Theorem~\ref{lem:copeland-ccac}}

\begin{thmlemcopelandccac}%
  \lemcopelandccac
\end{thmlemcopelandccac}

\begin{proof}
  We give a reduction from \probCVC (we use the notation as provided
  at the beginning of Section~\ref{sec:cvc}).  Given an
  instance~$(G,h)$ for \probCVC, we construct an instance for
  Copeland$^\alpha$-\textsc{CCAC}.  Let the registered candidate set~$C$ be
  $\{p, d\} \cup E(G)$, and let $V(G)$ be the set of unregistered
  candidates.  We introduce the following voters:
  \begin{enumerate}
  \item For each $\ell$, $1 \leq \ell \leq 4$, we construct four voters,
    two voters with the following preference order:
    \[
      B(\ell) \pref E^{(-\ell)} \pref V^{(-\ell)} \pref d \pref p,
    \]
    and two voters with the following preference order:
    \[
      p \pref d \pref \revnot{V^{(-\ell)}} \pref \revnot{E^{(-\ell)}} \pref B'(\ell).
    \]
  \item One voter with the preference order 
    $E \pref V \pref d \pref p$,
    and one voter with the preference order
    $d \pref p \pref \revnot{E} \pref \revnot{V}$.
  \item One voter with the preference order
    $p \pref V \pref E \pref d$,
    and one voter with the preference order
    $\revnot{E} \pref d \pref p \pref \revnot{V}$.
  \end{enumerate}
  We illustrate the results of head-to-head contests between the
  candidates in Figure~\ref{fig:copeland1reduction}. We claim that
  there is a vertex cover of size at most $h$ for $G$ if and only if
  $p$ can become a winner of the election by adding at most $k\coloneqq h$
  candidates.

  Sometimes, when we say that a vertex candidate and an edge candidate are \emph{adjacent} to each other, we mean that the corresponding vertex and edge are adjacent to each other.
  Consider a situation where we have added some subset~$A'$ of $k$
  candidates ($k \leq h$; take $k = 0$ to see the situation prior to
  adding any of the unregistered candidates). The candidates have
  the following scores:
  \begin{enumerate}
  \item $p$ has score $\alpha \edgesetsize + k$ ($p$ ties head-to-head contests
    with all edge candidates and wins all head-to-head
    contests with the vertex candidates).
  \item $d$ has score $1 + \alpha k$ ($d$ wins the head-to-head
    contest with $p$ and ties all head-to-head contests with the
    vertex candidates).
  \item Each added vertex candidate $v$ has score $3 + \alpha k$ ($v$
    ties the head-to-head contests with $d$ and the remaining $k-1$
    vertex candidates and wins the head-to-head contests with the
    three edge candidates that are adjacent to~$v$).
  \item Each edge candidate~$e$ has score $\alpha \edgesetsize + k + 1 - c(e)$,
    where $c(e)$ is the number of vertex candidates from $A'$ that are adjacent
    to $e$ ($e$ ties head-to-head contests with $p$ and the remaining
    edge candidates and wins head-to-head contests with $d$ and all
    added vertex candidates except those that are adjacent to
    $e$).    
  \end{enumerate}
  In effect, it holds that $p$ is a winner of the election if
  and only if $A'$ corresponds to a vertex cover of $G$.
\end{proof}
\begin{figure}
  \center
  \begin{tikzpicture}[draw=black!75, scale=0.75,->]
    \tikzstyle{vertex}=[circle,draw=black!80,minimum size=12pt,inner sep=0pt]
    \tikzstyle{special}=[draw,thick,->, >=stealth']

    \foreach [count=\i] \pos / \text in {
      {(0,-2)}/$E$,
      {(1,0)}/$d$,
      {(4,0)}/$p$,
      {(5,-2)}/$V$}
    {
      \node[vertex] (V\i) at \pos {\text};
    }

    \foreach \i / \j in {1/2,2/3,3/4} {
      \path[,>=stealth'] (V\i) edge (V\j);
    }

    \path[special,bend angle=15] (V1) to[bend right] node[midway,fill=white] {if $v \notin e$} (V4);
    \path[special,bend angle=15] (V4) to[bend right] node[midway,fill=white] {if $v \in e$} (V1);
  \end{tikzpicture}
  \caption{Illustration for the reduction used in the proof of Theorem~\ref{lem:copeland-ccac}.
  Each vertex in the graph corresponds to a candidate or a set of candidates,
  and there is an arc going from a vertex $u_1$ to a vertex $u_2$ if $u_1$ beats $u_2$ in a head-to-head contest.
  Edges indicating ties are omitted.
  The main idea is that an edge candidate beats a vertex candidate if and only if the vertex candidate
  is one of the endpoints of the edge candidate.}
  \label{fig:copeland1reduction}
\end{figure}
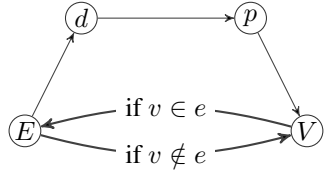

\newtheorem*{thmlemcopelandccdc}{Theorem~\ref{lem:copeland-ccdc}}

\begin{thmlemcopelandccdc}%
  \lemcopelandccdc
\end{thmlemcopelandccdc}

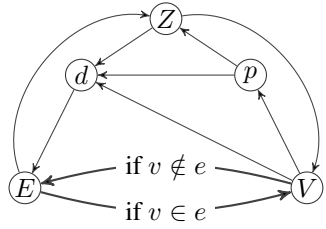
\begin{figure}
  \center
  \begin{tikzpicture}[draw=black!75, scale=0.75,->]
    \tikzstyle{vertex}=[circle,draw=black!80,minimum size=12pt,inner sep=0pt]
    \tikzstyle{special}=[draw,thick,->,>=stealth']

    \foreach [count=\i] \pos / \text in {
      {(0,-2)}/$E$,
      {(1,0)}/$d$,
      {(2.5,1)}/$Z$,
      {(4,0)}/$p$,
      {(5,-2)}/$V$}
    {
      \node[vertex] (V\i) at \pos {\text};
    }

    \foreach \i / \j in {2/1,3/2,4/3,4/2,5/4,5/2} {
      \path[>=stealth'] (V\i) edge (V\j);
    }

    \path[>=stealth'] (V1) edge[bend angle = 60, bend left] (V3);
    \path[>=stealth'] (V3) edge[bend angle = 60, bend left] (V5);
    
    \path[special,bend angle=15] (V1) to[bend right] node[midway,fill=white] {if $v \in e$} (V5);
    \path[special,bend angle=15] (V5) to[bend right] node[midway,fill=white] {if $v \notin e$} (V1);
  \end{tikzpicture}
  \caption{Illustration for the reduction used in the proof of Theorem~\ref{lem:copeland-ccdc}.
  Each vertex in the graph corresponds to a candidate or a set of candidates,
  and there is an arc going from a vertex $u_1$ to a vertex $u_2$ if $u_1$ beats $u_2$ in a head-to-head contest.
  Edges indicating ties are omitted.
  The main idea is that an edge candidate beats a vertex candidate if and only if the vertex candidate
  is one of the endpoints of the edge candidate.}
  \label{fig:copeland2reduction}
\end{figure}

\begin{proof}
  We give a reduction from \probCVC (we use the notation as provided
  at the beginning of Section~\ref{sec:cvc}).  Given an instance~$I = (G,h)$ for \probCVC,
  we construct an instance for Copeland$^\alpha$-CCDC.  The candidate
  set contains the edge candidates, the vertex candidates, the
  preferred candidate $p$, a dummy candidate~$d$, and a set of
  additional dummy candidates $Z = \{z_1, \ldots, z_{\edgesetsize + \vertexsetsize}\}$ (recall that $\edgesetsize\coloneqq |E(G)|$ denotes the number of edges in $E(G)$ and that  $\vertexsetsize\coloneqq |V(G)|$ denotes the number of vertices in $V(G)$).  We
  construct the following voters:
  \begin{enumerate}
  \item For each $\ell$, $1 \leq \ell \leq 4$, we construct two voters
    with  preference order:
    \[
      A(\ell) \pref E^{(-\ell)} \pref V^{(-\ell)} \pref Z \pref d \pref p,
    \]
    and two voters with preference order:
    \[
      p \pref d \pref \revnot{Z} \pref \revnot{V^{(-\ell)}} \pref \revnot{E^{(-\ell)}} \pref A'(\ell).
    \]
  \item We also construct the following ten voters:
    \begin{align*}
      v_1 \colon & V \pref E \pref Z \pref d \pref p, \\
      v'_1 \colon & p \pref d \pref \revnot{Z} \pref \revnot{V} \pref \revnot{E}, \\[6pt]
      v_2 \colon & V \pref p \pref d \pref E \pref Z, \\
      v'_2 \colon & \revnot{E} \pref \revnot{Z} \pref \revnot{V} \pref p \pref d, \\[6pt]
      v_3 \colon & p \pref Z \pref d \pref V \pref E, \\
      v'_3 \colon & \revnot{E} \pref \revnot{V} \pref p \pref \revnot{Z} \pref d, \\[6pt]
      v_4 \colon & d \pref E \pref Z \pref V \pref p, \\
      v'_4 \colon & p \pref \revnot{V} \pref \revnot{Z} \pref d \pref \revnot{E}, \\[6pt]
      v_5 \colon & Z \pref V \pref E \pref d \pref p, \\
      v'_5 \colon & p \pref d \pref \revnot{E} \pref \revnot{Z} \pref \revnot{V}.
    \end{align*}

  \end{enumerate}

  Figure~\ref{fig:copeland2reduction} illustrates the results of the
  head-to-head contests among the candidates.  Prior to deleting any
  of the candidates, we have the following scores:
  \begin{enumerate}
  \item each edge candidate~$e$ has $\edgesetsize+\vertexsetsize + \alpha \edgesetsize + 2$ points
  ($e$ wins head-to-head contests against all candidates in $Z$ due to voters~$v_2$ and $v'_2$,
  wins head-to-head contests against its ``incident'' vertex candidates due to the first group of voters,
  and ties with $p$ and the remaining edge candidates), 
  \item each vertex candidate~$u$ has $\alpha(\vertexsetsize - 1) + \edgesetsize - 1$ points
  ($u$ wins head-to-head contests against all edge candidates that are \emph{not} ``incident'' to $u$
  due to voters from the first group,
  and ties with the remaining vertex candidates),
  \item each candidate~$z$ from $Z$ has $\vertexsetsize + 1 + \alpha(\edgesetsize+\vertexsetsize- 1)$ points
  ($z$ wins head-to-head contests against all vertex candidates and $d$ 
  due to voters~$v_3, v'_3, v_5, v'_5$, 
  and ties with the remaining candidates from $Z$),
  \item $d$ has $\edgesetsize$ points 
    ($d$ wins head-to-head contests against all edge candidates due to voters $v_4$ and $v'_4$), and 
  \item $p$ has $\edgesetsize+\vertexsetsize + \alpha \edgesetsize + 1$ points 
    ($p$ wins head-to-head contests against all candidates from $Z$ 
    due to voters~$v_3$ and $v'_3$,
    wins head-to-head contests against $d$ due to voters~$v_2,v'_2,v_3,v'_3$,
    and ties with all edge candidates). 
  \end{enumerate}
  Thus, all edge candidates are co-winners,
  and $p$ is not a winner because each edge candidate has one point more than~$p$.
  However, $p$ has more points than any other non-edge candidate.
  Note that in the input graph it holds that $\edgesetsize = 3\vertexsetsize/2$.

  We claim that it is possible to ensure that $p$ is a winner by
  deleting at most $k\coloneqq h$ candidates if and only if there is a vertex
  cover of size $h$ for $G$. 

  If there is a vertex cover for $G$ of size $h$, then deleting the
  corresponding $h$ vertices ensures that $p$ is a winner.  To see why
  this is the case, note that after deleting vertices corresponding to
  a vertex cover the score of $p$ does not change, but the score of
  each edge candidate decreases by at least one. The scores of other
  candidates do not increase, so $p$ is a winner.
  
  For the reverse direction, assume that it is possible to ensure that $p$ is
  a winner by deleting at most $h$ candidates. Deleting candidates
  cannot increase $p$'s score, so it must be the case that each edge
  candidate loses at least one point.

  Observe that deleting candidates other than the vertex candidates will not make 
  the edge candidates lose more than one point than $p$. 
  The only possibility of deleting a candidate such that an edge
  candidate~$e$ loses a point but $p$ does not is by deleting one of
  the vertex candidates $v'(e)$ or $v''(e)$. 
  Thus, if it is possible
  to ensure that $p$ is a winner, then we must delete vertices that
  correspond to a vertex cover.
\end{proof}

\section{Deferred Proofs for the \setcovertech}
\label{app:set}

\newtheorem*{thmlemapprovalcombccdcnphone}{Theorem~\ref{lem:approval-comb-ccdc-np-h-1}}

\begin{thmlemapprovalcombccdcnphone}%
  \lemapprovalcombccdcnphone
\end{thmlemapprovalcombccdcnphone}

\begin{proof}
  We build upon the proof of Theorem~\ref{lem:r-cccdc}, but add
  $\constantk-1$ dummy candidates.  Specifically, given an instance
  $I\coloneqq\probSetCoverInstance$ for \probSetCover, we create an instance
  $I'$ of $\constantk$-Approval-\textsc{Comb}-CCDC as follows.  We
  construct an election $E = (C,V)$ where $C = \{p\} \cup
  \ElementCandidateSet \cup \SetCandidateSet \cup D$, where $D =
  \{d_1, \ldots, d_{\constantk-1}\}$, and where $V$ contains a single
  voter with the following preference order:
  \[
    D \pref \ElementCandidateSet \pref p \pref \SetCandidateSet.
  \]
  We use the bundling function as described in the introduction to the
  set-embedding section.  We claim that $I$ is a ``yes''-instance of
  \probSetCover if and only if it is possible to ensure $p$'s victory
  by deleting at most $h$ (bundles of) candidates.

  To see the correctness of the argument,
  note that if there is a solution that
  ensures $p$'s victory by deleting a specific number of candidates,
  then there is also a solution that achieves the same and
  does not delete any of the dummy candidates
  (it is always at least as useful to delete one of the set candidates instead of a dummy one).
\end{proof}

\newtheorem*{thmlemvetocombccdcnphone}{Theorem~\ref{lem:veto-comb-ccdc-np-h-1}}

\begin{thmlemvetocombccdcnphone}%
  \lemvetocombccdcnphone
\end{thmlemvetocombccdcnphone}

\begin{proof}
  Let the notation be as in the introduction to Section~\ref{sec:set}.
  Given an instance $I\coloneqq\probSetCoverInstance$ for \probSetCover, we
  create an instance $I'$ of $\constantk$-Veto-\textsc{Comb}-CCDC as
  follows.  We construct an election $E = (C,V)$ with candidate set:
  \[
     C = \{p, z\} \cup \ElementCandidateSet \cup \SetCandidateSet \cup D,
  \]
  where $D = \{d_1, \ldots, d_{\constantk - 1}\}$ is a set of dummy
  candidates (indeed, for $\constantk = 1$, that is, for Veto, $D =
  \emptyset$), and with the voter collection $V$ containing a single
  voter with the following preference order:
  \[
    z \pref \ElementCandidateSet \pref \SetCandidateSet \pref D \pref p.
  \]
  We use the set-embedding bundling function, with the added feature
  that $\combRule(z) = \SetCandidateSet$. We claim that $I$ is a
  ``yes''-instance of \probSetCover if and only if it is possible to
  ensure $p$'s victory by deleting at most $h+1$ bundles.
  
  Using similar reasoning as used in Theorem~\ref{lem:approval-comb-ccdc-np-h-1},
  it follows that the only way of ensuring that $p$ is a winner
  is to let all the remaining candidates receive no points at all.
  The only way to achieve this is to first delete up to $h$ candidates from
  $\{s_1, \ldots, s_m\}$ that correspond to a cover of the ground set and then to delete $z$.
\end{proof}

\newtheorem*{thmlempluralitycombdcdcnphthree}{Theorem~\ref{lem:plurality-comb-dcdc-np-h-3}}

\begin{thmlempluralitycombdcdcnphthree}%
  \lempluralitycombdcdcnphthree
\end{thmlempluralitycombdcdcnphthree}

\begin{proof}
  Let the notation be as in the introduction to Section~\ref{sec:set}.  Given an instance $I\coloneqq\probSetCoverInstance$ for
  \probSetCover, we create an instance $I'$ of Plurality-\textsc{Comb}-DCDC as
  follows.  We construct an election $E = (C,V)$ where $C =
  \{p,d\} \cup \ElementCandidateSet \cup \SetCandidateSet$, and where
  $V$ contains three voters with the following preference orders:
  \begin{align*}
    \ElementCandidateSet \pref p \pref \SetCandidateSet \pref d, &\\
    d \pref \ElementCandidateSet \pref p \pref \SetCandidateSet, &\text{\, and }\\
    p \pref d \pref \ElementCandidateSet  \pref \SetCandidateSet. &
  \end{align*}
  We use the set-embedding bundling function. We claim that the despised candidate $d$ can
  be precluded from winning by deleting at most $h$ (bundles of)
  candidates if and only if there is a set cover of size $h$ for~$I$.

  Prior to deleting any of the candidates, $d$, $p$, and one of the
  candidates from $X$ are tied as winners.  Since deleting candidates
  cannot make any candidate lose points and since deleting $p$ will
  make $d$ a unique winner, the only way of defeating $d$ is by
  ensuring that the first voter gives her point to $p$.  This means
  that all element candidates have to be removed from the election.
  By the same argument as in the previous proofs, doing so by deleting
  at most $h$ candidates is possible if and only if $I$ is a
  ``yes''-instance of \probSetCover.  
\end{proof}

\newtheorem*{thmlemapprovalcombdcdcnphtwo}{Theorem~\ref{lem:approval-comb-dcdc-np-h-2}}

\begin{thmlemapprovalcombdcdcnphtwo}%
  \lemapprovalcombdcdcnphtwo
\end{thmlemapprovalcombdcdcnphtwo}

\begin{proof}
  Let the notation be as in the introduction to Section~\ref{sec:set}.  Given an instance $I\coloneqq\probSetCoverInstance$ for
  \probSetCover, we create an instance $I'$ of
  $\constantk$-Approval-\textsc{Comb}-DCDC as follows. We construct an
  election $E = (C,V)$ with candidate set: 
  \[
    C = \{p, d\} \cup \ElementCandidateSet \cup \SetCandidateSet \cup D \cup F,
  \] 
  where $D = \{d_1, \ldots, d_{\constantk - 2}\}$ and $F = \{f_1,
  \ldots, f_{\constantk-1}\}$ are two sets of dummy candidates 
  (note that $D$ can be empty), 
  and with the voter collection $V$ containing two voters with the
  following preference orders:
  \begin{align*}
    &d \pref \ElementCandidateSet \pref D \pref p \pref \SetCandidateSet \pref F  \text{\, and}\\
    &p \pref F \pref d \pref \ElementCandidateSet \pref \SetCandidateSet \pref D.
  \end{align*}
  We use the set-embedding bundling function.  We claim that $I$ is a ``yes''-instance of
  \probSetCover if and only if it is possible to preclude $d$ from
  winning by deleting at most $h$ (bundles of) candidates.
  
  Initially, both $d$ and $p$ are winners (as well as some
  members of $\ElementCandidateSet \cup F$).  Deleting $p$ will make
  $d$ gain one more point (from the second voter), making it
  impossible for~$d$ to lose.  The same holds for the dummy candidates
  from set $F$.  In other words, if we change the set of candidates
  that gain a point from the second voter, then $d$ will obtain two
  points and will certainly be a winner.  This implies that the only
  way of making $d$ lose is to let either $p$ or at least one
  candidate from $F$ gain one point from the first voter.  By
  construction of the first voter's preference order, this is possible
  only for $p$ if and only if we delete all members of
  $\ElementCandidateSet$.  As in the previous proofs, deleting them
  (through deleting at most $h$ bundles of candidates) is possible if
  and only if $I$ is a ``yes''-instance of \probSetCover.
\end{proof}

\newtheorem*{thmvetocombdcdcnphone}{Theorem~\ref{veto-comb-dcdc-np-h-1}}

\begin{thmvetocombdcdcnphone}%
  \vetocombdcdcnphone
\end{thmvetocombdcdcnphone}

\begin{proof}
  We use the same construction as used in
  Theorem~\ref{lem:approval-comb-ccdc-np-h-1} for $\constantk$-Approval-\textsc{Comb}-CCDC
  but we reverse the preference order and replace $p$ with $d$,
  the despised candidate:
  \[
    \SetCandidateSet \pref d \pref \ElementCandidateSet \pref D.
  \]
  
  The crucial observation here is that with only one voter, the only
  way of preventing~$d$ from winning is to rank her within the last
  $\constantk$~positions.  This means that all element candidates have
  to ``disappear'' from the election (one could also try deleting the
  dummy candidates, but it is never a mistake to ``make disappear''
  the members of $\ElementCandidateSet$ instead, through deleting the
  appropriate candidates in $\SetCandidateSet$).  
  Thus we can conclude that the set of deleted candidates contains the
  set candidates only. If $d$ is to be precluded from winning
  by deleting at most $h$ candidates, this set must correspond to a
  set cover of size $h$.  Since we can assume that $h <
  |\SetCandidateSet|$, there is at least one set element not
  deleted, and this will be a winner.
\end{proof}

\newtheorem*{thmlembordacombdcdcnphtwo}{Theorem~\ref{lem:borda-comb-dcdc-np-h-2}}

\begin{thmlembordacombdcdcnphtwo}%
  \lembordacombdcdcnphtwo
\end{thmlembordacombdcdcnphtwo}

\begin{proof}
  Let the notation be as in the introduction to Section~\ref{sec:set}.  Given an instance $I\coloneqq\probSetCoverInstance$ for
  \probSetCover, we create an instance $I'$ of
  Borda-\textsc{Comb}-DCDC as follows.  We construct an election $E =
  (C,V)$ where $C = \{p,d,z\} \cup \ElementCandidateSet \cup
  \SetCandidateSet$ and where $V$ contains two voters with the
  following preference orders:
  \begin{align*}
    & d \pref \ElementCandidateSet \pref p \pref \SetCandidateSet \pref z  \text{\, and }\\
    & p \pref z \pref d \pref \overleftarrow{\ElementCandidateSet} \pref \overleftarrow{\SetCandidateSet}.
  \end{align*}
  We use the set-embedding bundling function.  We claim that $I$ is a
  ``yes''-instance of \probSetCover if and only if it is possible to
  preclude $d$ from winning by deleting at most $h$ (bundles of) candidates.

  For convenience, 
  we calculate the scores of all candidates:
  \begin{enumerate}
    \item $d$ has $2|\SetCandidateSet| + 2|\ElementCandidateSet| + 2$ points.
    \item $p$ has $2|\SetCandidateSet| + |\ElementCandidateSet| + 3$ points.
    \item Each element candidate~$x_i$ has $2|\SetCandidateSet| + |\ElementCandidateSet| + 1$ points.
    \item $z$ has $|\SetCandidateSet| + |\ElementCandidateSet| + 1$ points.
    \item Each set candidate~$s_j$ has $|\SetCandidateSet|$ points.
  \end{enumerate}
  It follows that $d$ has the highest number of points and, thus, is a
  winner.

  Since both voters rank $d$ ahead of the candidates in the set
  $\ElementCandidateSet \cup \SetCandidateSet$, no member of this set
  can have score higher than $d$, irrespective which other
  candidates we delete.  Similarly, irrespective which candidates we
  delete, $z$ will never have score higher than $d$. We conclude
  that candidate~$p$ is the only candidate that has a chance of defeating $d$.

  Since deleting candidates does not increase the scores of any of the
  remaining candidates, to ensure that $d$ is not a winner, we have to
  guarantee that he loses at least $|\ElementCandidateSet|$
  points (relative to $p$). This means that it is possible to ensure
  that $d$ is not a winner if and only if it is possible to remove all
  candidates from $\ElementCandidateSet$.  However, this is
  possible if and only if $I$ is a ``yes''-instance of
  $\probSetCover$.~\end{proof}

\newtheorem*{thmlemcopelandcombdcdcnphthree}{Theorem~\ref{lem:copeland-comb-dcdc-np-h-3}}

\begin{thmlemcopelandcombdcdcnphthree}%
  \lemcopelandcombdcdcnphthree
\end{thmlemcopelandcombdcdcnphthree}
  
\begin{proof}
  Let the notation be as in the introduction to Section~\ref{sec:set}.
  Given an instance~$I\coloneqq\probSetCoverInstance$ for \probSetCover,
  we construct an instance for
  Copeland$^\alpha$-\textsc{Comb}-DCDC. 
  Since our reduction will
  produce an instance with an odd number of voters, the particular
  value of $\alpha$ is immaterial.  
  We form the set of candidates:
  \[
  C = \{d, p\} \cup \ElementCandidateSet \cup \SetCandidateSet.  
  \] 

  We have three voters with the following preference orders:
  \begin{align*}
    p \pref d \pref \ElementCandidateSet \pref \SetCandidateSet, &\\
    d \pref \ElementCandidateSet \pref p \pref \SetCandidateSet, & \text{\, and }\\
    \overleftarrow{\ElementCandidateSet} \pref p \pref d \pref \overleftarrow{\SetCandidateSet}.&
  \end{align*}
  We use the set-embedding bundling function. We claim that $I$ is a
  ``yes''-instance of \probSetCover if and only if it is possible to
  preclude $d$'s victory by deleting at most $h$ (bundles of) candidates.
  
  The initial scores are:
  \begin{enumerate}
  \item $d$ receives $|\SetCandidateSet| + |\ElementCandidateSet|$
    points ($d$ wins head-to-head contests against all the other
    candidates but $p$).
  \item $p$ receives $|\SetCandidateSet|+1$ point ($p$ wins
    head-to-head contests against $d$ and all the members of
    $\SetCandidateSet$).
  \item Each member $x_i$ of $\ElementCandidateSet$ receives at most
    $|\SetCandidateSet|+|\ElementCandidateSet|$ (from head-to-head
    contests with $p$, all members of $\SetCandidateSet$, and the
    other members of $\ElementCandidateSet$).
  \item Each member $s_j$ of $\SetCandidateSet$ receives at most
    $|\SetCandidateSet|-1$ points (from head-to-head contests with
    the other members of $\SetCandidateSet$).
  \end{enumerate}

  Since deleting candidates cannot make any candidate gain more
  points, the only way of ensuring that $d$ is not a winner is to make
  sure that $d$'s score decreases relative to some other candidate.
  By the above list of scores, it follows that the only
  candidate that may end up with a score higher than $d$ is $p$. This
  happens only if we remove all the members 
  of~$\ElementCandidateSet$. As in the previous proofs using the
  set-embedding technique, doing so by deleting at most $h$ candidates
  is possible if and only if there is a set cover of size at most $h$
  for $I$.~\end{proof}

\newtheorem*{thmlemcopelandcombdcacccacnphthree}{Theorem~\ref{lem:copeland-comb-dcac-ccac-np-h-3}}

\begin{thmlemcopelandcombdcacccacnphthree}%
  \lemcopelandcombdcacccacnphthree
\end{thmlemcopelandcombdcacccacnphthree}

\begin{proof}
  Let the notation be as in the introduction to Section~\ref{sec:set}.  %
  Given an instance~$I\coloneqq\probSetCoverInstance$ for \probSetCover with
  $n'\coloneqq|\ElementCandidateSet|$, we construct an instance for
  Copeland$^\alpha$-\textsc{Comb}-DCAC.  Since our reduction will
  produce an instance with an odd number of voters, the particular
  value of $\alpha$ is immaterial.  We form the set of registered
  candidates:
  \[
    C = \{d, p\} \cup D \cup  F,
  \] 
  where $d$ is the despised candidate (and we will want to ensure that
  $p$ wins over $d$), and where $D\coloneqq\{d_1, \ldots, d_{n'-2}\}$ and
  $F\coloneqq\{f_1, \ldots, f_{n'-1}\}$ are two sets of dummy candidates.  We
  let the set of unregistered candidates be $A =
  \ElementCandidateSet\cup \SetCandidateSet$.  Finally, we construct
  three voters with the following preference orders:
  \begin{align*}
    & d\pref D \pref p \pref F \pref \ElementCandidateSet \pref \SetCandidateSet, \\
    & p\pref \overleftarrow{F} \pref \overleftarrow{\ElementCandidateSet} \pref \overleftarrow{D} \pref d \pref \overleftarrow{\SetCandidateSet}, \text{\, and }\\
    &\ElementCandidateSet \pref d \pref D \pref F \pref p \pref \SetCandidateSet.
  \end{align*}
  We use the set-embedding bundling function.  We claim that $I$ is a
  ``yes''-instance of \probSetCover if and only if it is possible to
  preclude $d$'s victory by adding at most $h$ (bundles of) candidates.

  Prior to adding any of the candidates, we have the following scores:
  \begin{enumerate}
  \item $d$ receives $2n'-2$ points ($d$ wins head-to-head contests
    with all the remaining registered candidates).
  \item $p$ receives $n'-1$ points ($p$ wins head-to-head contests with
    the members of $F$).
  \item Every dummy candidate~$d_i\in D$ receives at most $2n'-3$
    points ($d_i$ wins head-to-head contests with all the members of
    $F$, with $p$, and---at most---all remaining members of~$D$).
  \item Every dummy candidate~$f_i\in F$ receives at most $n'-2$ points
    ($f_i$ wins head-to-head contests with---at most---the remaining members of $F$).
  \end{enumerate}  
  
  Now,
  if there is a set cover for $I$ of size at most $h$, then adding the members of
  $\SetCandidateSet$ that correspond to the cover ensures that $d$ is
  not a winner (relative to $d$, $p$ gets additional $n'$ points).

  For the reverse direction, note that adding candidates to the election
  cannot decrease the score of any existing candidate.  Thus, in order
  to beat $d$, we must add candidates to increase (relative to $d$) the score of 
  some candidate. We make several observations:
  \begin{enumerate}
  \item The candidates in $\SetCandidateSet$ themselves do not
    contribute to the increase of a score of any candidate relative to
    $p$ because all the other candidates (including $d$) win
    head-to-head contests against them.
  \item The scores of the members of $D$ do not change relative to the
    score of $d$ irrespective which other candidates join the
    election.
  \item By the first observation in this enumeration, the maximum
    possible increase of a score of candidate is by $n'$ points (if
    this candidate defeats all members of $\ElementCandidateSet$ and
    members of $\ElementCandidateSet$ join the election). Since all members
    of set $F$ have score at most $n'-2$, none of them can obtain score
    higher than $d$, irrespective which candidates we add.
  \end{enumerate}
  As a final conclusion, we have that the only candidate that can possibly
  defeat $d$ is $p$, and this happens only if all members of $\ElementCandidateSet$
  join the election. It is possible to ensure that this happens by adding
  at most $h$ bundles of candidates if and only if there is a set cover for $I$
  of size at most $h$.

  We use the same construction for the case of Copeland$^\alpha$-\textsc{CCAC}, except that
  now $p$ is the preferred candidate and we increase the size of $D$ by one.
\end{proof}

\newtheorem*{thmlemmaximincombccacnphsix}{Theorem~\ref{lem:maximin-comb-ccac-np-h-6}}

\begin{thmlemmaximincombccacnphsix}%
  \lemmaximincombccacnphsix
\end{thmlemmaximincombccacnphsix}

\begin{proof}
  Let the notation be as in the introduction to Section~\ref{sec:set}.  Given an instance~$I\coloneqq\probSetCoverInstance$ for
  \probSetCover with $n' \coloneqq|\ElementCandidateSet|$, we construct an
  instance for Maximin-\textsc{Comb}-\textsc{CCAC}.  We let the set of
  registered candidates be $C \coloneqq \{p\}\cup D$, where $p$ is the
  preferred candidate and where $D\coloneqq\{d_1, \ldots, d_{n'}\}$ is a set of
  dummy candidates.  The unregistered candidate set is $A\coloneqq
  \ElementCandidateSet\cup \SetCandidateSet$.  We construct six voters
  with the following preference orders:
  \begin{align*}
    v_1\colon & p \pref x_1 \pref d_1 \pref \cdots \pref x_{n'} \pref d_{n'} \pref \SetCandidateSet,\\
    v_2\colon & p \pref x_{n'} \pref d_{n'} \pref \cdots \pref x_1 \pref d_1 \pref \SetCandidateSet,\\
    v_3\colon & x_1\pref \cdots \pref x_{n'} \pref d_1 \pref \cdots \pref d_{n'} \pref p \pref \SetCandidateSet,\\
    v_4\colon & d_{n'} \pref \cdots \pref d_1 \pref p \pref x_{n'} \pref \cdots \pref x_1 \pref \SetCandidateSet,\\
    v_5\colon & x_1\pref \cdots \pref x_{n'} \pref d_1 \pref \cdots \pref d_{n'} \pref p \pref \SetCandidateSet, \text{\ and}\\
    v_6\colon & d_{n'} \pref \cdots \pref d_1 \pref p \pref x_{n'} \pref \cdots \pref x_1 \pref \SetCandidateSet.
  \end{align*}
  (Note that $v_3$ and $v_5$ have the same preference order and
  that $v_4$ and $v_6$ have the same preference order.)  We use the
  set-embedding bundling function.  We claim that $I$ is a
  ``yes''-instance of \probSetCover if and only if it is possible to
  ensure $p$'s victory by adding at most $h$ (bundles of) candidates.

  Prior to adding any of the candidates, $p$ has two points and each
  candidate in $D$ has three points. All the voters rank the members
  of $\SetCandidateSet$ last, so the presence of these candidates in
  the election does not change the scores of $p$ and members of $D$.
  More so, members of $\SetCandidateSet$ themselves receive zero points
  each. If some candidate $x_i$ appears in the
  election, however, then we have the following effects:
  \begin{enumerate}
  \item This candidate's score is at most two
  (since only voters
    $v_3$ and $v_5$ prefer $x_i$ to~$p$).
  \item The score of $d_i$ becomes at most two
  (since only voters
    $v_4$ and $v_6$ prefer $d_i$ to~$x_i$).
  \item The score of $p$ does not change
  (since already $v_1$ and
    $v_2$ prefer $p$ to~$x_i$).
  \end{enumerate}
  This means that if there is a set cover of size at most $h$ for $I$,
  then adding the set candidates that correspond to this cover will
  bring all members of $\ElementCandidateSet$ to the election and $p$
  will be among the winners.

  First, note that it is impossible to
  increase the score of $p$ by adding candidates, and that for each
  $d_i$, the only way to decrease its score to at most two is to bring
  $x_i$ into the election.

  For the reverse direction, notice that in order to let $p$ win, we must
  add candidates to the election to decrease the score of every
  element candidate~$x_i$, and the only way to achieve this by
  adding at most $\solk$ bundles is by adding the $s_j$ corresponding
  to the set cover. This means that if it is possible to ensure $p$'s
  victory by adding at most $h$ candidates, then it must be possible to add
  all members of $\ElementCandidateSet$ into the election, and this
  means that there is a set cover of size at most~$h$.
\end{proof}

\section{Deferred Proofs for the \signaturetech}
\label{app:sig}

\newtheorem*{thmlemapprovaldcacfpt}{Theorem~\ref{lem:approval-dcac-fpt}}

\begin{thmlemapprovaldcacfpt}%
  \lemapprovaldcacfpt
\end{thmlemapprovaldcacfpt}

\begin{proof}
  There are two means of solving our problem. We can either run the
  brute-force algorithm on top of
  Theorem~\ref{lem:approval-dcac-turing-kernel}, obtaining running time
  $O(m \cdot (\constantk \cdot 3^n)^{t\cdot n}))$,
  or we can use a variant of the
  algorithm from Theorem~\ref{lem:plurality-dcac-fpt}. Below we describe
  how to adapt the algorithm from Theorem~\ref{lem:plurality-dcac-fpt},
  as it allows us to achieve a better running time.

  We use the same algorithm as in Theorem~\ref{lem:plurality-dcac-fpt},
  but with a more involved notion of a signature and with a more involved merging operator $\oplus$.
  Indeed,
  the algorithm to be described next is a generalization of the algorithm described in Theorem~\ref{lem:plurality-dcac-fpt}.

  So,
  if we have $n$ voters,
  then we define the \emph{unbounded} signature
  of a set $A'$ of unregistered candidates to be 
  an vector $\typevec$ with $n$ entries,
  such that the $i$th entry
  is a vector $\type{i}$ with $t$ values,
  defined as follows.
  The $j$th entry of $\type{i}$,
  for $1 \leq j \leq t$ and $1 \leq i \leq n$,
  contains the number of candidates in $A'$ that
  the $i$th voter prefers to all but $j-1$ registered candidates.
  Now a \emph{bounded} signature (simply, a signature) of a set $A'$ is its unbounded signature where all
  entries greater than $t$ are replaced by $t$.
  Altogether, there are
  $(t+1)^{t\cdot n}$ signatures.
  
  Given two signatures, $\typevec'$ and $\typevec''$, we define their
  merge, $\typevec = \typevec' \oplus \typevec''$, as follows: For
  each~$i$, $1 \leq i \leq n$, vector $\type{i}$ is computed by first
  calculating the component-wise sum of vectors $\typevec'$ and
  $\typevec''$, and then replacing with $t$ each entry greater than $t$.
  Now, if $A'$ and $A''$ are two disjoint sets
  of candidates with signatures $\typevec_{A'}$ and $\typevec_{A''}$,
  then $\typevec_{A'} \oplus \typevec_{A''}$ is a signature of their
  union. (Note that in our algorithm we apply operator $\oplus$ to
  ``signatures of disjoint sets of candidates'' only.)

  It is straightforward to verify that given a signature of a subset
  $A'$ of unregistered candidates, we can compute the scores of
  candidates $p$ and $d$. This suffices to describe our algorithm and
  to justify its correctness.  The running time is
  $O(m\cdot n \cdot t\cdot(t+1)^{t\cdot n})$ (it is calculated in the same way as in the
  proof of Theorem~\ref{lem:plurality-dcac-fpt}, except that now we have more
  signatures and the components of the signatures are $t$-dimensional
  vectors).
\end{proof}

An example for the algorithm described above, for $t$-Approval-DCAC, is provided next.

\begin{example}
Consider the following election.
\begin{align*}
  v_1 : \mathbf{d} \pref c \pref a \pref \mathbf{e} \pref b \pref \mathbf{p} \\
  v_2 : b \pref c \pref \mathbf{p} \pref \mathbf{d} \pref \mathbf{e} \pref a \\
  v_3 : a \pref c \pref \mathbf{d} \pref \mathbf{p} \pref b \pref \mathbf{e}
\end{align*}

The registered candidates are $\{d, p, e\}$
and the unregistered candidates are $\{a, b, c\}$.
We consider $2$-Approval (that is, $t = 2$),
therefore $d$ gets $3$ points, $p$ gets $2$ points, and $e$ gets $1$ point.

We have the following signatures.
\[
  \begin{array}{cc}
    ~ \\
    a\text{'s signature:} \\
    ~
  \end{array}
  \left( \begin{array}{cc}
    0 & 1 \\
    0 & 0 \\
    1 & 0
  \end{array} \right)
  \begin{array}{cc}
    ~ \\
    b\text{'s signature:} \\
    ~
  \end{array}
  \left( \begin{array}{cc}
    0 & 0 \\
    1 & 0 \\
    0 & 0
  \end{array} \right)
  \begin{array}{cc}
    ~ \\
    c\text{'s signature:} \\
    ~
  \end{array}
  \left( \begin{array}{cc}
    0 & 1 \\
    1 & 0 \\
    1 & 0
  \end{array} \right)
\]

Combining $a$ and $c$ together,
we have the following signature.
\[
  \begin{array}{cc}
    ~ \\
    a\text{'s signature } \oplus c\text{'s signature:} \\
    ~
  \end{array}
  \left( \begin{array}{cc}
    0 & 1 \\
    1 & 0 \\
    2 & 0
  \end{array} \right)
\]

Indeed,
it can be computed from the above signature,
that $\{a, c\}$ is indeed a solution, causing $p$ to win the election.
\end{example}

\newtheorem*{thmlemapprovaldcdcfpt}{Theorem~\ref{lem:approval-dcdc-fpt}}

\begin{thmlemapprovaldcdcfpt}%
\lemapprovaldcdcfpt
\end{thmlemapprovaldcdcfpt}

\newcommand{\ApprovalCheck}{\texttt{SanityCheck}\xspace}
\newcommand{\VetoCheck}{\texttt{VetoSanityCheck}\xspace}
\newcommand{\variable}[1]{\ensuremath{x_{#1}}}
\newcommand{\constant}[1]{\ensuremath{z_{#1}}}

\setlength{\textfloatsep}{40pt}

\begin{algorithm}[t!]
  \footnotesize
  \SetCommentSty{\color{gray}}
  \SetAlgoVlined

  \SetKwInput{Input}{Input}

  \SetKw{KwVariables}{Variables}
  \SetKw{KwConstants}{Constants}
  \SetKw{KwConstraints}{Constraints}
  \SetKw{KwObjective}{Objective}
  \SetKw{KwTrue}{true}
  \SetKw{KwFalse}{false}
  \SetKw{KwAnd}{and}
  \SetKw{KwAccept}{accept}
  \SetKw{KwReject}{reject}
  \SetKw{KwOr}{or}
  \SetKw{KwWith}{with}

  \SetKwFunction{Guess}{Guess}
  \SetKwFunction{ApprovalCheck}{SanityCheck}
  \SetKwFunction{VetoCheck}{VetoSanityCheck}
  \SetKwFunction{ILP}{ILP}
  \SetKwBlock{Block}{}{}

  \SetAlCapFnt{\footnotesize}

  \vspace{10pt}

  \Input{\\
    \begin{tabular}{ll}
      $((C, V), d, \solk)$&    \comm{---  input: an instance of $\constantk$-Approval-DCDC}\\
      $p$                &     \comm{---  a guessed candidate who is to defeat $d$}\\
    \end{tabular}
  }

  \vspace{10pt}

  \ForEach
  (\comm{\\ --- Run ILP for each sane $\scorevec$ such that $p$ beats $d$.})
  {\label{alg:beat-cond}
    $\scorevec = (\svec{1},\svec{2},\ldots,\svec{n})\in [4]^n$ \KwWith
  \phantom{foreach  }$|\{i \mid \svec{i}=1\}| < |\{i \mid \svec{i}=2\}|$}
  {
    \ForEach{$i \in [n]$}{
      \label{alg:approval-check}\If{\ApprovalCheck($\svec{i}$) = \KwFalse}{
          Next $\scorevec$\;
        }
      } \If{$p$ has more points than $d$ when $p$ and $d$ receive
        points as described by $\scorevec$ \KwAnd there is a solution for \ILP($\scorevec$)}
        {
          \KwAccept\;
        }
  }
  \KwReject\;

  \vspace{10pt}

  \ApprovalCheck{$\svec{i}$}
    \Block{
     \If(\\
     \comm{--- $\svec{i}=1$: only $d$ gains one point.}
     )
     {$\svec{i}=1$ \KwAnd ($v_i: p\pref d$)}{
       \Return \KwFalse\; 
     }
     \If(\\
     \comm{--- $\svec{i}=2$: only $p$ gains one point.}
     )
     {$\svec{i}=2$ \KwAnd ($v_i: d\pref p$)}{
       \Return \KwFalse\;
     }
       \Return \KwTrue\;
   }

  \vspace{10pt}

  \ILP{$\scorevec=(\svec{1},\svec{2},\ldots,\svec{n})$}:
  \Block{
    \KwVariables\\
      $\forall \signature \in [3]^{n}: \variable{\signature}$ \comm{\, --- \# \emph{deleted} candidates with $\{d,p\}$-signature $\signature$}
    
    \KwConstants\\
      $\forall \signature \in [3]^{n}: \constant{\signature}$ \comm{\, --- \# \emph{existing} candidates with $\{d,p\}$-signature $\signature$}

    \KwConstraints
    \Block{
      ${\sum_{\signature}x_{\signature}}\le \solk$
      $\forall \signature \in [3]^{n}: x_{\signature} \le z_{\signature}$
      \label{alg:constraint1}\\
      $\forall i \in [n]:$\\
      \If(\\
      \comm{--- $v_i: d \pref p$ \KwAnd only $d$ gains one point, or}\\
      \comm{--- $v_i: p \pref d$ \KwAnd only $p$ gains one point})
      {$\svec{i}=1$ \KwOr $\svec{i}=2$}
      {
        $\sum_{\forall \signature: \sig{i}=3}(\constant{\signature}-\variable{\signature}) \le t-1$\\\label{alg:the first gains one point}
        $\sum_{\forall \signature: \sig{i}=3 \vee \sig{i}=2}(\constant{\signature}-\variable{\signature})\ge t-1$\label{alg:the second gains zero points}
      }
      \ElseIf(
      \\
      \comm{--- Both $d$ and $p$ gain one point each}
      ){$\svec{i}=3$}
      {
        $\sum_{\forall \signature: \sig{i}=3 \vee \sig{i}=2}(\constant{\signature}-\variable{\signature}) + 2 \le t$
        \label{alg:both gain one point}
      }
      \Else(
      \\
      \comm{--- No one gains one point}
      ){
        $\sum_{\forall \signature: \sig{i}=3}(\constant{\signature}-\variable{\signature}) \ge t$
      }\label{alg:both gain zero points}
    }
  }    
  \caption{\small $\fpt$ algorithm for $\constantk$-Approval-DCDC.}
  \label{alg:fpt-approval-dcdc}
\end{algorithm}

We first describe our approach to proving the theorem and give the
formal proof below.  Let us fix a positive integer $t$ and let
$((C,V), d, \solk)$ be an instance of $t$-Approval-DCDC, where
$V=(v_1,\ldots,v_n)$.  We focus on the case of $\constantk$-Approval
and later we will argue how to adapt the results to apply to the case
of $\constantk$-Veto.  We guess a candidate $p$, whose role is to
defeat the despised candidate $d$.  For each such candidate~$p$ we do
the following. First, we make an initial brute-force search: For each
voter, we ``guess'' one of at most four possible choices of how $d$
and $p$ would gain points after our action of deleting candidates:
\begin{enumerate}
\item choice one: only $d$ receives one point,
\item choice two: only $p$ receives one point,
\item choice three: both candidates receive one point, and 
\item choice four: neither $p$ nor $d$ receive a point.
\end{enumerate}
We record our guesses in vector $\scorevec$.  For each guessed $p$ and
$\scorevec$, we check whether giving the points according to our guesses in
$\scorevec$ guarantees that $p$ has more points than $d$. If so, then we
run an integer linear program to verify if it is at all possible to
ensure that every voter gives points to candidates $p$ and $d$ as
described by vector $\scorevec$, and to compute the smallest number of
candidates we have to delete to ensure this.  The complete procedure,
for the case of $\constantk$-Approval-DCDC, is given as
Algorithm~\ref{alg:fpt-approval-dcdc}.

\begin{proof}[Proof of Theorem~\ref{lem:approval-dcdc-fpt}]
  We start by considering the case of $\constantk$-Approval-DCDC.  The
  running time for Algorithm~\ref{alg:fpt-approval-dcdc} is easy to
  verify: we guess a candidate~$p$ and a possible way of giving $p$
  and $d$ points, followed by running an ILP.  Therefore, the running
  time is $O(m\cdot 4^n)$ times the cost of running the ILP.  The ILP
  has $3^n$ variables and $(3^n+2n)$~constraints.  Thus, employing the
  famous result by~\citet{len:j:integer-fixed}, our
  algorithm runs in $O^*(4^n \cdot f(n))$ where $f$ is a function that
  describes the running time of the ILP solver and solely depends on
  $n$~\shortcite{Kan87,len:j:integer-fixed}.
  
  To prove the correctness of the algorithm, it suffices to show the
  correctness of the ILP program for a given guess of $p$ and
  $\scorevec$.  First, the constraint in Line~\ref{alg:constraint1}
  ensures that we do not delete more candidates with a given
  $\{d,p\}$-signature~$\signature$ than there are present in the election (c.f. the meaning of a signature in Definition~\ref{def:signature}).  The
  remaining signatures verify that we can implement vector $\scorevec$.
  For each $i$, $1 \leq i \leq n$, we verify whether it is possible to 
  implement guess $\svec{i}$:
  \begin{enumerate}
  \item If $\svec{i}$~$=$~$1$ (that is, $d$ gains a point from the $i$th
    voter but $p$ does not), then according to our sanity check
    (\ApprovalCheck) we have that $v_i$ prefers $d$ over $p$.  Thus,
    after the candidate deletion, $d$ must be ranked in the first
    $\constantk$ positions~(Line~\ref{alg:the first gains one point})
    and $p$ must be ranked behind the $\constantk$'th
    position~(Line~\ref{alg:the second gains zero points}).
  \item If $\svec{i} = 2$ which means that only $p$ gains one point,
    then $v_i$ prefers $p$ over $d$.  Thus, after the candidate
    deletion, $p$ must be ranked in the first $\constantk$
    positions~(Line~\ref{alg:the first gains one point}) and $d$ must
    be ranked behind the $\constantk$'th position~(Line~\ref{alg:the
      second gains zero points}).
  \item If $\svec{i} = 3$, 
    then both candidates gain one point each and must be ranked in the first $\constantk$ positions~(Line~\ref{alg:both gain one point}) after the candidate deletion.
  \item Otherwise, both gain zero points and must be ranked behind the
    $\constantk$'th position~(Line~\ref{alg:both gain zero points})
    after the candidate deletion.
  \end{enumerate}
  This justifies the correctness of the ILP and completes the proof
  for the case of $\constantk$-Approval.
  
  For the case of $\constantk$-Veto, it suffices to use the same
  approach as for $\constantk$-Approval, provided that we first
  reverse all preference orders and consider that a candidate is a
  winner if this candidate's score is the lowest (in essence, this is
  equivalent to replacing ``points'' with ``vetoes'' in the above
  reasoning).
\end{proof}

\end{document}